\documentclass{article}

\usepackage[labelfont=it,labelsep=period,font=small,skip=0pt,belowskip=-0.75em]{caption}

\usepackage{microtype}
\usepackage{graphicx}
\usepackage{subfig}
\usepackage{booktabs} %

\usepackage{hyperref}

\usepackage[accepted]{icml2022-arxiv}

\usepackage{amsmath}
\usepackage{amssymb}
\usepackage{mathtools}
\usepackage{amsthm}

\usepackage[inline,shortlabels]{enumitem}

\usepackage{multirow}

\usepackage{xurl}

\usepackage[capitalize,noabbrev]{cleveref}

\theoremstyle{plain}
\newtheorem{theorem}{Theorem}
\newtheorem{proposition}[theorem]{Proposition}
\newtheorem{lemma}[theorem]{Lemma}
\newtheorem{corollary}[theorem]{Corollary}
\theoremstyle{definition}

\newtheorem{hypothesis}[theorem]{Hypothesis}
\theoremstyle{remark}
\newtheorem*{remark}{Remark}

\usepackage{widebar}

\newcommand{\eqdef}{\coloneqq}
\newcommand{\NTK}{\textrm{eNTK}}
\newcommand*{\n}{N}
\newcommand*{\p}{P}
\renewcommand*{\c}{C}  %
\newcommand*{\Sigmahat}{\widehat\Sigma}
\newcommand*{\Sigmatilde}{\Sigma'}
\newcommand*{\tSigma}{\widetilde\Sigma}
\newcommand*{\tX}{\widetilde X}
\newcommand*{\tSigmahat}{\widehat{\widetilde\Sigma}}
\newcommand*{\tbeta}{\tilde\beta}
\newcommand*{\tkappa}{\tilde\kappa}
\newcommand*{\tm}{\widetilde m}
\newcommand*{\GCV}{\mathrm{GCV}}
\let\epsilon\varepsilon
\let\phi\varphi

\newcommand*{\Disttilde}{\smash{{\Dist'}}}
\newcommand*{\betatilde}{\beta'}
\newcommand*{\betatildehat}{\hat\beta'}
\newcommand*{\Sigmatildehat}{\smash{\widehat{\Sigma}'}}
\newcommand*{\xtilde}{x'}
\newcommand*{\Frob}{\mathrm{F}}

\newcommand*{\C}{\mathbb{C}}   %
\newcommand*{\HH}{\mathbb{H}}   %
\renewcommand*{\i}{\mathrm{i}}
\newcommand*{\Dist}{\mathcal{D}}
\DeclareMathOperator*{\E}{\mathbb{E}}
\let\Pr\relax
\DeclareMathOperator*{\Pr}{\mathbb{P}}  %

\newcommand*{\parfrac}[2]{\frac{\partial#1}{\partial#2}}  %
\DeclareMathOperator{\Tr}{Tr}
\newcommand*{\N}{\mathbb{N}}
\newcommand*{\R}{\mathbb{R}}
\newcommand*{\Risk}{\mathcal{R}}
\newcommand*{\TrainRisk}{\mathcal{R}_{\mathrm{empirical}}}
\newcommand*{\NormRisk}{\widehat{\mathcal{R}}_{\mathrm{norm}}}
\newcommand*{\SpecRisk}{\widehat{\mathcal{R}}_{\mathrm{spec}}}
\newcommand*{\OmniRisk}{{\mathcal{R}}_{\mathrm{omni}}}
\DeclareMathOperator*{\argmin}{arg\,min}  %
\makeatletter
\newcommand*{\tensor}{\@ifnextchar_{\tens@r\otimes}{\otimes}}
\def\tens@r#1_#2{\mathbin{\mathchoice
  {\mathop{#1}\limits_{#2}}{{#1}_{#2}}{{#1}_{#2}}{{#1}_{#2}}}}
\makeatother
\newcommand*{\transp}{{\mkern-1mu\mathsf{T}\mkern-0.5mu}}

\let\brack\relax

\DeclarePairedDelimiter{\abs}{\lvert}{\rvert}
\DeclarePairedDelimiter{\norm}{\lVert}{\rVert}
\DeclarePairedDelimiter{\ceil}{\lceil}{\rceil}
\DeclarePairedDelimiter{\floor}{\lfloor}{\rfloor}
\DeclarePairedDelimiter{\paren}{\lparen}{\rparen}
\DeclarePairedDelimiter{\brack}{\lbrack}{\rbrack}

\DeclarePairedDelimiter{\bangle}{\langle}{\rangle}

\let\Re\relax
\let\Im\relax
\DeclareMathOperator{\Re}{Re}  %
\DeclareMathOperator{\Im}{Im}  %

\newcommand*{\given}{\mkern2mu\ifnum\currentgrouptype=16\middle\fi|\mkern2mu}

\icmltitlerunning{Random Matrix Models Predict How Real-World Neural Representations Generalize}

\begin{document}

\twocolumn[
\icmltitle{More Than a Toy: Random Matrix Models Predict How \texorpdfstring{\\}{} Real-World Neural Representations Generalize}

\begin{icmlauthorlist}
\icmlauthor{Alexander Wei}{berkeley}
\icmlauthor{Wei Hu}{berkeley}
\icmlauthor{Jacob Steinhardt}{berkeley}
\end{icmlauthorlist}

\icmlaffiliation{berkeley}{UC Berkeley, Berkeley, California, USA}

\icmlcorrespondingauthor{Alexander Wei}{awei@berkeley.edu}

\vskip 0.3in
]

\printAffiliationsAndNotice{}  %

\begin{abstract}
Of theories for why large-scale machine learning models generalize despite being vastly overparameterized, which of their assumptions are needed to capture the qualitative phenomena of generalization in the real world? On one hand, we find that most theoretical analyses fall short of capturing these qualitative phenomena even for kernel regression, when applied to kernels derived from large-scale neural networks (e.g., ResNet-50) and real data (e.g., CIFAR-100). On the other hand, we find that the classical GCV estimator (Craven and Wahba, 1978) accurately predicts generalization risk even in such overparameterized settings. To bolster this empirical finding, we prove that the GCV estimator converges to the generalization risk whenever a local random matrix law holds. Finally, we apply this random matrix theory lens to explain why pretrained representations generalize better as well as what factors govern scaling laws for kernel regression. Our findings suggest that random matrix theory, rather than just being a toy model, may be central to understanding the properties of neural representations in practice.
\end{abstract}

\section{Introduction}\label{sec:introduction}

\begin{figure}
    \centering
    \includegraphics[width=\columnwidth]{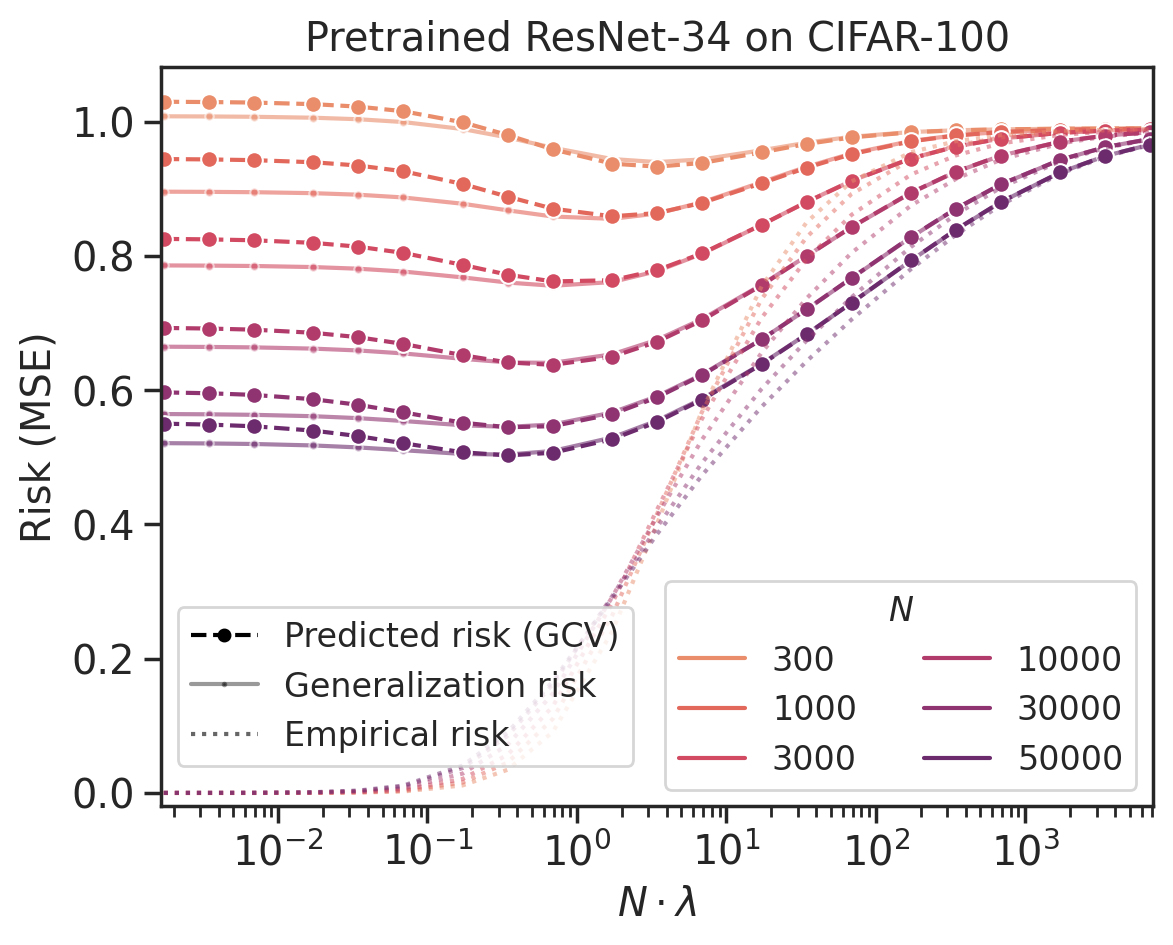}
    \caption{Predicted vs.\ actual generalization risk of a pretrained ResNet-34 empirical NTK on CIFAR-100 over dataset sizes $\n$ and ridge regularizations $\lambda$. Corresponding training risks are plotted in the background. The fit achieving the lowest MSE has $19.9\%$ test error on CIFAR-100 (vs.\ $15.9\%$ from finetuning the ResNet).}
    \label{fig:train-test}
\end{figure}

The fact that deep neural networks trained with many more parameters than data points can generalize well contradicts conventional statistical wisdom \cite{zhang17understanding}. This observation has inspired much theoretical work, with one of the goals being to explain the generalization and scaling behavior of such models. In this paper, we study how these theoretical perspectives map onto reality. What assumptions are necessary (or sufficient) to capture the qualitative phenomena (e.g., pretraining vs.\ random initialization, scaling laws) of large-scale models? And what do they reveal about generalization in the real world?

An adequate theoretical treatment should at least 
predict the behavior of high-dimensional \emph{linear} models. 
To assess this, we focus on linear models derived from neural representations (e.g., final layer activations or empirical neural tangent kernels) of large-scale networks on vision data. We test whether different theories can predict how kernel ridge regression on these representations generalizes, given only the training data.  %

In this setting of regression on realistic kernels, we find that most theoretical analyses already face severe challenges. A major difficulty is that the ground truth function has large---effectively infinite---kernel norm, which we verify empirically on several datasets. Consequently, norm-based generalization bounds are vacuous or even increase with dataset size, echoing concerns raised by \citet{belkin18understand} and \citet{nagarajan19uniform}. Other challenges for estimating generalization include the slow convergence of the empirical covariance matrix and the fact that noise and signal are indistinguishable in high-dimensional settings.

However, not all is lost. We find that the \emph{generalized cross-validation (GCV) estimator} \cite{craven78smoothing} does accurately predict the generalization risk, even when typical norm- or spectrum-based formulas struggle. GCV is accurate over a wide range of dataset sizes and regularization strengths, for classification tasks of varying complexities, and for representations extracted from residual networks both at random initialization and after pretraining. For instance, \Cref{fig:train-test} compares the GCV estimate against the true generalization risk for an ImageNet-pretrained ResNet-34 representation on CIFAR-100.

To justify the performance of the GCV estimator, we prove that it converges to the true generalization risk whenever a local random matrix law \cite{knowles17anisotropic} holds. Our analysis of this estimator allows for the highly anisotropic covariates and large-norm ground truth functions observed in our empirical setting.
Along the way, we also generalize recent random matrix analyses of high-dimensional ridge regression \cite{hastie21surprises, canatar21spectral, wu20optimal, jacot20kernel, loureiro21learning, richards21asymptotics, mel21arbitrary, simon21neural} to this setting. Finally, our analysis provides a new perspective on this classical estimator that explains how its form arises in connection to random matrix theory.

We next apply this random matrix theory lens to explore basic questions about neural representations: Why do pretrained models generalize better than randomly initialized ones? And what factors govern the rates observed in neural scaling laws \cite{kaplan20scaling}? We find that alignment---how easy it is to represent the ground truth function in the eigenbasis \cite{marquardt1975ridge, caponnetto07optimal, canatar21spectral}---is necessary to explain the performance of deep learning models. In particular, pretrained representations perform better than random representations due to better alignment, and \emph{despite} worse eigenvalue decay. Finally, we provide sample-efficient methods to estimate the alignment and eigenvalue decay, which circumvent the slow convergence of the sample covariance matrix, and show that these two quantities are sufficient to predict the scaling law rate of ridge regression on natural data.

Our empirical findings and theoretical analysis show that a random matrix theoretic perspective stands apart at capturing the generalization of high-dimensional linear models on real data. More classical approaches, which often boil down to norms and/or eigendecay, do not suffice because generalization typically depends on the specific alignment between a high-norm ground truth function and the population covariance matrix. More broadly, our results suggest that accounting for random matrix effects is necessary to model the qualitative phenomena of deep learning---and in the case of kernel regression, sufficient.

\begin{remark}
In addition to our scientific contribution, we develop a library for computing large-scale empirical neural tangent kernels (e.g., for all of CIFAR-10 on a ResNet-101): {\footnotesize\url{https://github.com/aw31/empirical-ntks}}. Our library fills in a gap in existing tools for exploring neural tangent kernels at scale.
\end{remark}

\subsection{Related Work}

Since \citet{zhang17understanding}, %
many researchers have sought to explain why overparameterized models generalize. High-dimensional linear models capture many of the central empirical phenomena and are a natural proving ground for theories of overparameterized models \cite{mei20generalization, belkin20two, bartlett20benign}. Recently, a flurry of works has precisely analyzed the generalization risk of high-dimensional ridge regression under various assumptions, typically Gaussian data in the asymptotic limit \cite{hastie21surprises, canatar21spectral, wu20optimal, jacot20kernel, rosset20fixed, loureiro21learning, richards21asymptotics, mel21arbitrary, simon21neural}. Our analysis, like that of \citet{hastie21surprises}, is based on a local random matrix law \cite{knowles17anisotropic} and produces non-asymptotic bounds that hold for general distributions.

Other, more classical, approaches to generalization include Rademacher complexity (e.g., \citet{bartlett01rademacher, bartlett02localized}), norm-based measures (e.g., \citet{bartlett96size, neyshabur15norm}), PAC-Bayes approaches for stochastic models (e.g., \citet{mcallester99pac, dziugaite17computing}), and spectral notions of effective dimension (e.g., \citet{zhang05learning, dobriban18high, bartlett20benign}). %
While some of these measures have been studied in large-scale experiments \cite{jiang20fantastic, dziugaite20search}, 
our evaluations focus on a different perspective: we study whether they capture the basic empirical phenomena of overparameterized models, such as scaling laws and the effect of pretraining.

To estimate generalization risk, we revisit the GCV estimator of \citet{craven78smoothing}. GCV was initially studied as an estimate of error over a \emph{fixed} sample \cite{golub79generalized, li86asymptotic, cao06oracle}. Such analyses, however, do not account for the randomness of the sample and thus fail to capture high-dimensional settings with disparate train and test risks. Recently, the high-dimensional setting has received more attention: \citet{jacot20kernel} analyze GCV for random, Gaussian covariates in the ``classical'' regime where train risk approximates test risk.\footnote{See \Cref{sec:classical} for a detailed discussion of the classical vs.\ non-classical regimes of high-dimensional ridge regression.} And,
\citet{hastie21surprises}, \citet{adlam20ntk}, and \citet{patil21uniform} asymptotically analyze GCV when the ratio $\p/\n$ between the dimension $\p$ and the sample size $\n$ converges to a fixed limit. In contrast, to study scaling in $\n$ (for fixed $\p$), we prove \emph{non-asymptotic} bounds on the convergence of GCV that hold: (i) beyond the classical regime, (ii) for a wide range of $\n/\p$, and (iii) for general covariance structures.
Experimentally, GCV has previously been studied by \citet{efron86biased} and \citet{rosset20fixed} in numerical simulations and by \citet{jacot20kernel} for shift-invariant kernels on the MNIST and Higgs datasets. Our experiments take these investigations to a significantly larger scale and focus on more realistic neural representations.

One phenomenon we study---neural scaling laws---was first observed by \citet{kaplan20scaling}. Since this observation, \citet{bahri21explaining} derive a spectrum-only formula for kernel regression scaling, and \citet{cui21generalization} derive precise rates for ridge regression scaling in random matrix regimes. In comparison, we show that alignment (and not just the eigenvalues) is essential for understanding scaling in practice, and we also use random matrix theory to give a more principled way to estimate the decay rates of the population eigenvalues and alignment coefficients.

Finally, the neural representations we study are motivated by the \emph{neural tangent kernel} (NTK) \cite{jacot18ntk}. There has been a rich line of theoretical work studying ultra-wide neural networks and their relationship to NTKs (e.g., \citet{arora19exact, lee19wide, yang19scaling}).
In contrast, we work with NTKs extracted from realistic, finite-width networks---including pretrained networks---and use them as a testbed for exploring measures of generalization.

\section{Preliminaries}\label{sec:preliminaries}

\subsection{High-dimensional Ridge Regression}\label{sec:model}

We study a simple model of linear regression, in which we predict labels $y\in\R$ from data points $x\in\R^\p$. Each $x$ is drawn from a distribution $\Dist$ with \emph{unknown} second moment $\Sigma\coloneqq\E_{x\sim\Dist}\bigl[xx^\transp\bigr]$, and its label $y$ is given by $y = \beta^\transp x$ for an \emph{unknown} ground truth function\footnote{We assume---for simplicity's sake---that the linear model is well-specified and that labels are noiseless. This holds without loss of generality in high dimensions: both noise and misspecification can be embedded into the model by adding an additional ``noise'' dimension. See \Cref{sec:challenges/noise,,sec:noise} for details.} $\beta\in\R^\p$. Let $\Sigma$ have eigendecomposition \smash{$\sum_{i=1}^\p \lambda_i v_iv_i^\transp$}, with $\lambda_1\ge\cdots\ge\lambda_\p$.

To estimate $\beta$, we assume we have a dataset $\{(x_i,y_i)\}_{i=1}^\n$ of $\n$ independent samples, with $x_i\sim\Dist$ and $y_i = \beta^\transp x_i$ for all $i$. For notational convenience, we write this dataset as $(X, y)$, where $X\in\R^{\n\times\p}$ has $i$-th row $x_i$ and \smash{$y\in\R^\n$} has $i$-th entry $y_i$. Let $\Sigmahat\coloneqq \frac 1\n X^\transp X$ be the empirical second moment matrix, with eigendecomposition \smash{$\sum_{i=1}^\p \hat\lambda_i \hat v_i\hat v_i^\transp$} such that \smash{$\hat\lambda_1\ge\cdots\ge\hat\lambda_\n$}.

Given training data $(X, y)$ and an estimator $\hat\beta = \hat\beta(X,y)$, our goal in this paper is to predict its generalization risk $\Risk$, defined as
$\Risk(\hat\beta) \coloneqq \E_{x\sim\Dist}\bigl[(\beta^\transp x - \hat\beta^\transp x)^2\bigr]$,
\emph{without access to} an independently drawn test dataset.

We focus on the ridge regression estimators $\hat\beta_\lambda$ given by
\[ \hat\beta_\lambda\coloneqq\argmin_{\hat\beta} \frac 1\n \sum_{i=1}^\n \bigl(y_i - \hat\beta^\transp x_i\bigr)^2 + \lambda\|\hat\beta\|_2^2 \]
for $\lambda > 0$, and $\hat\beta_0\coloneqq\lim_{\lambda\to 0^+} \hat\beta_\lambda$.%

Recent theoretical advances \cite{hastie21surprises, canatar21spectral, wu20optimal, jacot20kernel, loureiro21learning, richards21asymptotics, mel21arbitrary, simon21neural} have characterized \smash{$\Risk(\hat\beta_\lambda)$}
under a variety of random matrix assumptions. These works all show that $\Risk(\hat\beta_\lambda)$ can be approximated by the \emph{omniscient risk estimate}
\begin{equation}\label{eq:population}
\OmniRisk^\lambda\coloneqq \frac{\partial\kappa}{\partial\lambda}\cdot\kappa^2\sum_{i=1}^\p \biggl(\frac{\lambda_i}{(\kappa + \lambda_i)^2} \bigl(\beta^\transp v_i\bigr)^2 \biggr),
\end{equation}
where $\kappa = \kappa(\lambda, \n)$ is an \emph{effective regularization} term (see \eqref{eq:kappa} for a definition).
We call this expression the omniscient risk estimate because it depends on the \emph{unknown} second moment matrix $\Sigma$ and the \emph{unknown} ground truth $\beta$. Our analysis will approximate \eqref{eq:population} using only the empirical second moment matrix $\Sigmahat$ and the observations $y$, while also yielding a concrete relationship between train and test risk.

\subsection{Methods for Predicting Generalization Risk}

We discuss several baseline approaches for predicting generalization risk and then describe the GCV estimator.

The simplest method uses empirical risk (i.e., training error)
\smash{$\TrainRisk(\hat\beta)\coloneqq \frac 1\n \sum_{i=1}^\n (y_i - \hat\beta^\transp x_i)^2$} as a proxy. This is the foundation of uniform convergence approaches in learning theory (e.g., VC-dimension and Rademacher complexity). However, training error is a poor predictor of test error in the overparameterized regime, as seen in \Cref{fig:train-test}.

Ridge regression admits more specific analyses. A typical approach takes a bias-variance decomposition over label noise and bounds each term with norm- or spectrum-based quantities. For instance, the recent textbook of \citet{bach21learning} shows, based on matrix concentration inequalities, that
\begin{equation}\label{eq:naive}
\Risk(\hat\beta_\lambda)\le \underbrace{16\lambda\|\beta\|_2^2}_{\text{norm-based}} + \underbrace{16\frac{\sigma^2}{\n}\Tr(\Sigma(\Sigma + \lambda I)^{-1})}_{\text{spectrum-based}}
\end{equation}
holds when $\n\lambda$ is large enough, where $\sigma^2$ upper bounds the variance of the label noise. Such norm- or spectrum-based terms are typical of many theoretical analyses.

\paragraph{The GCV estimator.}
Cross-validation is a third approach to predicting generalization risk. However, cross-validation is not guaranteed to work in high dimensions and can fail in practice \cite{bates21crossvalidation}.
\citet{craven78smoothing} thus introduce the \emph{generalized cross-validation} (GCV) estimator\!\!
\begin{equation}\label{eq:gcv}
\GCV_\lambda\coloneqq{\left(\frac 1\n\sum_{i=1}^\n \frac \lambda{\lambda + \hat\lambda_i}\right)^{-2}}\TrainRisk(\hat\beta_\lambda),
\end{equation}
which they and \citet{golub79generalized} heuristically derive by modifying cross-validation to be rotationally invariant.\footnote{As we did for $\hat\beta_0$, we define $\GCV_0\coloneqq\lim_{\lambda\to 0^+} \GCV_\lambda$.} We will study this estimator empirically and show %
its form can be understood as a consequence of random matrix theory.

\subsection{Experimental Setup: Empirical NTKs}\label{sec:setup}

To benchmark our risk estimates in realistic settings, we use feature representations derived from large-scale, possibly pretrained neural networks.  %
Specifically, we use the \emph{empirical} neural tangent kernel (\NTK{}). Given a neural network $f(\,\cdot\,; \theta)$ with $\p$ parameters $(\theta\in\R^\p)$ and $\c$ output logits ($f(x; \theta)\in\R^\c$), the \NTK{} representation of a data point $x$ at $\theta_0$
is the Jacobian $\phi_{\NTK}(x) \eqdef \frac{\partial f}{\partial \theta}(x; \theta_0)\in\R^{\p\times\c}$.

\paragraph{Models and datasets.}

We consider \NTK{} representations of residual networks on several computer vision datasets, both at random initialization and after pretraining. %
Specifically, we consider  ResNet-\{18, 34, 50, 101\} applied to the %
CIFAR-\{10, 100\} \cite{krizhevsky09learning}, Fashion-MNIST \cite{xiao2017fashion}, Flowers-102 \cite{nilsback08flowers}, and Food-101 \cite{bossard14food} datasets. All random initialization was done following \citet{he15delving}; pretrained networks (obtained from PyTorch) were pretrained on ImageNet and had randomly re-initialized output layers.

To verify that pretrained \NTK{} representations achieve competitive generalization performance, we compare kernel regression on pretrained \NTK{}s to regression on the last layer activations and to finetuning the full network with SGD (see \Cref{table:ntks}). We find that pretrained \NTK{}s achieve accuracy much closer to that of finetuning than that of regression on the last layer.
The \NTK{}s we consider also have stronger empirical performance than the best-known infinite-width NTKs~\citep{arora19exact, li19enhanced, lee20finite}.

\begin{table}
    \centering
    \small
    \begin{tabular}{crrr} \toprule
        Configuration           & {\color{darkgray}\em\!Finetuning\!} & \NTK{}   & Last layer \\ \midrule
        CIFAR-10 / ResNet-18    & {\color{darkgray} $4.3\%$ }      & $6.7\%$  & $14.0\%$   \\
        CIFAR-100 / ResNet-34   & {\color{darkgray} $15.9\%$}      & $19.0\%$ & $33.9\%$   \\
    \!\!Flowers-102 / ResNet-50\!\! & {\color{darkgray} $5.6\%$}      & $7.0\%$  & $9.7\%$    \\
        Food-101 / ResNet-101   & {\color{darkgray} $15.3\%$}      & $21.3\%$ & $33.7\%$   \\ \bottomrule
    \end{tabular}
    \caption{Test classification error rates of finetuning with SGD, kernel regression on the \NTK{}, and linear regression on the last layer activations for various datasets and pretrained models.}
    \label{table:ntks}
\end{table}

\paragraph{Computational considerations.}
For computations with \NTK{} representations, we apply the kernel trick and instead work with the \emph{\NTK{} matrix} $\smash{\bigl[\phi_{\NTK{}}(x_i)^\transp \phi_{\NTK{}}(x_j)\bigr]_{i,j=1}^{\smash{\n}}}
\in \R^{\smash{(\n\times\c)\times (\n\times\c)}}$. To further speed up computation, we take advantage of the fact that, since our models have randomly initialized output layers, the expected \NTK{} can be written as $I_\c\tensor K_0$, for some kernel $K_0\in\R^{\n\times\n}$ and the $\c\times\c$ identity matrix $I_\c$ \cite{lee20finite}. The full $\n\c\times\n\c$ \NTK{} can thus be approximated by $I_\c\tensor K$, where $K$ is the \NTK{} with respect to a single randomly initialized output logit. Notice that kernel regression with respect to $I_\c\tensor K$ decomposes into $\c$ independent kernel regression problems, each with respect to $K$. To reduce compute, we apply this approximation in all of our experiments.

\paragraph{Baseline approaches.}
To illustrate some of the challenges inherent to this setting, we compare GCV against two norm- and spectrum-based expressions, similar to those of \eqref{eq:naive}. We describe these baselines in detail in \Cref{sec:empirics}.

\section{Challenges of High-Dimensional Regression from the Real World}\label{sec:challenges}

We make several empirical observations that challenge most theoretical analyses: 
\begin{enumerate*}[label=(\roman*)]
    \item The ground truth $\beta$ has effectively infinite norm, leading  $\norm{\hat\beta_\lambda}_2$ to grow quickly with $\n$ and making norm-based bounds vacuous.
    \item When $\n\ll\p$, the empirical second moment $\Sigmahat$ is not close to its population mean $\Sigma$.
    \item Many analyses estimate risk in terms of noise in the training set, but noise and signal are interchangeable in high dimensions, making such estimates break down.
\end{enumerate*}

\subsection{Norm-based Bounds Are Vacuous}\label{sec:norm}

\begin{figure}
    \centering
    \includegraphics[width=\columnwidth]{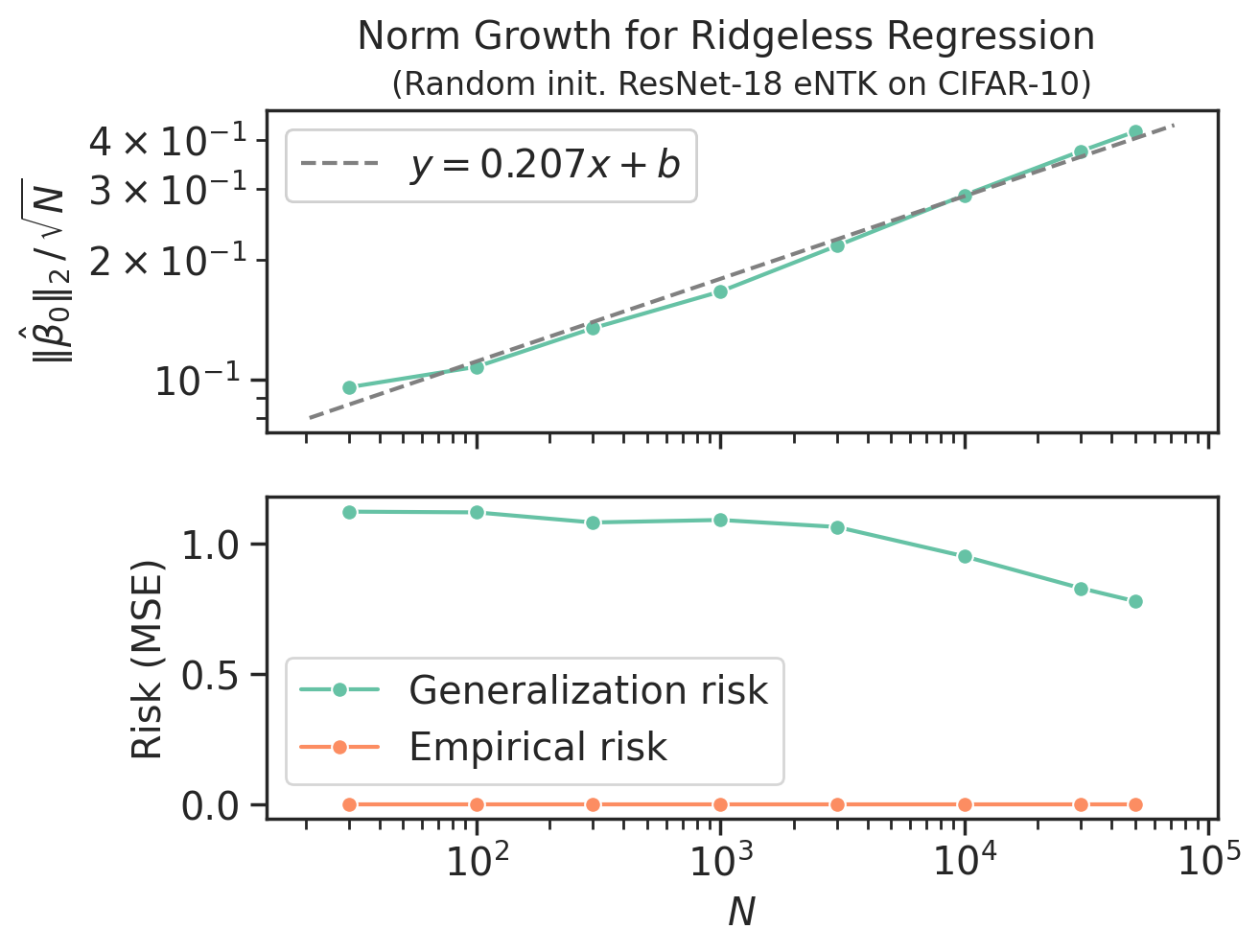}
    \caption{The top graph plots the growth of \smash{$\frac{\norm{\hat\beta_0}_2}{\sqrt{\n}}$} in $\n$ for linear regression on the \NTK{} of a randomly initialized ResNet-18 on Fashion-MNIST. The bottom graph shows that the generalization risk of $\hat\beta_0$ decreases in $\n$ under the same setup, despite the growth in $\norm{\hat\beta_0}_2$, while the empirical risk of $\hat\beta_0$ remains $0$ throughout.}
    \label{fig:norm-growth}
\end{figure}

The norms $\norm{\beta}_2$ and $\norm{\hat\beta}_2$ are often used to measure function complexity in generalization bounds. Here, we examine how these norms behave for kernel regression in practice.

Many theoretical analyses, including Rademacher complexity \cite{bartlett01rademacher}, give risk bounds for an estimator $\hat\beta$ in terms of the quantity $\norm{\hat\beta}_2 / \sqrt{\n}$ (or a monotonic function thereof). However, \smash{$\norm{\hat\beta}_2/\sqrt{\n}$} can \emph{increase} as $\n$ increases (and the generalization risk decreases): \Cref{fig:norm-growth} depicts this for $\hat\beta_0$ computed on the \NTK{} of a randomly initialized ResNet-18 on Fashion-MNIST. Moreover, this finding is consistent across models and datasets (see \Cref{sec:additional}).
Consequently, norm-based bounds give the wrong qualitative prediction for scaling. This echoes the findings of \citet{nagarajan19uniform} and shows norm-based bounds can fail even for practical linear models.

Other analyses rely on the norm $\norm{\beta}_2$ of the ground truth, either directly in the risk estimate (e.g., \citet{dobriban18high}) or as a term in the error bound (e.g., \citet{hastie21surprises}). However, \Cref{fig:norm-growth} suggests that $\beta$ has large norm: for a clean dataset like CIFAR-10, we can assume that the labels are close to noiseless.\footnote{Empirical studies on CIFAR-10 find mislabeled points at a less than 1\% prevalence \cite{northcutt21pervasive}, and the best models achieve over 99\% test accuracy \cite{dosovitskiy21image}.} In this case, \smash{$\hat\beta_0$} is the projection of $\beta$ onto $X$, from which it follows that $\norm{\hat\beta_0}_2\le\norm{\beta}_2$. Supposing that the superlinear growth of $\smash{\norm{\hat\beta_0}_2}$ in $\n$ continues, the norm $\norm{\beta}_2$ must be large. It may thus make the most sense to think of $\beta$ as having effectively infinite norm. However, this has the effect of making bounds that rely on $\norm{\beta}_2$ vacuous.\footnote{\citet{belkin18understand} suggest the perceptron analysis \cite{novikoff62convergence} as a way to understand generalization in the noiseless setting; however, a large $\norm{\beta}_2$ makes this approach ineffective as well.}

\subsection{$\Sigmahat$ Converges Slowly to $\Sigma$}

\begin{figure}
    \centering
    \includegraphics[width=\columnwidth]{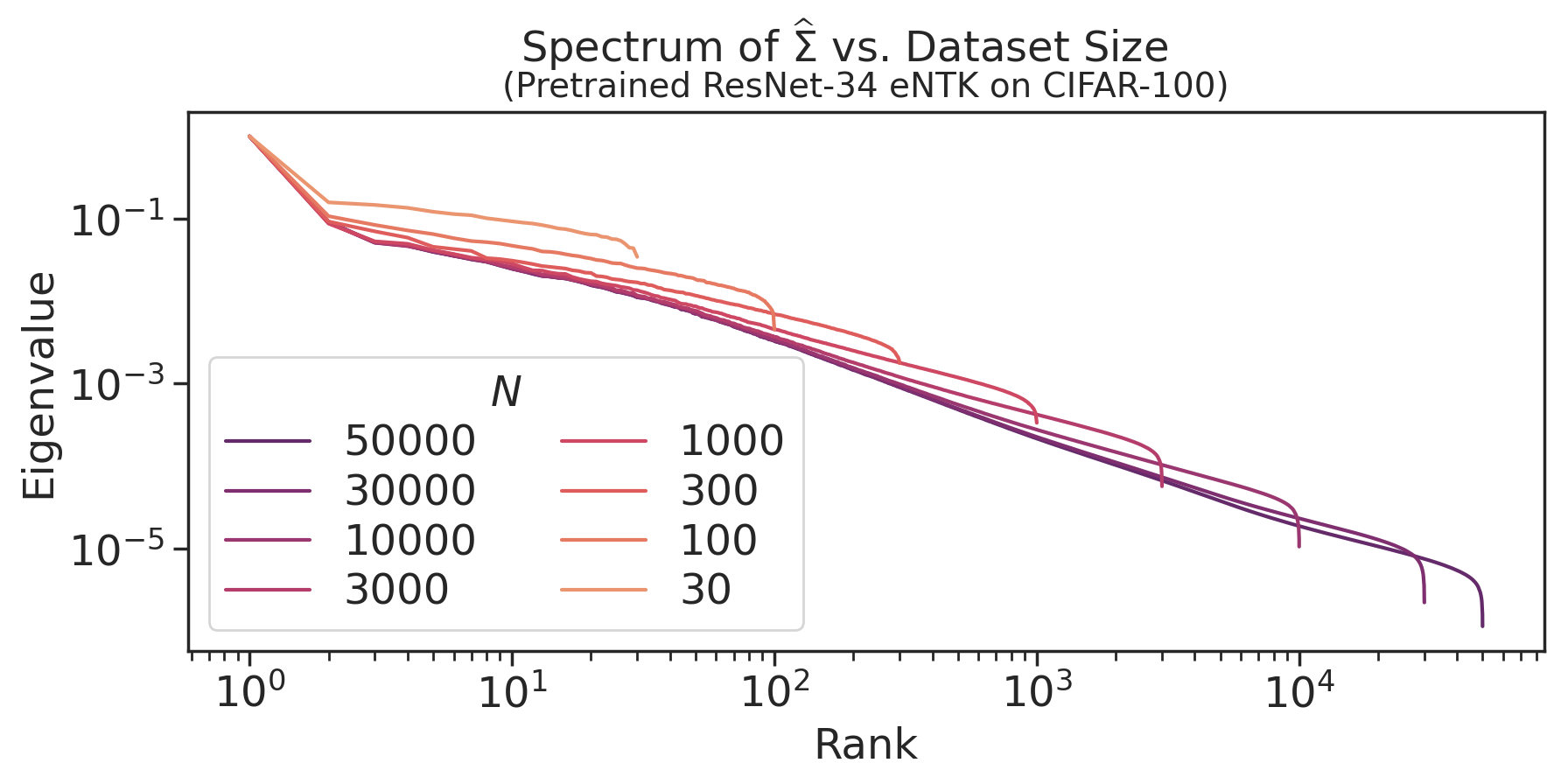}
    \caption{Each line plots the pairs $(i, \hat\lambda_i)$ for a \smash{$\Sigmahat$} computed from $\n$ pretrained ResNet-34 \NTK{} representations of CIFAR-100 images. The $\Sigmahat$ eigenvalues converge slowly, and it is not obvious---particularly from considering only a single $\n$---what the scaling trend is.}
    \label{fig:spectrum}
\end{figure}

The high dimensionality of our setting ($\p\gg\n$) implies the empirical second moment matrix $\Sigmahat$ is slow to converge to its expectation $\Sigma$. \Cref{fig:spectrum} depicts this slow convergence for the spectrum of $\Sigmahat$ derived from a pretrained ResNet-34 on CIFAR-100. Similar conclusions hold for other models and datasets---see \Cref{sec:additional}. We now discuss the consequences.

First, the slow convergence of $\Sigmahat$ makes it hard to empirically estimate quantities that depend on the spectrum of $\Sigma$, such as $\OmniRisk^\lambda$; \citet{loureiro21learning} and \citet{simon21blog} both note this challenge. Moreover, as shown in \Cref{fig:spectrum}, trends for eigenvalue decay extrapolated from $\Sigmahat$ may not hold for $\Sigma$. This can be problematic for estimating scaling law rates \cite{bahri21explaining, cui21generalization}.

The slow convergence also hurts analyses that rely on the approximation $\Sigmahat\approx\Sigma$, e.g.~those of of \citet{hsu14ridge} and \citet{bach21learning} for ridge regression: the assumptions needed to derive $\smash{\Sigmahat}\approx\Sigma$ would also imply $\TrainRisk(\hat\beta)\approx\smash{\Risk(\hat\beta)}$ (see \Cref{sec:classical}), which we know does not hold (see \Cref{fig:train-test}). Therefore, we do not have $\smash{\Sigmahat}\approx\Sigma$ in the manner needed for such analyses to apply.

\subsection{Kernel Regression Is Effectively Noiseless} \label{sec:challenges/noise}

Many works (e.g., \citet{belkin18understand, bartlett20benign}) have sought to explain the finding of \citet{zhang17understanding} that large models can generalize despite being able to interpolate random labels, and thus focus on analyzing overfitting with label noise. However, high-dimensional phenomena occur even on nearly noiseless datasets like CIFAR-10. We now discuss how label noise is unnecessary in a stronger sense: in high dimensions, any noisy instance of linear regression is \emph{indistinguishable} from a noiseless instance with a complex ground truth.

To show this, we embed linear regression with noisy labels into the noiseless model of \Cref{sec:model} by constructing for each noisy instance a sequence of noiseless instances that approximate it. We sketch the construction here, and present it in full in \Cref{sec:noise}.
Suppose that $y = \beta^\transp x + \xi$, where $\xi$ represents mean-zero noise. We rewrite $y$ as $y = \beta'^\transp x'$, where \smash{$x' = \begin{bsmallmatrix}x \\ t^{1/2}\xi\end{bsmallmatrix}$}, \smash{$\beta' = \begin{bsmallmatrix}\smash{\beta}\vphantom{t} \\ t^{-1/2}\end{bsmallmatrix}$}, and $t > 0$. As $t\to 0$, ridge regression on the ``augmented'' covariates $x'$ converges \emph{uniformly} over all $\lambda\ge 0$ to ridge regression on the original covariates $x$. The original, noisy instance is thus the limit of a sequence of noiseless instances.\footnote{In \Cref{sec:noise}, we show that the same reduction applies to misspecified problems. As an application, we additionally show how terms for variance from previous works can be read off of \eqref{eq:population}.}

This discussion suggests noiseless regression (allowing for $\beta$ of large norm) can capture our empirical setting, whereas analyses that require label noise may not directly apply.

\begin{figure*}
    \centering
    \includegraphics[width=2.05\columnwidth]{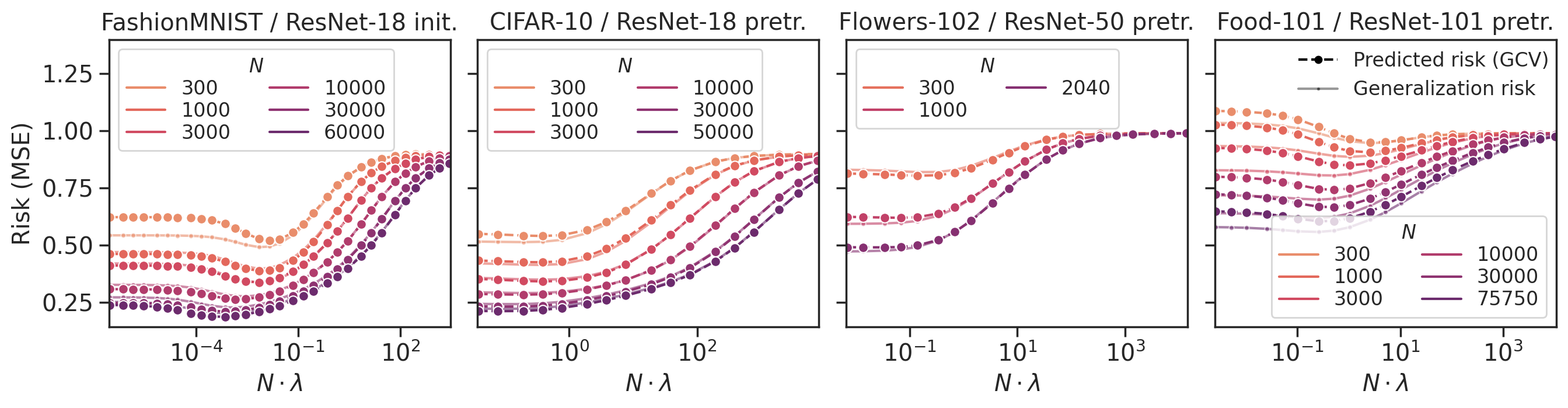}
    \caption{Generalization risk vs.\ the GCV prediction, for various datasets and networks, across sample sizes $N$ and regularization levels $\lambda$.}
    \label{fig:curves}
\end{figure*}

\section{Empirically Evaluating GCV}\label{sec:empirics}

Having demonstrated some of the challenges that our empirical setting poses for typical theories, we now empirically show that the GCV estimator,
\[ 
\GCV_\lambda = {\left(\frac 1\n\sum_{i=1}^\n \frac \lambda{\lambda + \hat\lambda_i}\right)^{-2}}\TrainRisk(\hat\beta_\lambda),
\]
accurately predicts the generalization risk. In \Cref{sec:isolation}, we first study the GCV estimator in isolation, following the setup in \Cref{sec:setup}. We observe excellent agreement between the predicted and actual generalization risks across a wide range of dataset sizes $\n$ and regularization strengths $\lambda$. In \Cref{sec:comparison}, we then quantitatively compare the GCV estimator against both norm- and spectrum-based measures of generalization, and find that GCV both has better correlation with the actual generalization risk and better predicts the asymptotic scaling. %

\subsection{The Predictive Ability of GCV}\label{sec:isolation}

To evaluate the GCV estimator, we compute an \NTK{} for each model-dataset pair listed in \Cref{tab:main}. For each \NTK{}, we then compare the GCV estimate to the actual generalization risk over a wide range of dataset sizes $\n$ and regularization levels $\lambda$. Full details of the experimental setup are provided in \Cref{sec:details}. Furthermore, in \Cref{sec:additional}, we run the same experiment for ridge regression on last layer activations.

The results of this experiment are plotted in \Cref{fig:train-test,,fig:curves}.
All curves demonstrate significant agreement between the predicted and actual empirical risks, with over 90\% of all predictions having at most 0.09 error in both relative and absolute terms.
For most instances, the predictions of GCV are nearly perfect for large $\n\lambda$ and only diverge slightly for small $\n\lambda$.
Importantly, the predictions are accurate in two regimes: (i) when mean-squared error is minimized,
and (ii) beyond the ``classical'' regime (i.e., even when the train-test gap is large). Finally, observe that predictions for fixed $\n\lambda$ tend to improve as $\n$ increases, suggesting convergence in the large $\n$ limit.

\subsection{Comparison to Alternate Approaches}\label{sec:comparison}

\begin{table*}[t]
    \centering
    \begin{tabular}{ccccccccc} \toprule
    \multirow{2}{*}{Configuration} & \multicolumn{2}{c}{\color{darkgray}\em Ground truth} & \multicolumn{2}{c}{GCV} & \multicolumn{2}{c}{Spectrum-only} & \multicolumn{2}{c}{Norm-based} \\ \cmidrule(lr){2-3} \cmidrule(lr){4-5} \cmidrule(lr){6-7} \cmidrule(lr){8-9}
    & {\color{darkgray}$r$} & {\color{darkgray}$\alpha$} & $r$ & $\alpha$ & $r$ & $\alpha$ & $r$ & $\alpha$ \\ \midrule
        Fashion-MNIST / ResNet-18 init.   &  {\color{darkgray}$1.000$}  &  {\color{darkgray}$-0.166$}  &  $0.996$  &  $-0.192$  &  $0.080$  &  $-0.008$  &  $-0.584$  &  $0.121$  \\
        CIFAR-10 / ResNet-18 pretr.       &  {\color{darkgray}$1.000$}  &  {\color{darkgray}$-0.162$}  &  $0.999$  &  $-0.182$  &  $0.977$  &  $-0.134$  &  $-0.641$  &  $0.044$  \\
        CIFAR-100 / ResNet-34 pretr.      &  {\color{darkgray}$1.000$}  &  {\color{darkgray}$-0.124$}  &  $0.996$  &  $-0.124$  &  $0.846$  &  $-0.070$  &  $-0.507$  &  $0.166$  \\
        Flowers-102 / ResNet-50 pretr.    &  {\color{darkgray}$1.000$}  &  {\color{darkgray}    --- }  &  $0.999$  &      ---   &  $0.665$  &      ---   &  $-0.786$  &     ---   \\
        Food-101 / ResNet-101 pretr.      &  {\color{darkgray}$1.000$}  &  {\color{darkgray}$-0.099$}  &  $0.979$  &  $-0.085$  &  $0.718$  &  $-0.035$  &  $-0.483$  &  $0.188$  \\ \bottomrule
    \end{tabular}
    \caption{The $r$ columns display the correlations of each prediction to generalization risk, and the $\alpha$ columns display the estimated scaling exponents. We do not run the scaling experiment for Flowers-102 because it only consists of 2040 images.}
    \label{tab:main}
\end{table*}

We next use the same setup to compare GCV against two alternative measures, based on the norm of $\hat\beta$ and the spectrum of $\Sigmahat$, respectively.  %

As discussed in \Cref{sec:norm}, the norm-based approach gives bounds of the form $\norm{\hat\beta}_2/\sqrt{\n}$.
Thus, we consider the estimate \smash{$\NormRisk^\lambda\coloneqq\norm{\hat\beta_\lambda}_2 / \sqrt{\n}$} in our experiments. More general norm-based quantities have been proposed to bound the generalization risk of neural networks (see, e.g., \citet{jiang20fantastic}); however, when specialized to linear models, these bounds simply become increasing functions of $\norm{\hat\beta}_2/\sqrt{\n}$.

For our spectrum-only estimate, we use a precise estimate %
of generalization risk in terms of ``effective dimension'' quantities \cite{zhang05learning} when $\beta$ is drawn from an isotropic prior.
We consider, for $\hat\kappa\coloneqq\bigl(\tfrac1N \textstyle\sum_{i=1}^\n (\lambda + \hat\lambda_i)^{-1}\smash{\bigr)^{-1}}$, the family
\begin{align*}
\SpecRisk^{\alpha,\sigma,\lambda}\coloneqq \hat\kappa^2 \paren*{\alpha^2\sum_{i=1}^\n \frac{\hat\lambda_i}{(\lambda + \hat\lambda_i)^2} + \frac{\sigma^2}{\n}\sum_{i=1}^\n\frac{1}{(\lambda + \hat\lambda_i)^2}}
\end{align*}
of estimates
derived from the main theorem of \citet{dobriban18high}.\footnote{See \Cref{sec:isotropic} for a derivation of this estimator.}
We fit $\alpha^2$ and $\sigma^2$ so that the predictions best match the observed generalization risks, obtaining an upper bound on the performance of this method over all $\alpha$ and $\sigma$.
This family of estimators lets us explore whether %
naturally-occurring data can be summarized by the two parameters of ``signal strength'' $\alpha$ and ``noise level'' $\sigma$.

To evaluate the ability of each predictor to model generalization, we consider two benchmarks. First, we measure the {correlation} between the predictions and the generalization risk for each dataset on the sets of $(\n, \lambda)$ pairs shown in \Cref{fig:train-test,,fig:curves}. Correlation lets us equitably compare un-scaled predictors, such as $\NormRisk$, to more precise estimates, such as $\GCV$ and $\SpecRisk$. Second, we test how well these estimators predict the scaling of optimally tuned ridge regression. For this, we find an optimal ridge parameter $\lambda^*_\n$ for each $\n$ and then estimate the power law rate (given by $\n^{-\alpha}$ for some $\alpha > 0$) of predicted generalization risk with respect to the sample size $\n$. (Applied to the ground truth, this would yield the scaling rate of the model.) %
Full details are provided in \Cref{sec:details}.

The results of these experiments are displayed in \Cref{tab:main}. Plots of the spectrum- and norm-based predictions are also presented in \Cref{sec:additional}. We find, perhaps unsurprisingly in light of \Cref{sec:norm}, that the norm-based measure has the \emph{wrong sign} when predicting generalization, both in terms of correlation and in terms of scaling.\footnote{This cannot be explained by excess regularization reducing the norm while also making performance worse: \Cref{fig:norm-growth} shows that the trend points the wrong way even when $\lambda=0$.} The spectrum-only approach also struggles to accurately predict generalization risk: it does not predict any scaling on Fashion-MNIST and achieves much lower correlations across the board. Finally, we observe that GCV correlates well with the actual generalization risks and accurately predicts scaling behavior on all datasets.

\section{A Random Matrix Perspective on GCV}\label{sec:theory}

We next justify the impressive empirical performance of GCV with a theoretical analysis. We prove a \emph{non-asymptotic} bound on the absolute error $\abs[\big]{\GCV_\lambda - \Risk(\hat\beta_\lambda)}$ of GCV under a random matrix hypothesis.

Our analysis of the GCV estimator has the following features: (i) It holds even for $\beta$ with large norm, requiring only a bound on $\E_{x\sim\Dist}\big[(\beta^\transp x)^2\big] = \beta^\transp\Sigma\beta$.
This is important for our empirical setting because, while $\norm{\beta}_2$ may be large (as discussed in \Cref{sec:norm}), the fact that our labels are $1$-hot implies $(\beta^\transp x)^2\le 1$ always. (ii) It is the first, to our knowledge, non-asymptotic analysis of GCV that applies beyond the ``classical'' regime, holding even when the train-test gap is large. (iii) It makes \emph{no additional assumptions} beyond a generic random matrix hypothesis and thus makes clear the connection between the GCV estimator and random matrix effects.\footnote{This random matrix hypothesis is known to hold for commonly considered random matrix models \cite{knowles17anisotropic} and is believed to hold even more broadly.} In particular, we do not make further assumptions about independence, moments, or dimensional ratio.

To illustrate the main technical ideas, we outline our theoretical approach at a high level in the remainder of this section and defer our formal treatment to \Cref{sec:proofs}. %

\subsection{The Random Matrix Hypothesis}

We assume a local version of the Marchenko-Pastur law as our random matrix hypothesis.
To state this hypothesis, we %
first define $\kappa = \kappa(\lambda, \n)$ as %
the (unique) positive solution to
\begin{equation}\label{eq:kappa}
 1 = \frac{\lambda}{\kappa} + \frac 1\n\sum_{i=1}^\p \frac{\lambda_i}{\kappa + \lambda_i},
\end{equation}
with $\kappa(0, \n)\coloneqq \lim_{\lambda\to 0^+} \kappa(\lambda, \n)$. We call $\kappa$ the \emph{effective regularization}, as it captures the combined effect of the explicit regularization $\lambda$ and the ``implicit regularization'' \cite{neyshabur2017implicit, jacot2020implicit} of ridge regression.
In terms of $\kappa$, the Marchenko-Pastur law can be roughly thought of as the statement
$\lambda(\lambda I + \smash{\Sigmahat})^{-1}\approx \kappa(\kappa I + \Sigma)^{-1}$.
We assume this approximation holds in the following sense:

\begin{hypothesis}[Marchenko-Pastur law over $\R_{>0}$, informal]\label{hypothesis:local-mp}
The \emph{local Marchenko-Pastur law} holds over $S\subseteq\R_{> 0}$ if,
for every deterministic $v\in\R^\p$ such that $v^\transp\Sigma v\le 1$, the following hold uniformly over all $\lambda\in S$:
\begin{align}
\frac{1}{\n}\sum_{i=1}^{\n}&\frac{1}{\hat\lambda_i + \lambda} \approx \frac{1}{\kappa} \label{eq:local-mp-A} \\
v^\transp \lambda \paren[\big]{\lambda I + \Sigmahat}^{-1}& v \approx v^\transp \kappa\paren[\big]{\kappa I + \Sigma}^{-1} v. \label{eq:local-mp-B}
\end{align}
\end{hypothesis}

\Cref{hypothesis:local-mp} is known to hold when $x$ is a linear function of independent (but not necessarily i.i.d.) random variables \cite{knowles17anisotropic}, which includes Gaussian covariates as a special case.
\Cref{hypothesis:local-mp} is expected, in fact, to hold in even greater generality, as an instance of the universality phenomenon for random matrices. %

While one cannot verify \Cref{hypothesis:local-mp} directly, since it depends on the unknown quantities $\Dist$ and $\beta$, we present evidence for its empirical validity in \Cref{sec:verifying-mp}. Specifically, we verify that \eqref{eq:local-mp-A} and \eqref{eq:local-mp-B} are consistent with each other in our empirical setting, by checking the relationships that they predict between empirically measurable quantities.

\subsection{The GCV Theorem}

We show the following error bound for the GCV estimator, which states that $\GCV_\lambda$ accurately predicts generalization risk under \Cref{hypothesis:local-mp} over a wide range of $\n$ and $\lambda$. Our bounds are stated under the normalizations $\E\brack[\big]{y^2}\le 1$ and $\E\brack[\big]{\norm{x}_2^2}\le 1$.
\begin{theorem}[Informal]\label{theorem:gcv}
Suppose \Cref{hypothesis:local-mp} holds over $S = (\frac 12\lambda, \frac 32\lambda)$. Then, for any $\epsilon > 0$,
\[ \abs*{\GCV_\lambda - \Risk(\hat\beta_\lambda)}\lesssim \frac{1}{\n^{\frac 12 - o(1)}}\cdot\frac{1}{\lambda}. \]
\end{theorem}

To prove \Cref{theorem:gcv}, our first step is to show
\[ \GCV_\lambda\approx\OmniRisk^\lambda. \]
We then prove a sharpened version of the result of \citet{hastie21surprises} to show that, if $\E\brack[\big]{y^2} = \beta^\transp\Sigma\beta\le 1$, then
\[ \OmniRisk^\lambda\approx\Risk(\hat\beta_\lambda). \]
The first step can be stated as follows.
\begin{proposition}[Informal]\label{proposition:gcv}
Suppose \Cref{hypothesis:local-mp} holds over $S = (\frac 12\lambda, \frac 32\lambda)$. Then, for any $\epsilon > 0$,
\[ \abs*{\GCV_\lambda - \OmniRisk^\lambda}\lesssim\frac{1}{\n^{1/2-o(1)}}\cdot\paren*{1 + \frac{1}{\paren{\n\lambda}^{3/2}}}. \]
\end{proposition}

For intuition, we give a heuristic proof of \Cref{proposition:gcv}. As a simplification, we use the approximate equalities $\approx$ in \Cref{hypothesis:local-mp} instead of precise error bounds. We further assume that $\approx$ is preserved by differentiation. We justify these approximations in our full analysis in \Cref{sec:proofs}.

\begin{proof}[Heuristic proof.]
By the closed form of $\TrainRisk(\hat\beta_\lambda)$, 
\small
\[ 
  \GCV_\lambda =
  {\paren*{\frac 1\n\sum_{i=1}^\n \frac{1}{\lambda+\hat\lambda_i}}^{\!\!-2}}
  \!\beta^\transp(\Sigmahat + \lambda I)^{-1}\Sigmahat(\Sigmahat + \lambda I)^{-1}\beta.
\]
\normalsize
\Cref{hypothesis:local-mp} implies
\smash{$\paren[\big]{\smash{\frac 1\n\sum_{i=1}^\n} \paren{\lambda+\hat\lambda_i}^{-1}}^{-2} \approx \kappa^2$}
and
\begin{equation}\label{eq:differentiation}
  \frac{\partial}{\partial\lambda}\paren{\beta^\transp \lambda (\Sigmahat + \lambda I)^{-1} \beta}
  \approx \frac{\partial}{\partial\lambda}\paren*{\beta^\transp \kappa (\Sigma + \kappa I)^{-1} \beta},
\end{equation}
assuming we may differentiate through the $\approx$.
Hence,
\begin{align*}
  \beta^\transp(\Sigmahat + \lambda I)^{-1}&\Sigmahat(\Sigmahat + \lambda I)^{-1}\beta \\
  &= \frac{\partial}{\partial\lambda}\paren{\beta^\transp \lambda (\Sigmahat + \lambda I)^{-1} \beta} \\
  &\approx \frac{\partial}{\partial\lambda}\paren*{\beta^\transp \kappa (\Sigma + \kappa I)^{-1} \beta} \\
  &= \frac{\partial\kappa}{\partial\lambda}\cdot\beta^\transp(\Sigma + \kappa I)^{-1}\Sigma(\Sigma + \kappa I)^{-1}\beta.
\end{align*}
Substituting into the equation for $\GCV_\lambda$, we obtain
\begin{align*}
  \GCV_\lambda
  &\approx \kappa^2\paren*{\frac{\partial\kappa}{\partial\lambda}\cdot\beta^\transp(\Sigma + \kappa I)^{-1}\Sigma(\Sigma + \kappa I)^{-1}\beta} \\
  &= \frac{\partial\kappa}{\partial\lambda}\cdot\kappa^2\sum_{i=1}^\p \left(\frac{\lambda_i}{(\kappa + \lambda_i)^2} \bigl(\beta^\transp v_i\bigr)^2 \right) \\
  &= \OmniRisk^\lambda. \qedhere
\end{align*}
\end{proof}

\section{Pretraining and Scaling Laws through a Random Matrix Lens}\label{sec:applications}

Having shown that a random matrix approach can fruitfully model generalization risk both in theory and in practice, we apply this theory towards answering: what factors determine whether a neural representation scales well when applied to a downstream task? To answer this question, we revisit the omniscient risk estimate,
\[ \OmniRisk^\lambda = \frac{\partial\kappa}{\partial\lambda}\cdot\kappa^2\sum_{i=1}^\p \biggl(\frac{\lambda_i}{(\kappa + \lambda_i)^2} \bigl(\beta^\transp v_i\bigr)^2 \biggr), \]
a quantity that depends on the eigenvalues $\lambda_i$ and the \emph{alignment coefficients} $(\beta^\transp v_i)^2$ between the eigenvectors and $\beta$. %

In \Cref{sec:pretraining}, we use eigendecay and alignment to understand why pretrained representations generalize better than randomly initialized ones. We find, perhaps unintuitively, that pretrained representations have \emph{slower} eigenvalue decay (and thus \emph{higher} effective dimension), but nonetheless scale better due to better alignment between the eigenvectors and the ground truth. Thus, it is necessary to consider alignment in addition to eigenvalue decay to explain the effectiveness of pretraining.

Motivated by this, in \Cref{sec:scaling-laws}, we study scaling laws for the eigendecay and alignment coefficients \cite{caponnetto07optimal, cui21generalization}. We show how to estimate their power law exponents in terms of empirically observable quantities. Combining these yields an empirically accurate estimate of the power law exponent of generalization, suggesting that eigendecay and alignment are \emph{sufficient} statistics for predicting scaling. %

\subsection{Pretraining}\label{sec:pretraining}

\begin{figure}
    \centering
    \includegraphics[width=\columnwidth]{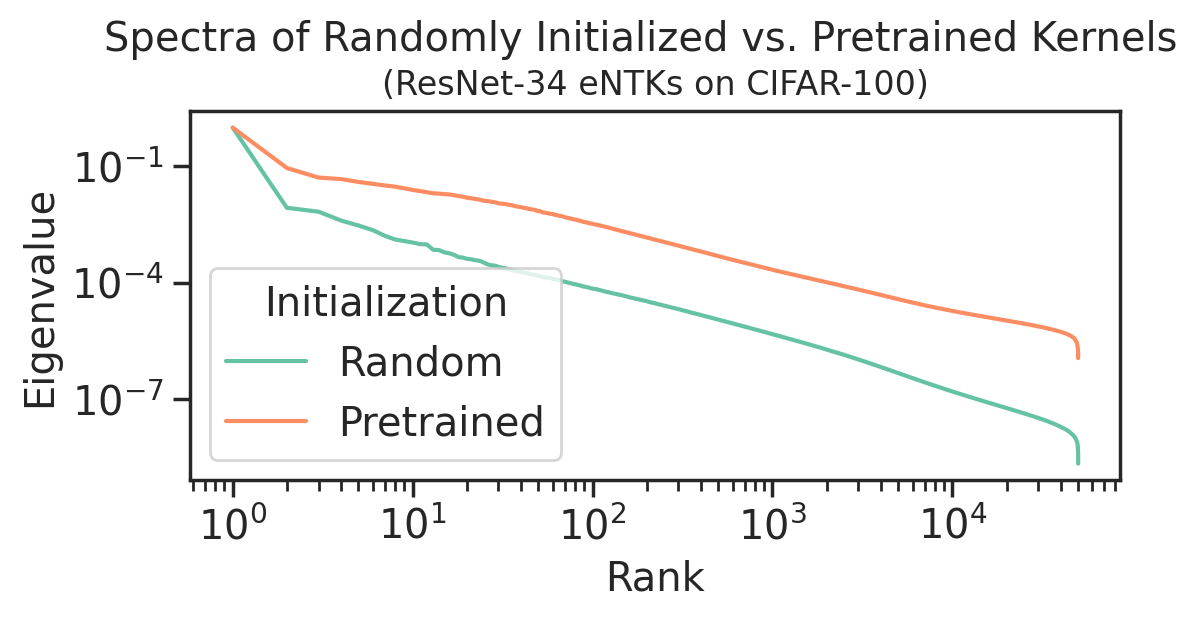}
    \caption{The pairs $(i, \hat\lambda_i)$ plotted for two ResNet-34 \NTK{}s: one at random initialization and one after pretraining. Note that the pretrained kernel has \emph{higher} effective dimension.}
    \label{fig:pretraining}
\end{figure}

A common intuition is that pretraining equips models with simple, ``low-dimensional'' representations of complex data.
Thus, one might expect that pretrained representations have lower effective dimension and that this is the cause of better generalization. \Cref{fig:pretraining}, however, shows the opposite to be true: on CIFAR-100, a pretrained ResNet-34 \NTK{} has \emph{slower} eigenvalue decay than a randomly initialized representation (and higher effective dimension). Moreover, this holds consistently across datasets and models, as shown in \Cref{sec:additional}. Thus, dimension alone cannot explain the benefit of pretraining.

The omniscient risk estimate $\OmniRisk^\lambda$ suggests a possible remedy.
While slower eigendecay will increase $\OmniRisk^\lambda$, the increase can be overcome if the alignment coefficients $(\beta^\transp v_i)^2$ decay faster. %
We will confirm this in \Cref{sec:scaling-laws} once we develop tools to estimate the decay rates of the eigenvalues and the alignment coefficients: across several models and datasets, pretrained representations exhibit slower eigendecay but better alignment (\Cref{tab:powerlaws}).
Our finding suggests that the role of pretraining is to make ``likely'' ground truth functions easily representable and in fact does not reduce data dimensionality.
In particular, the covariates cannot be considered in isolation from potential downstream tasks.

\subsection{Scaling Laws}\label{sec:scaling-laws}

\begin{table}
    \centering
    \resizebox{\columnwidth}{!}{%
    \begin{tabular}{ccccc}\toprule
        Configuration & $\hat\gamma$ & $\hat\delta$ & $\hat\alpha$ & $\alpha$ \\ \midrule
        {\footnotesize F-MNIST / ResNet-18 init.   }  & $0.657$ & $-0.462$ & $0.195$ & $0.166$ \\
        {\footnotesize F-MNIST / ResNet-18 pretr.  }  & $0.353$ & $-0.149$ & $0.204$ & $0.188$ \\
        {\footnotesize CIFAR-10 / ResNet-18 init.  }  & $0.535$ & $-0.468$ & $0.066$ & $0.059$ \\
        {\footnotesize CIFAR-10 / ResNet-18 pretr. }  & $0.270$ & $-0.089$ & $0.181$ & $0.162$ \\
        {\footnotesize CIFAR-100 / ResNet-34 init. }  & $0.482$ & $-0.466$ & $0.016$ & $0.014$ \\
        {\footnotesize CIFAR-100 / ResNet-34 pretr.}  & $0.257$ & $-0.128$ & $0.128$ & $0.124$ \\
        {\footnotesize Food-101 / ResNet-101 pretr.}  & $0.200$ & $-0.113$ & $0.087$ & $0.985$ \\ \bottomrule
    \end{tabular}%
    }
    \caption{The first two columns display the estimated power law rates $\hat\gamma$ (of eigendecay) and $\hat\delta$ (of alignment). The last two columns compare the estimate $\hat\alpha\coloneqq\hat\gamma + \hat\delta$ for the scaling rate of optimally tuned ridge regression against the actual scaling rate $\alpha$ of \smash{$\Risk(\hat\beta_{\lambda^*})$}.}
    \label{tab:powerlaws}
\end{table}

The omniscient risk estimate $\OmniRisk^\lambda$ shows that both alignment and eigendecay matter for generalization. To better understand the behavior of these quantities, which are given in terms of the unobserved $\Sigma$ and $\beta$, we show how the power law rates of these terms can be estimated from empirically observable quantities. We then use these rates to estimate the scaling law rate of the generalization error for optimally regularized ridge regression. We find the estimated rates accurately reflect observed scaling behavior, suggesting that power law models of alignment and eigendecay suffice to capture the scaling behavior of regression on natural data.

In this section, we suppose that the population eigenvalues and the alignment coefficients scale as
\[ \lambda_i\asymp i^{-1-\gamma} \quad\text{and}\quad (\beta^\transp v_i)^2\asymp i^{-\delta}, \]
for $\gamma > 0$ and $\delta < 1$. (Note that the latter implies $\norm{\beta}_2$ is effectively infinite when $\n\ll\p$.)
Assuming known $\gamma$ and $\delta$, \citet{cui21generalization} analyze \smash{$\OmniRisk^{\lambda^*}$}, for $\lambda^*$ the optimal ridge regularization, and show in the noiseless regime that
\begin{equation}\label{eq:cui}
\OmniRisk^{\lambda^*}\asymp N^{-\alpha},\quad\text{for $\alpha = \gamma + \delta$.}
\end{equation}
However, they do not give a satisfactory way to estimate $\gamma$ and $\delta$ from data: they propose simply using the eigenvalues $\Sigmahat$ as a proxy for those of $\Sigma$. But as we previously observed in \Cref{fig:spectrum}, convergence of $\Sigmahat$ to $\Sigma$ can be slow for high-dimensional regression problems.

The following propositions (proven in \Cref{sec:proofs-2}) provide a more principled way to estimate $\gamma$ and $\delta$ in terms of empirically observable quantities, using the same random matrix hypothesis from before.
\begin{proposition}\label{proposition:A}
Suppose that \Cref{hypothesis:local-mp} holds as $\lambda \to 0$ and that $\lambda_i\asymp i^{-1-\gamma}$. Then,
\[ \n^{-1}\Tr\bigl((XX^\transp)^{-1}\bigr)\asymp \n^{\gamma}. \]
\end{proposition}

\begin{proposition}\label{proposition:B}
Suppose that \Cref{hypothesis:local-mp} holds at $\lambda > 0$ and that $\lambda_i\asymp i^{-1-\gamma}$ and $(\beta^\transp v_i)^2\asymp i^{-\delta}$. Then,
\begin{equation}\label{eq:delta}
y^\transp\paren*{XX^\transp + \n\lambda I}^{-1} y\asymp\kappa(\lambda, \n)^{-\frac{1 - \delta}{1 + \gamma}}.
\end{equation}
\end{proposition}
Consequently, $\gamma$ can be estimated by fitting the slope of the points $\bigl(\log\n, \log\bigl(\n^{-1}\Tr((XX^\transp)^{-1})\bigr)\bigr)\in\R^2$. And $\delta$ can be estimated by inverting \eqref{eq:delta} and applying the estimate $\eqref{eq:local-mp-A}$ for $\kappa$ and the preceding estimate for $\gamma$.

To test this approach, we estimate $\gamma$ and $\delta$ as $\hat\gamma$ and $\hat\delta$ via the quantities in \Cref{proposition:A,,proposition:B} and apply these estimates to the datasets listed in \Cref{tab:powerlaws}. We also estimate $\hat\alpha = \hat\gamma + \hat\delta$ following \eqref{eq:cui} and compare $\hat\alpha$ to the actual rate $\alpha$ in \Cref{tab:powerlaws}.

We find that $\hat\alpha$ accurately approximates $\alpha$ for all datasets, suggesting that the power law assumption can be used to model naturally-occurring data. Additionally, we observe for all datasets that $\smash{\hat\delta} < 0$, which suggests that the coefficients $(\beta^\transp v_i)^2$ \emph{grow} in $i$, reinforcing our conclusion from \Cref{sec:challenges} that $\beta$ has large norm. Finally, for all pairs of randomly initialized and pretrained models in \Cref{tab:powerlaws}, note that the pretrained model has smaller $\gamma$ and thus slower eigenvalue decay, but much larger $\delta$. This verifies our hypothesis that pretrained representations scale better due to improved alignment (and despite higher dimension).

\section{Discussion} \label{sec:conclusion}

In this paper, we identify that the GCV estimator accurately predicts the generalization of ridge regression on neural representations of large-scale networks and real data, while other more classical approaches fall short. We then elucidate the connection between GCV and random matrix laws, showing that GCV accurately predicts generalization risk whenever a local Marchenko-Pastur law holds. Finally, we find that this perspective lets us answer basic conceptual questions about neural representations.
Our findings suggest several promising directions for future inquiry, which we now discuss.

First, we believe that the random matrix approach has much more to offer towards understanding the statistics of high-dimensional learning: the structure imposed by a random matrix assumption stood apart at capturing the qualitative phenomena of ridge regression. It is thus conceivable that such structure will be \emph{necessary} to understand settings beyond ridge regression, e.g., classification accuracy for logistic regression or modeling natural covariate shifts.

However, there remain open problems even in the setting of ridge regression. For instance, current understanding of random matrix laws does not encompass all the regimes of interest: a natural scaling of regularization is $\lambda\asymp\n^{-1}$, but the existing theory \cite{knowles17anisotropic} requires $\lambda$ to be bounded away from $0$. Additionally, it would be of interest to achieve a bound on the error of $\OmniRisk^\lambda$ that scales well in the $\lambda\asymp\n^{-1}$ limit (like what we have for \Cref{proposition:gcv}).

Finally, and most broadly, we hope that the perspective we take towards studying neural representations can inspire more insight towards what is learned by neural networks. We find that eNTK representations reveal much more than the typically considered final-layer activations and serve as a reasonable proxy for understanding finetuning on a pretrained model. Can eNTKs be used as a tractable model to untangle more of the mysteries around large-scale models? For instance, what do \NTK{}s reveal about features learned via different training procedures? And can the evolution of the \NTK{} and its associated metrics (e.g., eigendecay and alignment) during training shed light on feature learning?

\section*{Acknowledgements}

We would like to thank Yasaman Bahri, Karolina Dziugaite, Preetum Nakkiran, and Nilesh Tripuraneni for their valuable feedback. Alexander Wei acknowledges support from an NSF Graduate Research Fellowship.

\bibliography{main}

\begin{thebibliography}{63}
\providecommand{\natexlab}[1]{#1}
\providecommand{\url}[1]{\texttt{#1}}
\expandafter\ifx\csname urlstyle\endcsname\relax
  \providecommand{\doi}[1]{doi: #1}\else
  \providecommand{\doi}{doi: \begingroup \urlstyle{rm}\Url}\fi

\bibitem[Adlam \& Pennington(2020)Adlam and Pennington]{adlam20ntk}
Adlam, B. and Pennington, J.
\newblock The neural tangent kernel in high dimensions: Triple descent and a
  multi-scale theory of generalization.
\newblock In \emph{Proceedings of the 37th International Conference on Machine
  Learning}, pp.\  74--84, 2020.

\bibitem[Arora et~al.(2019)Arora, Du, Hu, Li, Salakhutdinov, and
  Wang]{arora19exact}
Arora, S., Du, S.~S., Hu, W., Li, Z., Salakhutdinov, R., and Wang, R.
\newblock On exact computation with an infinitely wide neural net.
\newblock In \emph{Advances in Neural Information Processing Systems 32}, pp.\
  8139--8148, 2019.

\bibitem[Bach(2023)]{bach21learning}
Bach, F.
\newblock \emph{Learning Theory from First Principles}.
\newblock MIT, 2023.

\bibitem[Bahri et~al.(2021)Bahri, Dyer, Kaplan, Lee, and
  Sharma]{bahri21explaining}
Bahri, Y., Dyer, E., Kaplan, J., Lee, J., and Sharma, U.
\newblock Explaining neural scaling laws.
\newblock \emph{arXiv}, abs/2102.06701, 2021.

\bibitem[Bai \& Silverstein(2010)Bai and Silverstein]{bai2010spectral}
Bai, Z. and Silverstein, J.~W.
\newblock \emph{Spectral Analysis of Large Dimensional Random Matrices}.
\newblock Springer, 2010.

\bibitem[Bartlett(1996)]{bartlett96size}
Bartlett, P.~L.
\newblock For valid generalization the size of the weights is more important
  than the size of the network.
\newblock In \emph{Advances in Neural Information Processing Systems 9}, pp.\
  134--140, 1996.

\bibitem[Bartlett \& Mendelson(2001)Bartlett and
  Mendelson]{bartlett01rademacher}
Bartlett, P.~L. and Mendelson, S.
\newblock Rademacher and gaussian complexities: Risk bounds and structural
  results.
\newblock In \emph{Computational Learning Theory, 14th Annual Conference on
  Computational Learning Theory}, pp.\  224--240, 2001.

\bibitem[Bartlett et~al.(2002)Bartlett, Bousquet, and
  Mendelson]{bartlett02localized}
Bartlett, P.~L., Bousquet, O., and Mendelson, S.
\newblock Localized rademacher complexities.
\newblock In \emph{Computational Learning Theory, 15th Annual Conference on
  Computational Learning Theory}, pp.\  44--58, 2002.

\bibitem[Bartlett et~al.(2020)Bartlett, Long, Lugosi, and
  Tsigler]{bartlett20benign}
Bartlett, P.~L., Long, P.~M., Lugosi, G., and Tsigler, A.
\newblock Benign overfitting in linear regression.
\newblock \emph{Proceedings of the National Academy of Sciences}, 117\penalty0
  (48):\penalty0 30063--30070, 2020.

\bibitem[Bates et~al.(2021)Bates, Hastie, and
  Tibshirani]{bates21crossvalidation}
Bates, S., Hastie, T., and Tibshirani, R.
\newblock Cross-validation: what does it estimate and how well does it do it?
\newblock \emph{arXiv}, abs/2104.00673, 2021.

\bibitem[Belkin et~al.(2018)Belkin, Ma, and Mandal]{belkin18understand}
Belkin, M., Ma, S., and Mandal, S.
\newblock To understand deep learning we need to understand kernel learning.
\newblock In \emph{Proceedings of the 35th International Conference on Machine
  Learning}, pp.\  540--548, 2018.

\bibitem[Belkin et~al.(2020)Belkin, Hsu, and Xu]{belkin20two}
Belkin, M., Hsu, D., and Xu, J.
\newblock Two models of double descent for weak features.
\newblock \emph{{SIAM} J. Math. Data Sci.}, 2\penalty0 (4):\penalty0
  1167--1180, 2020.

\bibitem[Bloemendal et~al.(2016)Bloemendal, Knowles, Yau, and
  Yin]{bloemendal2016principal}
Bloemendal, A., Knowles, A., Yau, H.-T., and Yin, J.
\newblock On the principal components of sample covariance matrices.
\newblock \emph{Probability theory and related fields}, 164\penalty0
  (1):\penalty0 459--552, 2016.

\bibitem[Bossard et~al.(2014)Bossard, Guillaumin, and Gool]{bossard14food}
Bossard, L., Guillaumin, M., and Gool, L.~V.
\newblock Food-101 - mining discriminative components with random forests.
\newblock In \emph{Proceedings of the Thirteenth European Conference on
  Computer Vision}, pp.\  446--461, 2014.

\bibitem[Canatar et~al.(2021)Canatar, Bordelon, and
  Pehlevan]{canatar21spectral}
Canatar, A., Bordelon, B., and Pehlevan, C.
\newblock Spectral bias and task-model alignment explain generalization in
  kernel regression and infinitely wide neural networks.
\newblock \emph{Nature Communications}, 12\penalty0 (1):\penalty0 1--12, 2021.

\bibitem[Cao \& Golubev(2006)Cao and Golubev]{cao06oracle}
Cao, Y. and Golubev, Y.
\newblock On oracle inequalities related to smoothing splines.
\newblock \emph{Mathematical Methods of Statistics}, 15\penalty0 (4):\penalty0
  398--414, 2006.

\bibitem[Caponnetto \& Vito(2007)Caponnetto and Vito]{caponnetto07optimal}
Caponnetto, A. and Vito, E.~D.
\newblock Optimal rates for the regularized least-squares algorithm.
\newblock \emph{Found. Comput. Math.}, 7\penalty0 (3):\penalty0 331--368, 2007.

\bibitem[Craven \& Wahba(1978)Craven and Wahba]{craven78smoothing}
Craven, P. and Wahba, G.
\newblock Smoothing noisy data with spline functions.
\newblock \emph{Numerische Mathematik}, 31:\penalty0 377--403, 1978.

\bibitem[Cui et~al.(2021)Cui, Loureiro, Krzakala, and
  Zdeborova]{cui21generalization}
Cui, H., Loureiro, B., Krzakala, F., and Zdeborova, L.
\newblock Generalization error rates in kernel regression: The crossover from
  the noiseless to noisy regime.
\newblock In \emph{Advances in Neural Information Processing Systems 34}, 2021.

\bibitem[Dobriban \& Wager(2018)Dobriban and Wager]{dobriban18high}
Dobriban, E. and Wager, S.
\newblock High-dimensional asymptotics of prediction: Ridge regression and
  classification.
\newblock \emph{The Annals of Statistics}, 46\penalty0 (1):\penalty0 247--279,
  2018.

\bibitem[Dosovitskiy et~al.(2021)Dosovitskiy, Beyer, Kolesnikov, Weissenborn,
  Zhai, Unterthiner, Dehghani, Minderer, Heigold, Gelly, Uszkoreit, and
  Houlsby]{dosovitskiy21image}
Dosovitskiy, A., Beyer, L., Kolesnikov, A., Weissenborn, D., Zhai, X.,
  Unterthiner, T., Dehghani, M., Minderer, M., Heigold, G., Gelly, S.,
  Uszkoreit, J., and Houlsby, N.
\newblock An image is worth 16x16 words: Transformers for image recognition at
  scale.
\newblock In \emph{9th International Conference on Learning Representations},
  2021.

\bibitem[Dziugaite \& Roy(2017)Dziugaite and Roy]{dziugaite17computing}
Dziugaite, G.~K. and Roy, D.~M.
\newblock Computing nonvacuous generalization bounds for deep (stochastic)
  neural networks with many more parameters than training data.
\newblock In \emph{Proceedings of the Thirty-Third Conference on Uncertainty in
  Artificial Intelligence}, 2017.

\bibitem[Dziugaite et~al.(2020)Dziugaite, Drouin, Neal, Rajkumar, Caballero,
  Wang, Mitliagkas, and Roy]{dziugaite20search}
Dziugaite, G.~K., Drouin, A., Neal, B., Rajkumar, N., Caballero, E., Wang, L.,
  Mitliagkas, I., and Roy, D.~M.
\newblock In search of robust measures of generalization.
\newblock In \emph{Advances in Neural Information Processing Systems 33}, 2020.

\bibitem[Efron(1986)]{efron86biased}
Efron, B.
\newblock How biased is the apparent error rate of a prediction rule?
\newblock \emph{Journal of the American Statistical Association}, 81\penalty0
  (394):\penalty0 461--470, 1986.

\bibitem[Erdős \& Yau(2017)Erdős and Yau]{erdos17dynamical}
Erdős, L. and Yau, H.-T.
\newblock \emph{A Dynamical Approach to Random Matrix Theory}.
\newblock American Mathematical Society, 2017.

\bibitem[Golub et~al.(1979)Golub, Heath, and Wahba]{golub79generalized}
Golub, G.~H., Heath, M., and Wahba, G.
\newblock Generalized cross-validation as a method for choosing a good ridge
  parameter.
\newblock \emph{Technometrics}, 21\penalty0 (2):\penalty0 215--223, 1979.

\bibitem[Hastie et~al.(2020)Hastie, Montanari, Rosset, and
  Tibshirani]{hastie21surprises}
Hastie, T., Montanari, A., Rosset, S., and Tibshirani, R.~J.
\newblock Surprises in high-dimensional ridgeless least squares interpolation.
\newblock \emph{arXiv}, abs/1903.08560, 2020.

\bibitem[He et~al.(2015)He, Zhang, Ren, and Sun]{he15delving}
He, K., Zhang, X., Ren, S., and Sun, J.
\newblock Delving deep into rectifiers: Surpassing human-level performance on
  imagenet classification.
\newblock In \emph{2015 {IEEE} International Conference on Computer Vision},
  pp.\  1026--1034, 2015.

\bibitem[Hsu et~al.(2014)Hsu, Kakade, and Zhang]{hsu14ridge}
Hsu, D.~J., Kakade, S.~M., and Zhang, T.
\newblock Random design analysis of ridge regression.
\newblock \emph{Found. Comput. Math.}, 14\penalty0 (3):\penalty0 569--600,
  2014.

\bibitem[Jacot et~al.(2018)Jacot, Hongler, and Gabriel]{jacot18ntk}
Jacot, A., Hongler, C., and Gabriel, F.
\newblock Neural tangent kernel: Convergence and generalization in neural
  networks.
\newblock In \emph{Advances in Neural Information Processing Systems 31}, pp.\
  8580--8589, 2018.

\bibitem[Jacot et~al.(2020{\natexlab{a}})Jacot, Simsek, Spadaro, Hongler, and
  Gabriel]{jacot2020implicit}
Jacot, A., Simsek, B., Spadaro, F., Hongler, C., and Gabriel, F.
\newblock Implicit regularization of random feature models.
\newblock In \emph{Proceedings of the 37th International Conference on Machine
  Learning}, pp.\  4631--4640, 2020{\natexlab{a}}.

\bibitem[Jacot et~al.(2020{\natexlab{b}})Jacot, Simsek, Spadaro, Hongler, and
  Gabriel]{jacot20kernel}
Jacot, A., Simsek, B., Spadaro, F., Hongler, C., and Gabriel, F.
\newblock Kernel alignment risk estimator: Risk prediction from training data.
\newblock In \emph{Advances in Neural Information Processing Systems 33},
  2020{\natexlab{b}}.

\bibitem[Jiang et~al.(2020)Jiang, Neyshabur, Mobahi, Krishnan, and
  Bengio]{jiang20fantastic}
Jiang, Y., Neyshabur, B., Mobahi, H., Krishnan, D., and Bengio, S.
\newblock Fantastic generalization measures and where to find them.
\newblock In \emph{8th International Conference on Learning Representations},
  2020.

\bibitem[Kaplan et~al.(2020)Kaplan, McCandlish, Henighan, Brown, Chess, Child,
  Gray, Radford, Wu, and Amodei]{kaplan20scaling}
Kaplan, J., McCandlish, S., Henighan, T., Brown, T.~B., Chess, B., Child, R.,
  Gray, S., Radford, A., Wu, J., and Amodei, D.
\newblock Scaling laws for neural language models.
\newblock \emph{arXiv}, abs/2001.08361, 2020.

\bibitem[Knowles \& Yin(2017)Knowles and Yin]{knowles17anisotropic}
Knowles, A. and Yin, J.
\newblock Anisotropic local laws for random matrices.
\newblock \emph{Probability Theory and Related Fields}, 169:\penalty0 257--352,
  2017.

\bibitem[Krizhevsky(2009)]{krizhevsky09learning}
Krizhevsky, A.
\newblock Learning multiple layers of features from tiny images.
\newblock Technical report, University of Toronto, 2009.

\bibitem[Lee et~al.(2019)Lee, Xiao, Schoenholz, Bahri, Novak, Sohl{-}Dickstein,
  and Pennington]{lee19wide}
Lee, J., Xiao, L., Schoenholz, S.~S., Bahri, Y., Novak, R., Sohl{-}Dickstein,
  J., and Pennington, J.
\newblock Wide neural networks of any depth evolve as linear models under
  gradient descent.
\newblock In \emph{Advances in Neural Information Processing Systems 32}, pp.\
  8570--8581, 2019.

\bibitem[Lee et~al.(2020)Lee, Schoenholz, Pennington, Adlam, Xiao, Novak, and
  Sohl{-}Dickstein]{lee20finite}
Lee, J., Schoenholz, S.~S., Pennington, J., Adlam, B., Xiao, L., Novak, R., and
  Sohl{-}Dickstein, J.
\newblock Finite versus infinite neural networks: an empirical study.
\newblock In \emph{Advances in Neural Information Processing Systems 33}, 2020.

\bibitem[Li(1986)]{li86asymptotic}
Li, K.-C.
\newblock {Asymptotic Optimality of $C_L$ and Generalized Cross-Validation in
  Ridge Regression with Application to Spline Smoothing}.
\newblock \emph{The Annals of Statistics}, 14\penalty0 (3):\penalty0
  1101--1112, 1986.

\bibitem[Li et~al.(2019)Li, Wang, Yu, Du, Hu, Salakhutdinov, and
  Arora]{li19enhanced}
Li, Z., Wang, R., Yu, D., Du, S.~S., Hu, W., Salakhutdinov, R., and Arora, S.
\newblock Enhanced convolutional neural tangent kernels.
\newblock \emph{arXiv}, abs/1911.00809, 2019.

\bibitem[Loureiro et~al.(2021)Loureiro, Gerbelot, Cui, Goldt, Krzakala, Mezard,
  and Zdeborova]{loureiro21learning}
Loureiro, B., Gerbelot, C., Cui, H., Goldt, S., Krzakala, F., Mezard, M., and
  Zdeborova, L.
\newblock Learning curves of generic features maps for realistic datasets with
  a teacher-student model.
\newblock In \emph{Advances in Neural Information Processing Systems 34}, 2021.

\bibitem[Marchenko \& Pastur(1967)Marchenko and Pastur]{marchenkopastur}
Marchenko, V.~A. and Pastur, L.~A.
\newblock Distribution of eigenvalues for some sets of random matrices.
\newblock \emph{Matematicheskii Sbornik}, 114\penalty0 (4):\penalty0 507--536,
  1967.

\bibitem[Marquardt \& Snee(1975)Marquardt and Snee]{marquardt1975ridge}
Marquardt, D.~W. and Snee, R.~D.
\newblock Ridge regression in practice.
\newblock \emph{The American Statistician}, 29\penalty0 (1):\penalty0 3--20,
  1975.

\bibitem[McAllester(1999)]{mcallester99pac}
McAllester, D.~A.
\newblock Pac-bayesian model averaging.
\newblock In \emph{Proceedings of the Twelfth Annual Conference on
  Computational Learning Theory}, pp.\  164--170, 1999.

\bibitem[Mei \& Montanari(2020)Mei and Montanari]{mei20generalization}
Mei, S. and Montanari, A.
\newblock The generalization error of random features regression: Precise
  asymptotics and double descent curve.
\newblock \emph{arXiv}, abs/1908.05355, 2020.

\bibitem[Mel \& Ganguli(2021)Mel and Ganguli]{mel21arbitrary}
Mel, G. and Ganguli, S.
\newblock A theory of high dimensional regression with arbitrary correlations
  between input features and target functions: sample complexity, multiple
  descent curves and a hierarchy of phase transitions.
\newblock In \emph{Proceedings of the 38th International Conference on Machine
  Learning}, volume 139, pp.\  7578--7587, 2021.

\bibitem[Nagarajan \& Kolter(2019)Nagarajan and Kolter]{nagarajan19uniform}
Nagarajan, V. and Kolter, J.~Z.
\newblock Uniform convergence may be unable to explain generalization in deep
  learning.
\newblock In \emph{Advances in Neural Information Processing Systems 32}, pp.\
  11611--11622, 2019.

\bibitem[Neyshabur(2017)]{neyshabur2017implicit}
Neyshabur, B.
\newblock Implicit regularization in deep learning.
\newblock \emph{arXiv}, abs/1709.01953, 2017.

\bibitem[Neyshabur et~al.(2015)Neyshabur, Tomioka, and Srebro]{neyshabur15norm}
Neyshabur, B., Tomioka, R., and Srebro, N.
\newblock Norm-based capacity control in neural networks.
\newblock In \emph{Proceedings of The 28th Conference on Learning Theory},
  volume~40, pp.\  1376--1401, 2015.

\bibitem[Nilsback \& Zisserman(2008)Nilsback and Zisserman]{nilsback08flowers}
Nilsback, M. and Zisserman, A.
\newblock Automated flower classification over a large number of classes.
\newblock In \emph{Sixth Indian Conference on Computer Vision, Graphics {\&}
  Image Processing}, pp.\  722--729, 2008.

\bibitem[Northcutt et~al.(2021)Northcutt, Athalye, and
  Mueller]{northcutt21pervasive}
Northcutt, C.~G., Athalye, A., and Mueller, J.
\newblock Pervasive label errors in test sets destabilize machine learning
  benchmarks.
\newblock In \emph{Thirty-fifth Conference on Neural Information Processing
  Systems Datasets and Benchmarks Track}, 2021.

\bibitem[Novikoff(1962)]{novikoff62convergence}
Novikoff, A.~B.
\newblock On convergence proofs on perceptrons.
\newblock In \emph{Proceedings of the Symposium on the Mathematical Theory of
  Automata}, volume~12, pp.\  615--622. Polytechnic Institute of Brooklyn,
  1962.

\bibitem[Patil et~al.(2021)Patil, Wei, Rinaldo, and Tibshirani]{patil21uniform}
Patil, P., Wei, Y., Rinaldo, A., and Tibshirani, R.~J.
\newblock Uniform consistency of cross-validation estimators for
  high-dimensional ridge regression.
\newblock In Banerjee, A. and Fukumizu, K. (eds.), \emph{The 24th International
  Conference on Artificial Intelligence and Statistics}, pp.\  3178--3186,
  2021.

\bibitem[Richards et~al.(2021)Richards, Mourtada, and
  Rosasco]{richards21asymptotics}
Richards, D., Mourtada, J., and Rosasco, L.
\newblock Asymptotics of ridge(less) regression under general source condition.
\newblock In \emph{The 24th International Conference on Artificial Intelligence
  and Statistics}, volume 130, pp.\  3889--3897, 2021.

\bibitem[Rosset \& Tibshirani(2020)Rosset and Tibshirani]{rosset20fixed}
Rosset, S. and Tibshirani, R.~J.
\newblock From fixed-{X} to random-{X} regression: Bias-variance
  decompositions, covariance penalties, and prediction error estimation.
\newblock \emph{Journal of the American Statistical Association}, 115\penalty0
  (529):\penalty0 138--151, 2020.

\bibitem[Simon(2021)]{simon21blog}
Simon, J.~B.
\newblock A first-principles theory of neural network generalization, October
  2021.
\newblock URL \url{https://bair.berkeley.edu/blog/2021/10/25/eigenlearning/}.

\bibitem[Simon et~al.(2021)Simon, Dickens, and DeWeese]{simon21neural}
Simon, J.~B., Dickens, M., and DeWeese, M.~R.
\newblock Neural tangent kernel eigenvalues accurately predict generalization.
\newblock \emph{arXiv}, abs/2110.03922, 2021.

\bibitem[Steinhardt(2021)]{steinhardt21notes}
Steinhardt, J.
\newblock Robust and nonparametric statistics, April 2021.
\newblock URL
  \url{https://jsteinhardt.stat.berkeley.edu/teaching/stat240-spring-2021/notes.pdf}.

\bibitem[Wu \& Xu(2020)Wu and Xu]{wu20optimal}
Wu, D. and Xu, J.
\newblock On the optimal weighted $\ell_2$ regularization in overparameterized
  linear regression.
\newblock In \emph{Advances in Neural Information Processing Systems 33}, 2020.

\bibitem[Xiao et~al.(2017)Xiao, Rasul, and Vollgraf]{xiao2017fashion}
Xiao, H., Rasul, K., and Vollgraf, R.
\newblock Fashion-mnist: a novel image dataset for benchmarking machine
  learning algorithms.
\newblock \emph{arXiv}, abs/1708.07747, 2017.

\bibitem[Yang(2019)]{yang19scaling}
Yang, G.
\newblock Scaling limits of wide neural networks with weight sharing: Gaussian
  process behavior, gradient independence, and neural tangent kernel
  derivation.
\newblock \emph{arXiv}, abs/1902.04760, 2019.

\bibitem[Zhang et~al.(2017)Zhang, Bengio, Hardt, Recht, and
  Vinyals]{zhang17understanding}
Zhang, C., Bengio, S., Hardt, M., Recht, B., and Vinyals, O.
\newblock Understanding deep learning requires rethinking generalization.
\newblock In \emph{5th International Conference on Learning Representations},
  2017.

\bibitem[Zhang(2005)]{zhang05learning}
Zhang, T.
\newblock Learning bounds for kernel regression using effective data
  dimensionality.
\newblock \emph{Neural Comput.}, 17\penalty0 (9):\penalty0 2077--2098, 2005.

\end{thebibliography}
\bibliographystyle{icml2022}

\newpage
\appendix
\onecolumn

\section{Analysis of the GCV Estimator (Proofs for \Cref{sec:theory})}\label{sec:proofs}

In this section, we prove our main theoretical result: that the GCV estimator approximates the generalization risk of ridge regression (\Cref{theorem:gcv}). We now give formal statements of \Cref{hypothesis:local-mp,,theorem:gcv}. For \Cref{theorem:gcv}, we will assume that \Cref{hypothesis:local-mp} holds for $\Dist$ as well as a family of linear transformations of $\Dist$.

Before giving formal statements, we make note of a few mathematical conventions that we use throughout this section:
\begin{itemize}
    \item We say that a family of events $A^\n$ indexed by $\n$ occurs \emph{with high probability} if, for any (large) constant $D > 0$, there exists a threshold $\n_D$ such that $A^\n$ occurs with probability at least $1 - \n^{-D}$ for all $\n\ge\n_D$.
    \item For any two families of functions $f^\n, g^\n\colon S\to\R_{\ge 0}$ indexed by $\n$,
    we say that \emph{$f\lesssim g$ uniformly over $S$} if there exists a constant $C > 0$ such that, with high probability, $f^\n(z) \le C\cdot g^\n(z)$ {uniformly} over all $z\in S$. In particular, $\lesssim$ omits constant factors from bounds.
    \item We let $\i$ (in roman type) denote the imaginary unit and use $i$ (in italic type) as an indexing variable.
\end{itemize}

With these conventions in mind, \Cref{hypothesis:local-mp} is formalized as follows:

\begin{hypothesis}[Local Marchenko-Pastur law over $\R_{>0}$]\label{hypothesis:local-mp-formal}
The \emph{local Marchenko-Pastur law} holds over an open set $S\subseteq\R_{>0}$ if, for every deterministic vector $v\in\R^\p$ such that $v^\transp\Sigma v\le 1$, both
\begin{equation}\label{eq:global-real-conv}
\abs*{\frac{1}{\kappa} - \frac{1}{\n}\sum_{i=1}^{\n} \frac{1}{\hat\lambda_i + \lambda}}\lesssim \n^{-\frac 12 + o(1)}\cdot\frac{1}{\kappa}\sqrt{\parfrac{\kappa}{\lambda}}
\end{equation}
and
\begin{equation}\label{eq:local-real-conv}
\abs*{v^\transp \bigl(\kappa(\kappa I + \Sigma)^{-1}\bigr) v - v^\transp \bigl(\lambda (\lambda I + \Sigmahat)^{-1}\bigr) v} \lesssim \n^{-\frac 12 + o(1)}\cdot\frac{1}{\kappa}\sqrt{\parfrac{\kappa}{\lambda}}
\end{equation}
hold uniformly over all $\lambda\in S$.
\end{hypothesis}

To analyze the omniscient risk estimate, we will need a slight extension of \Cref{hypothesis:local-mp-formal}, requiring that \Cref{hypothesis:local-mp-formal} hold for a family of linear transformations of the data distribution $\Dist$:

\begin{hypothesis}\label{hypothesis:local-mp-extended}
\Cref{hypothesis:local-mp-formal} holds for $z = (I + t\Sigma)^{-\frac 12}x$, where $x\sim\Dist$, uniformly\footnote{Since $t$ is $1$-dimensional, this uniformity assumption can be relaxed with a standard $\epsilon$-net argument, which we omit for brevity.} over all $t\in \{s\in\R : \abs s < \frac 12\norm{\Sigma}_{\mathrm{op}}^{-1}\}$.
\end{hypothesis}

\Cref{theorem:gcv} can now formally be stated as follows:

\begin{theorem}\label{theorem:gcv-formal}
Suppose $\lambda > 0$ is such that \Cref{hypothesis:local-mp-extended} holds over $S = (\frac 12\lambda, \frac 32\lambda)$. Then,
\[ \abs*{\GCV_\lambda - \Risk(\hat\beta_\lambda)}\lesssim {\n^{-\frac 12 + o(1)}}\cdot\beta^\transp\Sigma\beta\cdot\brack*{\frac{\norm{\Sigma}_{\mathrm{op}}}{\lambda} + \paren*{\frac{\Tr(\Sigma)}{\n\lambda}}^{3/2}}. \]
\end{theorem}

Recall from \Cref{sec:theory} that our analysis of the GCV estimator proceeds in two steps: showing that $\GCV_\lambda\approx\OmniRisk^\lambda$ and then showing that \smash{$\OmniRisk^\lambda\approx\Risk(\hat\beta_\lambda)$}. For the first step, we show the following proposition (formally restating \Cref{proposition:gcv}):

\begin{proposition}\label{proposition:gcv-app}
Suppose $\lambda > 0$ is such that \Cref{hypothesis:local-mp-formal} holds over $S = (\frac 12\lambda, \frac 32\lambda)$. Then,
\[ \abs*{\GCV_\lambda - \OmniRisk^\lambda}\lesssim\n^{-\frac 12 + o(1)}\cdot\beta^\transp\Sigma\beta\cdot\paren*{1 + \frac{\Tr(\Sigma)}{\n\lambda}}^{3/2}. \]
\end{proposition}

For the second step, we show the following proposition:
\begin{proposition}\label{proposition:hastie-app}
Suppose $\lambda > 0$ is such that \Cref{hypothesis:local-mp-extended} holds over $S = \frac (12\lambda,\frac 32\lambda)$. Then,
\[ \abs*{\OmniRisk^\lambda - \Risk(\hat\beta_\lambda)} \lesssim \n^{-\frac 12 + o(1)}\cdot\beta^\transp\Sigma\beta\cdot\frac{\norm{\Sigma}_{\mathrm{op}}}{\lambda}. \]
\end{proposition}

In the remainder of this section, we prove \Cref{theorem:gcv-formal}. To be self-contained, we briefly recap the setup, precise assumptions, and some background material in \Cref{sec:theory-prelims}. Next, we prove a general lemma to justify the differentiation step (i.e., \eqref{eq:differentiation}) in \Cref{sec:differentiation}. Then, we prove \Cref{theorem:gcv-formal} via \Cref{proposition:gcv-app,,proposition:hastie-app} in \Cref{sec:gcv-analysis,,sec:reanalysis}.

\subsection{Theoretical Preliminaries}\label{sec:theory-prelims}

\subsubsection{Model}

We recall our basic setup from \Cref{sec:preliminaries}.
We consider a random design model of linear regression, in which covariates $x_i$ are drawn i.i.d.\ from a distribution $\Dist$ over $\R^\p$ with second moment $\Sigma\in\R^{\p\times\p}$.
Labels are generated by a ground truth $\beta\in\R^p$, with the $i$-th label given by $y_i = \beta^\transp x_i$. In this model, the distribution $\Dist$ (and in particular its second moment $\Sigma$) and the ground truth $\beta$ are \emph{unobserved}. Instead, all we observe are $\n$ independent samples $(x_1, y_1),\ldots,(x_\n, y_\n)$.

For our theoretical analysis, we additionally impose the mild assumption that $\lambda\ge N^{-C}$ for some (large) constant $C > 0$.\footnote{This assumption is made for convenience: relaxing it worsens the bound by only a $\log(1/\lambda)$ factor.} Note that, beyond our random matrix hypothesis, we \emph{do not} assume anything about the dimensional ratio $\p/\n$, allowing for it to vary widely, and we \emph{do not} assume anything about the covariate distribution $\Dist$.

For the sake of simplicity, we focus on the case where $\E_{x\sim\Dist}[x] = 0$. We note that our analysis can be extended to obtain a correction for non-zero means via the Sherman-Morrison rank-$1$ update formula, but we do not pursue this extension further at this time. %

\subsubsection{Ridge Regression}

We first recall the notation defined in \Cref{sec:preliminaries}.
Let $X\in\R^{\n\times\p}$ be the matrix of covariates and $y\in\R^\n$ be the vector of labels. The empirical second moment matrix is denoted by $\Sigmahat\coloneqq\frac{1}{\n} X^\transp X$. The eigendecompositions of $\Sigma$ and \smash{$\Sigmahat$} are written as \smash{$\Sigma = \sum_{i=1}^\p \lambda_i v_i v_i^\transp$} and \smash{$\Sigmahat = \sum_{i=1}^\n \hat\lambda_i \hat v_i\hat v_i^\transp$}, respectively, with $\lambda_1\ge\cdots\ge\lambda_\p$ and \smash{$\hat\lambda_1\ge\cdots\ge\hat\lambda_\n$}.

Let \smash{$\hat\beta_\lambda$} be the ridge regression estimator
\[ \hat\beta_\lambda\coloneqq\argmin_{\hat\beta} \frac 1\n \sum_{i=1}^\n \bigl(y_i - \hat\beta^\transp x_i\bigr)^2 + \lambda\|\hat\beta\|_2^2 \]
for $\lambda > 0$, and let $\hat\beta_0\coloneqq\lim_{\lambda\to 0^+} \hat\beta_\lambda$. For $\lambda > 0$, one has the closed form $\hat\beta_\lambda = \paren[\big]{\Sigmahat + \lambda I}^{-1}\frac{1}{\n}X^\transp y = \paren[\big]{\Sigmahat + \lambda I}^{-1}\Sigmahat\beta$.

Given an estimator $\hat\beta\in\R^\p$ for $\beta$, its generalization and empirical risks are
\[ \Risk(\hat\beta)\coloneqq \E_{x\sim\Dist}[(\beta^\transp x - \hat\beta^\transp x)^2]\quad\text{and}\quad\TrainRisk(\hat\beta)\coloneqq \frac{1}{\n} \sum_{i=1}^{\n} (y_i - \hat\beta^\transp x_i)^2, \]
respectively. For ridge regression when $\lambda > 0$, one has the closed form expressions
\begin{equation}\label{eq:closed}
\Risk(\hat\beta_\lambda) = \lambda^2\beta^\transp\paren[\big]{\Sigmahat + \lambda I}^{-1}\Sigma\paren[\big]{\Sigmahat + \lambda I}^{-1}\beta \quad\text{and}\quad\TrainRisk(\hat\beta_\lambda) = \lambda^2\beta^\transp\paren[\big]{\Sigmahat + \lambda I}^{-1}\Sigmahat\paren[\big]{\Sigmahat + \lambda I}^{-1}\beta.
\end{equation}

\subsubsection{The Asymptotic Stieltjes Transform}

To relate our random matrix hypothesis (\Cref{hypothesis:local-mp}) to the existing random matrix literature, we define the $\n$-sample \emph{asymptotic Stieltjes transform} $m$ of $\Sigma$, as $m$ is the more standard object to consider in random matrix theory. We will state a version of \Cref{hypothesis:local-mp} in terms of $m$ and later use the properties of $m$ to analyze the GCV estimator.

Before defining $m$, it is helpful to recall the definition of {effective regularization} $\kappa = \kappa(\lambda, \n)$, for $\lambda > 0$, as the (unique) positive solution to
\begin{equation}\label{eq:kappa-app}
 1 = \frac{\lambda}{\kappa} + \frac 1\n\sum_{i=1}^\p \frac{\lambda_i}{\kappa + \lambda_i}.
\end{equation}
The $\n$-sample asymptotic Stieltjes transform $m$ of $\Sigma$ is the analytic continuation of $m(z) = 1 / \kappa(-z, \n)$ (as a function on the negative reals) to $\C\setminus\R_{\ge 0}$. We define $m$ using an equation similar to \eqref{eq:kappa-app}. Let $\HH\coloneqq\{z\in\C : \Im(z) > 0\}$ denote the complex upper half-plane. For each $z\in\HH$, one can show that there exists a unique solution in $\HH$ to
\[ 1 = -zm + \frac{1}{\n}\sum_{i=1}^{\p} \frac{m\lambda_i}{1 + m \lambda_i}, \]
which we take to be $m(z)$. By the Schwarz reflection principle, this function on $\HH$ has a unique analytic continuation to $\C\setminus\R_{\ge 0}$. A key property of $m$ is that there exists a unique positive measure $\varrho$ on $[0, \infty)$ such that
\begin{equation}\label{eq:stieltjes}
m(z) = \int\frac{d\varrho(x)}{x - z}.
\end{equation}
In other words, $m$ is the \emph{Stieltjes transform} of $\varrho$. This measure $\varrho$ is known as the $\n$-sample \emph{asymptotic eigenvalue density} of $\Sigma$. For proofs of these claims, we refer the reader to \citet{bai2010spectral} and \citet[Section 2.2]{knowles17anisotropic}.

\subsubsection{The Random Matrix Hypothesis}\label{sec:laws}

To make our analysis as general as possible and to make the connection to random matrix theory clear, we give our analysis for any distribution $\Dist$ that satisfies \Cref{hypothesis:local-mp-formal}. This hypothesis is a modern interpretation of the Marchenko-Pastur law and formalizes the heuristic random matrix theory identity\footnote{In comparison, the classical Marchenko-Pastur law \cite{marchenkopastur} derives \smash{$\Tr\paren[\big]{\paren{\lambda I + \Sigmahat}^{-1}}\approx\frac{\kappa}{\lambda}\Tr\paren[\big]{\paren{\kappa I + \Sigma}^{-1}}$} over the complex plane, from which it follows that the spectral measure of \smash{$\Sigmahat$} converges to the measure whose Stieltjes transform is given by the right-hand side.}
\[ \lambda\paren[\big]{\lambda I + \Sigmahat}^{-1}\approx\kappa\paren[\big]{\kappa I + \Sigma}^{-1}. \]
To further connect \Cref{hypothesis:local-mp-formal} to the random matrix literature, we state here a stronger version of \Cref{hypothesis:local-mp-formal} (in that it implies \Cref{hypothesis:local-mp-formal}) that has been shown to hold for commonly studied random matrix models \cite{knowles17anisotropic}. While this stronger hypothesis provides uniform convergence for \emph{complex}-valued $\lambda$, we will only need uniform convergence for $\lambda$ on the positive real line as in \Cref{hypothesis:local-mp-formal}.

Let $\Omega\coloneqq\{z\in\C : \Re(z) < 0\}$. The stronger hypothesis, in terms of the asymptotic Stieltjes transform $m$, is as follows:

\begin{hypothesis}[Local Marchenko-Pastur law over $\Omega\setminus\R$]\label{hypothesis:local-mp-app}
The \emph{local Marchenko-Pastur law} holds over an open set $S\subseteq\Omega\setminus\R$ if for every deterministic vector $v\in\R^\p$ such that $v^\transp\Sigma v\le 1$, both
\begin{equation}\label{eq:global}
\abs*{m(z) - \frac{1}{\n}\sum_{i=1}^{\n} \frac{1}{\hat\lambda_i - z}}\lesssim \n^{-\frac 12 + o(1)}\sqrt{\frac{\Im(m(z))}{\Im(z)}}
\end{equation}
and
\begin{equation}\label{eq:local}
\abs*{v^\transp \bigl(I + m(z)\Sigma\bigr)^{-1} v - v^\transp\bigl(I - z^{-1}\Sigmahat\bigr)^{-1} v} \lesssim \n^{-\frac 12 + o(1)}\sqrt{\frac{\Im(m(z))}{\Im(z)}}.
\end{equation}
hold uniformly over all $z\in S$.
\end{hypothesis}

While we do not make further assumptions, we note that \Cref{hypothesis:local-mp-app} is known to hold under general, \emph{non-asymptotic} assumptions, which subsume the typical random matrix theory assumptions of Gaussian covariates and fixed dimensional ratio $\p/\n$. For instance, \citet[Theorem 3.16 and Remark 3.17]{knowles17anisotropic} show that \Cref{hypothesis:local-mp-app} holds for any open $S\subseteq\Omega\setminus\R$ if the following conditions are satisfied, for an a priori fixed (large) constant $C > 0$:
\begin{itemize}
    \itemsep0em
    \item \emph{Sufficient independence.} The following two assumptions hold:
    \begin{itemize}
        \item The covariates $x\sim\Dist$ are distributed as a linear transformation $T z$ of independent (but not necessarily identically distributed) random variables $z_1,\ldots,z_\p$ such that $\E[z_i] = 0$, and $\E[z_i^2] = 1$ for all $i$.\footnote{The assumption $\E[z_i^2] = I$ is without loss: we can absorb any scaling of $z_i$ into $T$.}\textsuperscript{,}\footnote{To see the necessity of this condition, note that if $x = z_1\cdot (1, 1, \ldots, 1)^\transp$, then we would not obtained the desired convergence.}
        
        \item At least a $C^{-1}$ fraction of the eigenvalues of $\Sigma$ are at least $C^{-1}$, and $\norm{\Sigma}_{\mathrm{op}}\le C$ (i.e., the spectrum of $\Sigma$ is not concentrated at $0$ relative to $\norm{\Sigma}_{\mathrm{op}}$).\footnote{To see the necessity of this condition, note that if we allowed for $T = (1, 1, \ldots, 1)^\transp\cdot (1, 0, 0, \ldots, 0)$ (in which case $\Sigma$ would have only one non-zero eigenvalue), then we would again have $x = z_1 \cdot (1, 1, \ldots, 1)^\transp$.}
    \end{itemize}
    \item \emph{Bounded moments.} The random variables $z_1,\ldots,z_\p$ have uniformly bounded $p$-th moments for all $p < \infty$.
    \item \emph{Bounded domain.} The domain $S$ is such that $C^{-1} \le \abs z\le C$ for all $z\in S$. 
    \item \emph{Log-bounded dimensional ratio\footnote{\citet[Section 2.1]{knowles17anisotropic} note that their results can be obtained under a relaxed dimensional ratio assumption using the techniques of \citet{bloemendal2016principal}.}.} The dimensions $\n$, $\p$ satisfy $\n^{1/C}\le \p\le \n^{C}$.
\end{itemize}
The non-asymptotic nature of the dimensional ratio assumption is particularly relevant to us because $\n$ varies while $\p\gg\n$ is fixed when we study scaling in our empirical setting. As a consequence, the dimensional ratio $\p/\n$ takes on a wide range of values. (In contrast, the classical asymptotic assumptions of $\p\to\infty$ and $\p/\n\to\gamma$ are insufficient for our purposes.)

The following lemma shows that \Cref{hypothesis:local-mp-app} implies \Cref{hypothesis:local-mp-formal} (note the change in sign due to $z = -\lambda$):

\begin{lemma}
Let $S\subseteq\Omega$ be open. If \Cref{hypothesis:local-mp-app} holds on $S\setminus\R$, then \Cref{hypothesis:local-mp-formal} holds on $\{\lambda : -\lambda\in S\cap\R\}$. %
\end{lemma}

\begin{proof}
Fix $\lambda\in S$. Consider $z = -\lambda + \i\eta$ in the limit $\eta\to 0^+$. Because $S$ is open, \Cref{hypothesis:local-mp-app} holds for $z = -\lambda + \i\eta$ in a (complex) neighborhood of $\lambda$. Since $m$ maps reals to reals, $\lim_{\eta\to 0^+}{\Im(m(z))}/{\Im(z)} = \smash{\parfrac{}{\eta}}\Im(m(-\lambda)) = m'(-\lambda)$ by the Cauchy-Riemann equations. Moreover, $m'(-\lambda) = \smash{\frac{1}{\kappa^2}\parfrac{\kappa}{\lambda}}$. Hence
\begin{align*}
\abs*{m(-\lambda) - \frac{1}{\n}\sum_{i=1}^{\n} \frac{1}{\hat\lambda_i + \lambda}} 
&= \lim_{\eta\to 0^+}\,\abs*{m(z) - \frac{1}{\n}\sum_{i=1}^{\n} \frac{1}{\hat\lambda_i - z}} \lesssim \lim_{\eta\to 0^+}\n^{-\frac 12+ o(1)}\sqrt{\frac{\Im(m(z))}{\Im(z)}} = \n^{-\frac 12+ o(1)}\cdot\frac{1}{\kappa}\sqrt{\parfrac{\kappa}{\lambda}}
\end{align*}
for \eqref{eq:global-real-conv} and likewise for \eqref{eq:local-real-conv}.
\end{proof}

\subsubsection{The Omniscient Risk Estimate}

Recent works \cite{hastie21surprises, canatar21spectral, wu20optimal, jacot20kernel, loureiro21learning, richards21asymptotics, mel21arbitrary, simon21neural} have shown under a variety of random matrix assumptions that the generalization risk $\Risk(\hat\beta_\lambda)$ of ridge regression can be approximated by the \emph{omniscient risk estimate} $\OmniRisk^\lambda$:
\begin{equation}\label{eq:population-app}
\OmniRisk^\lambda\coloneqq\frac{\partial\kappa}{\partial\lambda}\cdot\kappa^2\sum_{i=1}^\p \biggl(\frac{\lambda_i}{(\kappa + \lambda_i)^2} \bigl(\beta^\transp v_i\bigr)^2 \biggr) = \parfrac{\kappa}{\lambda}\kappa^2\beta^\transp(\Sigma + \kappa I)^{-1}\Sigma(\Sigma + \kappa I)^{-1}\beta.
\end{equation}
The analysis of \citet{hastie21surprises} is the most general of these and establishes \eqref{eq:population-app} under a similar set of assumptions as \Cref{hypothesis:local-mp-app}, with approximation error proportional to $\norm{\beta}^2_2$.

However, in our empirical setting with effectively infinite $\norm{\beta}_2$, we need a stronger version of this result than was previously known. Thus, we improve the result of \citet{hastie21surprises} so that the error bound scales in the expected size of the label $\beta^\transp\Sigma\beta$ rather than the squared norm $\norm{\beta}_2^2$ (see \Cref{proposition:hastie-app}). To prove this generalization requires a more careful analysis, as the analysis of \citet{hastie21surprises} does not directly extend to large $\norm{\beta}_2$.

\subsection{Bounding the Derivative of a Bounded, Real Analytic Function}\label{sec:differentiation}

A key step of our analysis will be arguing that we may differentiate the local random matrix law, as in \eqref{eq:differentiation}, while preserving the approximate equality. In this section, we show a general lemma that lets us accomplish this. Concretely, we will bound the derivative of a bounded, real analytic function. Our approach here streamlines the argument of \citet{hastie21surprises}, allowing for sharper bounds while also being easier to apply. 

Let $h\colon U\to\R$, for some $U\subseteq\R$. (In applications, $h$ will represent the difference of two ``approximately equal'' functions.) Suppose $h$ is real analytic at $x_0$ with radius of convergence $R > 0$. Then $h$ has an analytic continuation $\tilde h$ to the open ball $V\coloneqq\{z\in\C : \abs{z - x_0} < R\}$. Let $K\subseteq V$ be the closed ball $\{z\in\C : \abs{z - x_0}\le\frac 12 R\}$. Given that $h$ and $\smash{\tilde h}$ are bounded on $K\cap\R$ and $K$, respectively, the next lemma bounds $h'(x_0)$ with only a \emph{logarithmic} dependence on the bound on \smash{$\tilde h$}. In our applications, this logarithmic dependence will be negligible: the dominant factor will be the ratio $\delta/R$.

\begin{lemma}\label{lemma:derivative}
Suppose $M\ge \delta > 0$, and $h\colon U\to\R$ is such that $\abs{h(x)}\le\delta$ on $K\cap\R$ and $\abs{\tilde h(z)}\le M$ on $K$. Then,
\[ \abs{h'(x_0)}\lesssim\frac{\delta}{R}\paren*{1 + \log\paren*{\frac M\delta}}^2. \]
\end{lemma}

\begin{proof}
Given the power series expansion $h(x) = \sum_{j=0}^\infty c_j (x - x_0)^j$ of $h$ at $x_0$, the Cauchy integral formula tells us that
\[ \abs{c_j} = \abs*{\frac{1}{2\pi i}\int_{\partial K} \frac{\tilde h(z)}{(z - x_0)^{j+1}}\,dz}\le\paren*{\frac{2}{R}}^{j} M. \]
Let $h_k(x)\coloneqq\sum_{j=0}^k c_j (x - x_0)^j$ the $k$-th order Taylor expansion of $h$ at $x_0$. If $I\coloneqq \bigl[x_0 - \frac 14 R, x_0 + \frac 14 R\bigr]$ and $x\in I$, then
\[ \abs{h(x) - h_k(x)} = \abs*{\sum_{j=k+1}^\infty c_j (x - x_0)^j}\le\sum_{j=k+1}^\infty \abs{c_j}\abs{x - x_0}^j\le 2^{-k}M. \]
Let $\norm{\cdot}_\infty$ denote the sup norm for continuous functions $I\to\R$. Setting $k\coloneqq\floor{1 + \log_2(M/\delta)}$, we have by the triangle inequality that $\norm{h_k}_\infty\le\norm{h}_\infty + \norm{h - h_k}_\infty\le 2\delta$. Let $\mathcal{P}_k$ be the vector space of degree $k$ polynomial functions $I\to\R$. The Markov brothers' inequality says that the linear functional $\mathcal{P}_k\to\R$ given by $p\mapsto p'(x_0)$ has operator norm at most $4k^2/R$ with respect to $\norm{\cdot}_\infty$. Hence
\[ \abs{h'(x_0)} = \abs{h'_k(x_0)}\le \frac{4k^2}{R}\norm{h_k}_\infty\lesssim \frac{\delta}{R}\paren*{1 + \log\paren*{\frac{M}{\delta}}}^2. \qedhere \]
\end{proof}

\subsection{Proof of \Cref{proposition:gcv-app}}\label{sec:gcv-analysis}

To prove \Cref{proposition:gcv-app}, we follow the outline in \Cref{sec:theory}.
Define
\[ f(\lambda)\coloneqq\beta^\transp\beta - \beta^\transp\lambda\paren[\big]{\Sigmahat + \lambda I}^{-1}\beta\quad\text{and}\quad g(\lambda)\coloneqq\beta^\transp\beta - \beta^\transp\kappa\paren[\big]{\Sigma + \kappa I}^{-1}\beta. \]
The $\beta^\transp\beta$ terms in $f$ and $g$ ensure that $f$ and $g$ can be bounded, so that we may apply \Cref{lemma:derivative}.
Additionally, define
\[ h(\lambda)\coloneqq f(\lambda) - g(\lambda) = - \beta^\transp\lambda\paren[\big]{\Sigmahat + \lambda I}^{-1}\beta + \beta^\transp\kappa\paren[\big]{\Sigma + \kappa I}^{-1}\beta. \]
Note that $f$, $g$, and $h$ may be analytically continued to take complex arguments $w = \lambda - \i\eta$, since we may take $\kappa = 1 / m(-w)$. We will also need these extended functions when applying \Cref{lemma:derivative}.

Algebraically, the key drivers of our analysis are the relationships obtained from differentiating $f$ and $g$ with respect to $\lambda$:
\begin{align*}
f'(\lambda) &=\beta^\transp\paren[\big]{\Sigmahat + \lambda I}^{-1}\Sigmahat\paren[\big]{\Sigmahat + \lambda I}^{-1}\beta = \frac{1}{\lambda^2}\TrainRisk(\hat\beta_\lambda) = \paren*{\sum_{i=1}^\n\frac{1}{\lambda + \hat\lambda_i}}^2\GCV_\lambda  \\
g'(\lambda) &= \parfrac{\kappa}{\lambda}\beta^\transp\paren[\big]{\Sigma + \kappa I}^{-1}\Sigma\paren[\big]{\Sigma + \kappa I}^{-1}\beta = \frac{1}{\kappa^2}\OmniRisk^\lambda.
\end{align*}
The main technical steps in the analysis will be to bound $\abs{h'(\lambda)} =\abs{f'(\lambda) - g'(\lambda)}$ and $\abs{\kappa^2 f'(\lambda) - \GCV_\lambda}$, so that we may relate $\GCV_\lambda$ and $\OmniRisk^\lambda$. The former we will bound via \Cref{lemma:derivative}; the latter we will bound using \Cref{hypothesis:local-mp-formal}.

\subsubsection{Auxiliary Lemmas}

We now set up the lemmas that let us formalize our heuristic argument from \Cref{sec:theory}.

The next three lemmas note some basic properties of the effective regularization $\kappa$:

\begin{lemma}\label{lemma:kappa}
For all $\lambda > 0$, $\kappa = \kappa(\lambda, \n)$ satisfies
\[ 1\le\parfrac{\kappa}{\lambda}\le\frac{\kappa}{\lambda}\le 1 + \frac{\Tr(\Sigma)}{\n\lambda}. \]
\end{lemma}

\begin{proof}
Rearranging \eqref{eq:kappa-app} gives us
\[ \kappa = \lambda + \frac{1}{\n}\sum_{i=1}^\p \lambda_i\paren*{1 - \frac{\lambda_i}{\kappa + \lambda_i}}\le\lambda + \frac{\Tr(\Sigma)}{\n}. \]
Dividing by $\lambda$ immediately yields \smash{$\frac{\kappa}{\lambda}\le 1 + \frac{\Tr(\Sigma)}{\n\lambda}$}. To get the first two inequalities, we compute $\parfrac{\kappa}{\lambda}$. By the implicit function theorem applied to \eqref{eq:kappa-app}, $\parfrac{\kappa}{\lambda}$ satisfies
\[ \parfrac{\kappa}{\lambda} = 1 + \parfrac{\kappa}{\lambda}\cdot\frac{1}{\n}\sum_{i=1}^\p\frac{\lambda_i^2}{(\kappa + \lambda_i)^2}. \]
Solving for $\parfrac{\kappa}{\lambda}$, we obtain
\begin{equation}\label{eq:dkappa}
\parfrac{\kappa}{\lambda} = \frac{1}{1 - \frac{1}{\n}\sum_{i=1}^\p \frac{\lambda_i^2}{(\kappa + \lambda_i)^2}}.
\end{equation}
From here, it is clear that $\parfrac{\kappa}{\lambda}\ge 1$. And the upper bound $\parfrac{\kappa}{\lambda}\le\frac{\kappa}{\lambda}$ follows from the fact that
\[ 1 - \frac{1}{\n}\sum_{i=1}^\p\frac{\lambda_i^2}{(\kappa + \lambda_i)^2}\ge 1 - \frac{1}{\n}\sum_{i=1}^{\p}\frac{\lambda_i}{\kappa + \lambda_i} = \frac{\lambda}{\kappa}. \qedhere \]
\end{proof}

\begin{lemma}\label{lemma:kappareal}
Suppose $\kappa = \kappa(\lambda, \n)$ and $\tilde\kappa = 1 / m(-\lambda + \i\eta)$ for $\lambda, \eta > 0$. Then $\Re(\tilde\kappa)\ge\kappa$.
\end{lemma}

\begin{proof}
Note that $\Re(\tilde\kappa)$ satisfies
\[ \Re(\tilde\kappa) = \lambda + \frac{1}{\n}\sum_{i=1}^\p \lambda_i\paren*{1 - \Re\paren*{\frac{\lambda_i}{\tilde\kappa + \lambda_i}}}\ge \lambda + \frac{1}{\n}\sum_{i=1}^\p \lambda_i\paren*{1 - {\frac{\lambda_i}{\Re(\tilde\kappa) + \lambda_i}}}. \]
On the other hand, since $\kappa$ is the unique positive solution to \eqref{eq:kappa-app} and $0 < \lambda$, it holds for all $\kappa'\in [0, \kappa)$ that
\[ \kappa' < \lambda + \frac{1}{\n}\sum_{i=1}^\p \lambda_i\paren*{1 - \frac{\lambda_i}{\kappa' + \lambda_i}}. \]
Therefore, it must be the case that $\Re(\tilde\kappa)\ge\kappa$.
\end{proof}

\begin{lemma}\label{lemma:comparison}
Suppose $\lambda > 0$, and let $\kappa = \kappa(\lambda, \n)$. If $\lambda' > \frac 12\lambda$, then
\[ \frac{1}{\kappa(\lambda', \n)}\sqrt{\parfrac{\kappa}{\lambda}(\lambda', \n)} \lesssim\frac{1}{\kappa}\sqrt{\parfrac{\kappa}{\lambda}}. \]
\end{lemma}

\begin{proof}
The left- and right-hand sides of the desired inequality are simply $\sqrt{m'(-\lambda')}$ and $\sqrt{m'(-\lambda)}$, respectively. Define $t\coloneqq\lambda' / \lambda$. Then it suffices to show $m'(-t\lambda)\lesssim m'(-\lambda)$ for all $t > \frac 12$. By the integral representation \eqref{eq:stieltjes} of $m$,
\[ m'(-t\lambda) = \int \frac{d\varrho(x)}{(x + t\lambda)^2}\le \frac{1}{(\min(t, 1))^2}\int\frac{d\varrho(x)}{(x + \lambda)^2}\lesssim m'(-\lambda), \]
where we use the elementary inequality $\min(t, 1) / (x + t\lambda)\le 1 / (x + \lambda)$.
\end{proof}

Using properties of $\kappa$, we bound $\abs{f(w)}$ and $\abs{g(w)}$ for complex $w = \lambda - \ i\eta$ so that we may later apply \Cref{lemma:derivative}.

\begin{lemma}\label{lemma:gcv-bounded}
Suppose $\beta^\transp\Sigma\beta\le 1$. Then functions $f$ and $g$ satisfies the bounds
\[ \E\brack*{\sup_{\Re(w)\ge\lambda_0} \abs{f(w)}}\le\frac{1}{\lambda_0}\quad\text{and}\quad\sup_{\Re(w)\ge\lambda_0} \abs{g(w)}\le\frac{1}{\lambda_0}. \]
\end{lemma}

\begin{proof}
We first bound $f(w)$ as follows:
\begin{align*}
\abs{f(w)}
&= \abs*{\beta^\transp \Sigmahat \paren[\big]{\Sigmahat + w I}^{-1} \beta}
= \abs*{y^\transp\paren*{XX^\transp + \n\cdot w I}^{-1}y}
\le \norm*{\paren*{XX^\transp + \n\cdot wI}^{-1}}_{\mathrm{op}}\norm{y}_2^2\le\frac{1}{\lambda_0}\frac{1}{\n}\sum_{i=1}^\n y_i^2.
\end{align*}
Note that we used the fact that $\paren[\big]{\frac{1}{\n}XX^\transp + wI}^{-1}$ is normal to bound its operator norm by its spectral radius.
Our bound on $\abs{f(w)}$ holds uniformly over all $w$ such that $\Re(w)\ge\lambda_0$. Hence, taking an expectation, we have
\[ \E\brack*{\sup_{\Re(w)\ge\lambda_0} \abs{f(w)}}\le\E\brack*{\frac{1}{\lambda_0}\frac{1}{\n}\sum_{i=1}^\n y_i^2} = \frac{1}{\lambda_0}\beta^\transp\Sigma\beta\le\frac{1}{\lambda_0}. \]
We also have
\[ \abs{g(w)} = \abs*{\beta^\transp\Sigma\paren[\big]{\Sigma + \kappa I}^{-1}} = \abs*{\beta^\transp\Sigma^{1/2}\paren[\big]{\Sigma + \kappa I}^{-1}\Sigma^{1/2}\beta}\le\beta^\transp\Sigma\beta\cdot\norm*{\paren[\big]{\Sigma + \kappa I}^{-1}}_{\mathrm{op}}\le\frac{1}{\Re(\kappa)}\le\frac{1}{\lambda_0}, \]
where the last inequality follows from \Cref{lemma:kappareal,,lemma:kappa}.
\end{proof}

\subsubsection{Proof of \Cref{proposition:gcv-app}}

\begin{proof}[Proof of \Cref{proposition:gcv-app}]
We first bound $\abs{h'(\lambda)} = \abs{f'(\lambda) - g'(\lambda)}$ by applying \Cref{lemma:derivative} to $h$ with $U \coloneqq (0, 2\lambda)$. By \Cref{hypothesis:local-mp-formal,,lemma:comparison}, we may take \[ \delta\lesssim \n^{-\frac 12 + o(1)}\frac{1}{\kappa}\sqrt{\parfrac{\kappa}{\lambda}}. \] And by \Cref{lemma:gcv-bounded}, $\abs{g(w)}\le 1/\lambda$ when $\Re(w)\ge\frac 12\lambda$. Setting $M = \n^D/\lambda$, \Cref{lemma:gcv-bounded} together with Markov's inequality gives us the high probability bound
\[ \Pr\brack*{\sup_{\Re(w)\ge\frac12\lambda}\abs{f(w)}\ge M}\le\n^{-D}. \]
Therefore, by \Cref{lemma:derivative},
\begin{equation}\label{eq:gcv-proof}
\abs*{f'(\lambda) - g'(\lambda)} = \abs*{h'(\lambda)}\lesssim\frac{\delta}{\lambda}\log\paren*{\frac{M}{\delta}}\lesssim\n^{-\frac 12 + o(1)}\cdot\frac{1}{\lambda\kappa}\sqrt{\parfrac{\kappa}{\lambda}}.
\end{equation}

We now bound the error of $\GCV_\lambda$.
Substituting the closed form \eqref{eq:closed} for $\TrainRisk(\hat\beta_\lambda)$ into the definition of $\GCV_\lambda$, we have that
\[ \GCV_\lambda = \paren*{\frac{1}{\n}\sum_{i=1}^\n\frac{1}{\lambda + \hat\lambda_i}}^{-2}\beta^\transp\paren[\big]{\Sigmahat + \lambda I}^{-1}\Sigmahat\paren[\big]{\Sigmahat + \lambda I}^{-1}\beta. \]
Let 
${\hat\kappa\coloneqq \Big(\frac{1}{\n}\sum_{i=1}^\n \frac{1}{\lambda + \hat\lambda_i}\smash{\Big)^{-1}}}$.
By \Cref{hypothesis:local-mp-formal},
\[ \abs*{1 - \frac{\kappa}{\hat\kappa}} \lesssim
\n^{-\frac 12 + o(1)}\sqrt{\parfrac{\kappa}{\lambda}}. \]
For sufficiently large $\n$, the right-hand side is less than $\frac 12$, which implies $\kappa\ge \frac 12\hat\kappa$. Therefore,
\begin{align*}
\abs*{\kappa^2 - \hat\kappa^{2}}
\le (\kappa + \hat\kappa)\cdot\abs*{\kappa - \hat\kappa}
\le 3\kappa\cdot\hat\kappa\cdot\abs*{1 - \frac{\kappa}{\hat\kappa}}
&\lesssim \kappa^2 \n^{-\frac12 + o(1)}\sqrt{\parfrac{\kappa}{\lambda}}.
\end{align*}
This yields the comparison
\begin{align*}
\Big|\GCV_\lambda - \kappa^2 f'(\lambda)\Big|
\lesssim f'(\lambda)\cdot\kappa^2\n^{-\frac12 + o(1)}\sqrt{\parfrac{\kappa}{\lambda}} 
\lesssim g'(\lambda)\cdot\kappa^2\n^{-\frac12 + o(1)}\sqrt{\parfrac{\kappa}{\lambda}} %
&\le \n^{-\frac 12 + o(1)}\paren*{\parfrac{\kappa}{\lambda}}^{3/2} %
\end{align*}
where we applied \eqref{eq:gcv-proof} to get the third expression. We further have from \eqref{eq:gcv-proof} that
\[ \abs*{\kappa^2 f'(\lambda) - \OmniRisk^\lambda}\lesssim\n^{-\frac 12 + o(1)}\cdot\frac{\kappa}{\lambda}\sqrt{\parfrac{\kappa}{\lambda}}. \]
Thus, the triangle inequality followed by \Cref{lemma:kappa} implies
\[ \abs*{\GCV_\lambda - \OmniRisk^\lambda}\lesssim \n^{-\frac12 + o(1)}\paren*{\paren*{\parfrac{\kappa}{\lambda}}^{3/2} + \frac{\kappa}{\lambda}\sqrt{\parfrac{\kappa}{\lambda}}}\lesssim \n^{-\frac 12 + o(1)}\paren*{1 + \frac{\Tr(\Sigma)}{\n\lambda}}^{3/2}. \qedhere \]
\end{proof}

\subsection{Proof of \Cref{proposition:hastie-app}}\label{sec:reanalysis}

As we did for \Cref{proposition:gcv-app}, we first outline a heuristic proof. Let $U \coloneqq (-\frac 12\norm{\Sigma}_{\mathrm{op}}^{-1}, \frac 12\norm{\Sigma}_{\mathrm{op}}^{-1})$. For $t\in U$ and $\lambda > 0$,
let $\tilde\kappa = \tilde\kappa(t, \lambda, \n)$ denote the asymptotic Stieltjes transform associated to the covariance matrix $\Sigma(I + t\Sigma)^{-1}$, and
define
\begin{align*}
f(t)&\coloneqq \beta^\transp \paren[\big]{I + t\Sigma}^{-1} \beta - \beta^\transp \lambda \paren[\big]{\Sigmahat + \lambda (I + t\Sigma)}^{-1} \beta, \\
g(t)&\coloneqq \beta^\transp \paren[\big]{I + t\Sigma}^{-1} \beta - \beta^\transp \tilde\kappa \paren[\big]{\Sigma + \tilde\kappa (I + t\Sigma)}^{-1} \beta,
\end{align*}
and 
\[ h(t)\coloneqq f(t) - g(t) = -\beta^\transp \lambda \paren[\big]{\Sigmahat + \lambda (I + t\Sigma)}^{-1} \beta + \beta^\transp \tilde\kappa \paren[\big]{\Sigma + \tilde\kappa (I + t\Sigma)}^{-1} \beta. \]
Letting $\widetilde m = 1 / \tilde\kappa$, note that
\[ f'(0)  = \lambda^2\beta^\transp\paren[\big]{\Sigmahat + \lambda I}^{-1}\Sigma\paren[\big]{\Sigmahat + \lambda I}^{-1}\beta = \Risk(\hat\beta_\lambda)\quad\text{and}\quad g'(0) = \paren*{1 + \parfrac{\widetilde m}{t}}\kappa^2\beta^\transp\paren[\big]{\Sigma + \kappa I}^{-1}\Sigma\paren[\big]{\Sigma + \kappa I}^{-1}\beta. \]
We will show that $1 + \parfrac{\widetilde m}{t} = \parfrac{\kappa}{\lambda}$ (see \Cref{lemma:mtilde}), in which case $g'(0) = \OmniRisk^\lambda$.
\Cref{proposition:hastie-app} thus follows, predicated on $h(t)\approx 0$ and differentiation preserving the approximate equality.

\subsubsection{Auxiliary Lemmas}

We now set up the lemmas that let us formalize this heuristic argument. First, we show that $h(t)\approx 0$.

\begin{lemma}\label{lemma:omni-equal}
Suppose $\beta^\transp\Sigma\beta\le 1$ and \Cref{hypothesis:local-mp-extended} holds over $S = (\frac 12\lambda, \frac 32\lambda)$. Then,
\[ \abs[\Big]{\tbeta^\transp \lambda \paren[\big]{\tSigmahat + \lambda I}^{-1} \tbeta - \tbeta^\transp \tkappa\paren[\big]{\tSigma + \tkappa I}^{-1} \tbeta}\lesssim \n^{-\frac 12 + o(1)}\cdot\frac{1}{\tkappa}\sqrt{\parfrac{\tkappa}{\lambda}}. \qedhere \]
\end{lemma}

\begin{proof}
Let $Q\coloneqq I + t\Sigma$. That $t\in U$ implies $Q\succeq\frac 12 I$. Further, define $\tSigma\coloneqq Q^{-\frac 12}\Sigma Q^{-\frac 12}$, $\tX\coloneqq XQ^{-\frac 12}$, \smash{$\tSigmahat\coloneqq\frac{1}{\n} \tX^\transp\tX$}, and \smash{$\tbeta\coloneqq Q^{-\frac 12}\beta$}. Note that
\begin{align*}
h(t)
&= -\beta^\transp Q^{-\frac 12}\lambda \paren[\big]{Q^{-\frac 12}\Sigmahat Q^{-\frac12} + \lambda I}^{-1} Q^{-\frac12} \beta + \beta^\transp Q^{-\frac12} \tilde\kappa \paren[\big]{Q^{-\frac12}\Sigma Q^{-\frac12}+ \tilde\kappa I}^{-1} Q^{-\frac12}\beta \\
&= -\tbeta^\transp \lambda \paren[\big]{\tSigmahat + \lambda I}^{-1} \tbeta + \tbeta^\transp \tkappa\paren[\big]{\tSigma + \tkappa I}^{-1} \tbeta.
\end{align*}
Because $Q$ and $\Sigma$ commute, $\tbeta^\transp\tSigma\tbeta = \beta^\transp\Sigma^{\frac12}Q^{-2}\Sigma^{\frac12}\beta\le\norm{Q^{-1}}_{\mathrm{op}}^2\le 4$. By \Cref{hypothesis:local-mp-extended}, since $\tSigma = \Sigma(I + t\Sigma)^{-1}$,
\[ \abs[\Big]{\tbeta^\transp \lambda \paren[\big]{\tSigmahat + \lambda I}^{-1} \tbeta - \tbeta^\transp \tkappa\paren[\big]{\tSigma + \tkappa I}^{-1} \tbeta}\lesssim\tbeta^\transp\tSigma\tbeta\cdot\n^{-\frac 12 + o(1)}\cdot\frac{1}{\tkappa}\sqrt{\parfrac{\tkappa}{\lambda}}\lesssim \n^{-\frac 12 + o(1)}\cdot\frac{1}{\tkappa}\sqrt{\parfrac{\tkappa}{\lambda}}. \qedhere \]
\end{proof}

The next two lemmas verify that the conditions for applying \Cref{lemma:derivative} hold. Verifying these conditions turns out to be the most technically challenging part of our analysis. \Cref{lemma:omni-continuation} shows that we can analytically continue $\tilde\kappa$ (which we only defined for $t\in U\subseteq\R$) to the complex plane. It follows from \Cref{lemma:omni-continuation} that $f$ and $g$ can be analytically continued over the same domain. We then check in \Cref{lemma:omni-bounded} that this analytic continuation is bounded with high probability.

Our analysis for \Cref{lemma:omni-continuation} extends $\tilde\kappa$ using a fixed point definition of effective regularization. This argument proceeds in three steps: (i) we show for each $w = t - \i\eta$ that a fixed point exists using the Brouwer fixed point theorem; (ii) we argue that this fixed point is unique via the Schwarz lemma; (iii) we verify that the set of fixed points defined by these $w$ give rise to a holomorphic function using the implicit function theorem and the Schwarz reflection principle.

Proving \Cref{lemma:omni-bounded} in the case of $f$ requires a more involved analysis than its analog \Cref{lemma:gcv-bounded}. The previous approach based on diagonalizing the positive semidefinite matrix $\Sigmahat$ fails because $\Sigmahat + \lambda (I + w\Sigma)$ is no longer normal when $w$ is complex. (The failure of normality arises because $\Sigma$ and \smash{$\Sigmahat$} do not commute.) While the same $1/\lambda$ bound still holds, proving it is much more difficult; our argument makes careful use of the properties of symmetric matrices $A + \i B$ with positive definite real part $A\succ 0$.

\begin{lemma}\label{lemma:omni-continuation}
The effective regularization $\tilde\kappa(t, \lambda, \n)$ has an analytic continuation in $t$ to the strip $\{z\in\C : \Re(z)\in U\}$.
\end{lemma}

\begin{proof}[Proof of \Cref{lemma:omni-continuation}]
For fixed $\lambda > 0$ and $t\in U$, define
\[ \phi_{\lambda,t}(z)\coloneqq \lambda + \frac{1}{\n}\sum_{i=1}^\p \paren*{\frac{1}{z} + \frac{1}{\tilde\lambda_i}}^{-1}, \]
where $\tilde\lambda_i\coloneqq \lambda_i / (1 + t\lambda_i)$ is the $i$-th eigenvalue of $\tSigma$.\footnote{Technically, we need to handle zero eigenvalues (in which case the inverse $1/\tilde\lambda_i$ becomes undefined). But such eigenvalues do not contribute to the definition \eqref{eq:kappa-app} and thus may safely be ignored. That is, we assume without loss of generality that $\lambda_i > 0$ for all $i$.}
Note that \eqref{eq:kappa-app} for $\tilde\kappa = \tilde\kappa(t, \lambda, \n)$ can be rearranged to $\tilde\kappa = \phi_{\lambda,t}(\tilde\kappa)$. That is, we can define $\tilde\kappa$ as the unique fixed point of $\phi_{\lambda,t}$ on $\R_{>0}$.

We extend this definition from $t\in U$ to $w$ in the complex plane. Suppose $w = t - \i\eta$ satisfies $t\in U$ and $\eta > 0$. (We will handle $\eta < 0$ via the Schwarz reflection principle.) Define $\tilde\lambda_i\coloneqq \lambda_i / (1 + w\lambda_i)$ and $\phi_{\lambda,w}(z)$ as above. Since $t\in U$ and $\eta > 0$, we have \smash{$\Re(\tilde\lambda_i) > 0$ and $\Im(\tilde\lambda_i) > 0$} for all $i$. Let $\tilde\kappa(w, \lambda, \n)$ be the unique fixed point of $\phi_{\lambda,w}$ in $\HH$.
We validate that $\tilde\kappa$ is well-defined as a holomorphic function in $w$ through the three steps outlined above.

We show the existence of $\tilde\kappa$ by applying the Brouwer fixed point theorem to $\phi_{\lambda,w}$ acting on the compact, convex set
\[ K\coloneqq\{z\in\C : \Re(z)\ge\lambda, \Im(z)\ge 0, \abs{z}\le M\}, \]
where
\smash{$M\coloneqq \lambda + \sum_{i=1}^\p \paren[\big]{\Re\paren{1 / \tilde\lambda_i}}^{-1}$}. We first verify that $\phi_{\lambda,w}$ maps $K$ into $K$. Let $z\in K$ and $q_i\coloneqq 1/z + 1/\tilde\lambda_i$. Then $\Re(q_i) > 0$ and $\Im(q_i) < 0$, which in turn implies $\Re(q_i^{-1}) > 0$ and $\Im(q_i^{-1}) > 0$. Hence, \[ \Re(\phi_{\lambda,w}(z)) = \lambda + \sum_{i=1}^\p \Re(q_i^{-1}) > \lambda\quad\text{and}\quad\Im(\phi_{\lambda,w}(z)) = \sum_{i=1}^\p \Im(q_i^{-1}) > 0. \]
And by the triangle inequality,
\[ \abs*{\phi_{\lambda,w}(z)}\le\lambda + \sum_{i=1}^\p \frac{1}{\abs{q_i}} < \lambda + \sum_{i=1}^{\p} \frac{1}{\Re(1/\tilde\lambda_i)} = M. \]
These bounds show that $\phi_{\lambda,w}$ maps $K$ into the interior of $K$.
By the Brouwer fixed point theorem, $\varphi_{\lambda,w}$ has a fixed point $\tilde\kappa$ in the interior of $K$. In particular, this fixed point satisfies $\tilde\kappa\in\HH$.

We now argue that this fixed point $\tilde\kappa$ is unique over all $z\in\HH$. Following the above argument, one sees that $\phi_{\lambda,w}$ maps $\HH$ to $\HH$. Moreover, $\phi_{\lambda,w}$ is not the identity map. It is then a standard consequence of the Schwarz lemma that $\phi_{\lambda,w}$ has at most one fixed point: We may identify $\HH$ with the unit disk using a biholomorphic map that sends $\tilde\kappa$ to $0$. (Such a map exists by the Riemann mapping theorem.) The induced automorphism on the unit disk cannot fix any other point---otherwise the Schwarz lemma would imply that it is the identity. Thus, $\phi_{\lambda,w}$ has at most one fixed point.

Having shown that $\tilde\kappa$ is well-defined for each $w = t - \i\eta$, we now verify that it defines a holomorphic function over the set of such $w$. By the (holomorphic) implicit function theorem, if $\parfrac{}{z}(z - \phi_{\lambda,w}(z))\neq 0$ at $z = \tilde\kappa$, then we can extend $\tilde\kappa$ to a holomorphic function such that $\tilde\kappa(z) = \varphi_{\lambda,z}(\tilde\kappa(z))$ in a neighborhood of $w$. By continuity, $\Im(\tilde\kappa(z)) > 0$ in a neighborhood of $w$. Uniqueness then implies that this function coincides with our definition of $\tilde\kappa$ in this neighborhood. In particular, $\tilde\kappa$ is holomorphic at $w$. It remains to check that $\parfrac{}{z}(z - \phi_{\lambda,w}(z))\neq 0$ at $z = \tilde\kappa$. Substituting in \eqref{eq:kappa-app}, 
\begin{align*}
\parfrac{}{z}(z - \phi_{\lambda,w}(z))\Big|_{z=\tilde\kappa}
&= 1 - \frac{1}{\n} \sum_{i=1}^\p \frac{\tilde\lambda_i^2}{(\tilde\kappa + \tilde\lambda_i)^2} \\
&= \frac{\lambda}{\tilde\kappa} + \frac{1}{\n}\sum_{i=1}^\p \frac{\tilde\lambda_i}{\tilde\kappa + \tilde\lambda_i} - \frac{1}{\n} \sum_{i=1}^\p \frac{\tilde\lambda_i^2}{(\tilde\kappa + \tilde\lambda_i)^2} \\
&= \frac{\lambda}{\tilde\kappa} + \frac{1}{\n}\sum_{i=1}^\p\paren*{\frac{\tilde\kappa}{\tilde\lambda_i} + 2 +  \frac{\tilde\lambda_i}{\tilde\kappa}}^{-1}.
\end{align*}
Note that $\Re(\tilde\kappa/\tilde\lambda_i), \Re(\tilde\lambda_i/\tilde\kappa) > 0$ because both $\tilde\kappa$ and $\tilde\lambda_i$ have positive real and imaginary parts. Thus, each term in the sum has positive real part. Since $\Re(\lambda/\tilde\kappa) > 0$ as well, \smash{$\Re\bigl(\parfrac{}{z}(z - \phi_{\lambda,w}(z))\big|_{z=\tilde\kappa}\bigr) > 0$}.

Lastly, we confirm $\tilde\kappa$ extends continuously to a map $U\to\R$, which lets us conclude that $\tilde\kappa$ extends to $w = t - \i\eta$ with $\eta < 0$ by the Schwarz reflection principle. For $t_0\in U$ and $\tilde\kappa > 0$ such that $\tilde\kappa = \phi_{\lambda,t}(\tilde\kappa)$, the same implicit function theorem argument shows that $\tilde\kappa$ extends to a holomorphic function $\tilde\kappa(z)$ in a neighborhood of $t_0$. The fixed point condition implies $\tilde\kappa$ decreases in $t$, i.e., $\tilde\kappa'(t_0) < 0$. Thus, $\tilde\kappa(w)\in\HH$ for all $w = t - \i\eta$ with $\eta > 0$ in a neighborhood of $t_0$. Uniqueness then implies this $\tilde\kappa(w)$ is consistent with the definition of $\tilde\kappa$ above, so our definition extends continuously to $U$.
\end{proof}

\begin{lemma}\label{lemma:omni-bounded}
Suppose $\beta^\transp\Sigma\beta\le 1$. Then functions $f$ and $g$ satisfy the bounds
\[ \E\brack*{\sup_{\Re(w)\in U}\abs{f(w)}}\lesssim\frac{1}{\lambda}\quad\text{and}\quad\sup_{\Re(w)\in U}\abs{g(w)}\lesssim\frac{1}{\lambda}. \]
\end{lemma}

Before proving \Cref{lemma:omni-bounded}, we prove a lemma about symmetric matrices with positive definite real part. In analogy to how positive definite matrices generalize positive numbers and how symmetric matrices generalize real numbers,  we establish how symmetric matrices with positive definite real part generalize complex numbers in the right half-plane.

\begin{lemma}\label{lemma:right}
Suppose $Q\in\C^{\p\times\p}$ is such that $A\coloneqq\Re(Q)$ is positive definite and $B\coloneqq\Im(Q)$ is symmetric. Then:
\begin{enumerate}[(i)]
\itemsep=0pt
\item $Q$ is invertible, with its inverse $Q^{-1}$ also being symmetric and having positive definite real part;
\item the spectrum $\sigma(Q)$ of $Q$ satisfies $\sigma(Q)\subseteq\{z\in\C : \Re(z) \ge \norm{A^{-1}}^{-1}_{\mathrm{op}}\}$;
\item the operator norm of $Q^{-1}$ is bounded as $\norm{Q^{-1}}_{\mathrm{op}}\le\norm{A^{-1}}_{\mathrm{op}}$.
\end{enumerate}
\end{lemma}

\begin{proof}
For (i), let $T = A^{-\frac 12} B A^{-\frac 12}$ and write $Q = A^{\frac 12}(I + \i T)A^{\frac 12}$. Note that $T^2\succeq 0$ and so $I + T^2$ is invertible. Thus, we may compute $(I + \i T)\cdot (I - \i T)(I + T^2)^{-1} = I$ to see that $(I + \i T)^{-1} = (I - \i T)(I + T^2)^{-1}$. It follows that
\begin{align*}
Q^{-1}
&= A^{-\frac 12} (I - \i T)(I + T^2)^{-1} A^{-\frac 12} \\
&= A^{-\frac 12} (I + T^2)^{-1} A^{-\frac 12} - \i\cdot A^{-\frac 12} (I + T^2)^{-\frac 12}T(I + T^2)^{-\frac 12} A^{-\frac 12} \\
&= (A + BA^{-1}B)^{-1} - \i\cdot (A^2 + A^{\frac 12}BA^{-1}BA^{\frac 12})^{-\frac 12}B(A^2 + A^{\frac 12}BA^{-1}BA^{\frac 12})^{-\frac 12}.
\end{align*}

For (ii), observe that if $\lambda < \norm{A^{-1}}^{-1}_{\mathrm{op}}$, then $A\succ\lambda I$. Applying (i), we have that $Q - \lambda I + \i\eta I$ is invertible for all $\eta\in\R$. It follows that $\lambda - \i\eta\not\in\sigma(Q)$ for all such $\lambda$ and $\eta$. In other words, $\sigma(Q)\subseteq\{z\in\C : \Re(z)\ge\norm{A^{-1}}^{-1}_{\mathrm{op}}\}$.

For (iii), note that $S\coloneqq\widebar Q{}^\transp Q$ is normal and $\Re(S) = A^2 + B^2\succeq A^2$. Hence $S^{-1}$ is normal and its operator norm equals its spectral radius. We thus have
\[ \norm{Q^{-1}}_{\mathrm{op}}^2 = \norm{S^{-1}}_{\mathrm{op}} = \sup_{z\in\sigma(S^{-1})} \abs{z} = \sup_{z\in\sigma(S)} \frac{1}{\abs z}\le\sup_{z\in\sigma(S)}\frac{1}{\abs{\Re(z)}}\le \norm{\Re(S)^{-1}}_{\mathrm{op}}\le\norm{A^{-1}}^2_{\mathrm{op}}, \]
where the penultimate inequality applies (ii) to $S$.
\end{proof}

\begin{proof}[Proof of \Cref{lemma:omni-bounded}.]
We start by bounding $\E\brack[\big]{\sup_{\Re(w)\in U}\abs{f(w)}}$. Let $w = t - \i\eta$, for $t\in U$ and $\eta\in\R$. Let $Q\coloneqq I + w\Sigma$. (Note that $Q$ is a matrix with complex-valued entries.) By the Woodbury matrix identity,
\begin{align*}
f(w)&=
\beta^\transp Q^{-1} \beta - \beta^\transp  \paren[\big]{\lambda^{-1}\Sigmahat + Q}^{-1} \beta 
= \beta^\transp Q^{-1}\Sigmahat^{\frac12}\paren[\big]{\lambda I + \Sigmahat^{\frac12}Q^{-1}\Sigmahat^{\frac12}}^{-1}\Sigmahat^{\frac12}Q^{-1}\beta.
\end{align*}
We first bound the norm of $\Sigmahat^{\frac12}Q^{-1}\beta$ \emph{uniformly} over $w$; then, we bound the operator norm of \smash{$\paren[\big]{\lambda I + \Sigmahat^{\frac12}Q^{-1}\Sigmahat^{\frac12}}^{-1}$}.

Let $u\coloneqq\Sigmahat^{\frac12}Q^{-1}\beta$. In addition, define $u_0\coloneqq\Sigmahat^{\frac 12} Q_0^{-1}\beta$, where $t_0 = \inf U$ and $Q_0\coloneqq I + t_0\Sigma$. I claim that $\norm{u}_2\le\norm{u_0}_2$, which we will show as \Cref{lemma:tensors}, whose proof we defer:

\begin{lemma}\label{lemma:tensors}
If $u = \Sigmahat^{\frac 12}Q^{-1}\beta$ and $u_0 = \Sigmahat^{\frac 12} Q_0^{-1}\beta$, then $\norm{u}_2\le\norm{u_0}_2$.
\end{lemma}

Supposing \Cref{lemma:tensors},  it thus suffices to bound $\norm{u_0}_2$ to get a uniform bound over all $w$. We have, since $Q_0\succeq\frac 12 I$,
\[ \E\brack[\bigg]{\sup_{\Re(w)\in U} \norm{u}_2}\le \E\brack[\big]{\norm{u_0}_2} = \E\brack*{\beta^\transp Q_0^{-1}\Sigmahat Q_0^{-1}\beta} = \beta^\transp Q_0^{-1}\Sigma Q_0^{-1}\beta = \beta^\transp\Sigma^{\frac 12} Q_0^{-2}\Sigma^{\frac 12}\beta\le\norm{Q_0^{-1}}_{\mathrm{op}}^2\le 4. \]

To bound the operator norm of \smash{$\paren[\big]{\lambda I + \Sigmahat^{\frac12}Q^{-1}\Sigmahat^{\frac12}}^{-1}$}, note that $\lambda I + \Sigmahat^{\frac12}Q^{-1}\Sigmahat^{\frac12}$ can be written as $C + \i D$ with $C\succeq\lambda I$. Thus, by \Cref{lemma:right}, 
\[ \norm*{\paren[\big]{\lambda I + \Sigmahat^{\frac12}Q^{-1}\Sigmahat^{\frac12}}^{-1}}_{\mathrm{op}}\le \frac{1}{\lambda}. \]
Putting everything together, we obtain
\[ \E\brack*{\sup_{\Re(w)\in U}\abs{f(w)}} = \E\brack*{\sup_{\Re(w)\in U} u^\transp\paren[\big]{\lambda I + \Sigmahat^{\frac12}Q^{-1}\Sigmahat^{\frac12}}^{-1} u}\le\E\brack*{\norm{u_0}_2^2\cdot\norm*{\paren[\big]{\lambda I + \Sigmahat^{\frac12}Q^{-1}\Sigmahat^{\frac12}}^{-1}}_{\mathrm{op}}}\le\frac{4}{\lambda}. \]

We now move to bounding $\abs{g(w)}$. By the Woodbury matrix identity,
\[ g(w) =
\beta^\transp Q^{-1} \beta - \beta^\transp  \paren[\big]{\tilde\kappa^{-1}\Sigma + Q}^{-1} \beta 
= \beta^\transp Q^{-1}\Sigma^{\frac12}\paren[\big]{\tilde\kappa I + \Sigma^{\frac12}Q^{-1}\Sigma^{\frac12}}^{-1}\Sigma^{\frac12}Q^{-1}\beta. \]
Since $Q$ and $\Sigma$ commute,
\[ \abs{g(w)} = \abs[\big]{\beta^\transp\Sigma^{\frac12} Q^{-1}\paren[\big]{\tilde\kappa I + \Sigma^{\frac12}Q^{-1}\Sigma^{\frac12}}^{-1}Q^{-1}\Sigma^{\frac12}\beta}\le\norm[\big]{\paren[\big]{\tilde\kappa I + \Sigma^{\frac12}Q^{-1}\Sigma^{\frac12}}^{-1}}_{\mathrm{op}}\cdot\norm{Q^{-1}}_{\mathrm{op}}^2\le \frac{4}{\Re(\tilde\kappa)}\le\frac{4}{\lambda}, \]
where for the penultimate inequality we applied \Cref{lemma:right} and $\norm{Q^{-1}}_{\mathrm{op}}\le w$.
\end{proof}

\begin{proof}[Proof of \Cref{lemma:tensors}.]
Write $Q^{-1} = A + B\i$ and $Q_0^{-1} = A_0 + B_0\i$ for real matrices $A, A_0\succ 0$ and $B, B_0$ symmetric, which we can do by \Cref{lemma:right}. Then,
\[ \norm{u}_2^2 = \beta^\transp\widebar Q{}^{-1}\Sigmahat Q^{-1}\beta = \beta^\transp(A - B\i)\Sigmahat (A + B\i)\beta = \beta^\transp\paren[\big]{A\Sigmahat A + B\Sigmahat B}\beta. \]
Let $\bangle{\cdot,\cdot}_{\Frob}$ denote the Frobenius inner product on $\R^{\p\times\p}$. And let $A\tensor A$ denote the operator given by $S\mapsto A\cdot\bangle{A, S}_{\Frob}$ on $\R^{\p\times\p}$, with $B\tensor B$ denoting the same for $B$.
Then, we may further rewrite
\begin{align*}
\norm{u}_2^2
&= \beta^\transp\paren[\big]{A\Sigmahat A + B\Sigmahat B}\beta 
= \sum_{i=1}^\n \hat\lambda_i \paren*{(\beta^\transp A \hat v_i)^2 + (\beta^\transp B\hat v_i)^2)}
= \sum_{i=1}^\n \hat\lambda_i\bangle[\Big]{\beta\hat v_i^\transp, \paren[\big]{A\tensor A + B\tensor B}\paren*{\beta\hat v_i^\transp}}_{\Frob}.
\end{align*}
We likewise have for $u_0$ that
\[ \norm{u_0}_2^2 = \sum_{i=1}^\n \hat\lambda_i\bangle[\Big]{\beta\hat v_i^\transp, \paren[\big]{A_0\tensor A_0 + B_0\tensor B_0}\paren*{\beta\hat v_i^\transp}}_{\Frob}. \]
To show that $\norm{u}_2\le\norm{u_0}_2$, it therefore suffices to show $A\tensor A + B\tensor B\preceq A_0\tensor A_0 + B_0\tensor B_0$ in the Loewner order on operators $\R^{\p\times\p}\to\R^{\p\times\p}$.

We show $A\tensor A + B\tensor B\preceq A_0\tensor A_0 + B_0\tensor B_0$ by computing $A$ and $B$ explicitly. From \Cref{lemma:right} (and using the fact that $I + t\Sigma$ and $\eta\Sigma$ commute),
\[ A = (I + t\Sigma)\paren*{(I + t\Sigma)^2 + \eta^2\Sigma^2}^{-1}\quad\text{and}\quad B = \i\eta\Sigma\paren*{(I + t\Sigma)^2 + \eta^2\Sigma^2}^{-1}. \]
Note that $A$, $B$, $A_0$, $B_0$ are all diagonalized in the eigenbasis of $\Sigma$. The operators $A\tensor A + B\tensor B$ and $A_0\tensor A_0 + B_0\tensor B_0$ can thus be seen as diagonal $(\p\times\p)\times(\p\times\p)$ matrices in this basis. The $v_iv_j^\transp$ diagonal entry of $A\tensor A + B\tensor B$ is
\[ \frac{(1 + t\lambda_i)(1 + t\lambda_j) + \eta^2 \lambda_i\lambda_j}{((1 + t\lambda_i)^2 + \eta^2\lambda_i^2)((1 + t\lambda_j)^2 + \eta^2\lambda_j^2)}. \]
We first show that this quantity is decreasing in $\eta$ for all $i, j$ when $\eta > 0$. Thus, for a given $t$, it is maximized at $\eta = 0$. We then show that this quantity, at $\eta = 0$, is decreasing in $t$ for all $i, j$. Taking $t\to t_0^+$, we conclude that
\[ A\tensor A + B\tensor B\preceq A_0\tensor A_0 + B_0\tensor B_0. \]

We now verify the numerical claims above. We have, for $a_i = \lambda_i^{-1} + t \ge 0$ and $x = \eta^2$, that
\[ \frac{(1 + t\lambda_i)(1 + t\lambda_j) + \eta^2 \lambda_i\lambda_j}{((1 + t\lambda_i)^2 + \eta^2\lambda_i^2)((1 + t\lambda_j)^2 + \eta^2\lambda_j^2)} = \frac{1}{\lambda_i\lambda_j}\frac{a_ia_j + x}{(a_i^2 + x)(a_j^2 + x)}. \]
When $x$ increases by $\delta$, the numerator increases by $\delta$ and the denominator increases by $\delta^2 + \delta(a_i^2 + a_j^2 + 2x)$. Since
\[ \frac{\delta}{\delta^2 + \delta(a_i^2 + a_j^2 + 2x)}\le\frac{1}{a_i^2 + a_j^2 + 2x}\le\frac{a_ia_j + x}{(a_i^2 + x)(a_j^2 + x)}, \]
the mediant inequality implies the right-hand side is decreasing in $x$. Thus, for a given $t$, $A\tensor A + B\tensor B$ is maximized (in the Loewner order) at $\eta = 0$. Supposing $\eta = 0$, the $v_iv_j^\transp$ diagonal entry becomes ${(1 + t\lambda_i)^{-1}(1 + t\lambda_j)^{-1}}$,
which is clearly decreasing in $t$.
\end{proof}

The next lemma calculates $\parfrac{\widetilde m}{t}(0)$, which appears in $g'(0)$.

\begin{lemma}\label{lemma:mtilde}
Let $\widetilde m(t) = 1/ \tilde\kappa(t, \lambda, \n)$ and $\kappa = \kappa(\lambda, \n)$. Then,
\[ \parfrac{\widetilde m}{t}(0) = \parfrac{\kappa}{\lambda} - 1. \]
\end{lemma}

\begin{proof}
Note that $\tm\coloneqq\tkappa^{-1}$ satisfies
\[ 1 = \lambda\tm + \frac{1}{\n}\sum_{i=1}^\p\paren*{1 - \frac{1 + t\lambda_i}{1 + t\lambda_i + \tm\lambda_i}}. \]
By the implicit function theorem,
\[ 0 = \lambda\parfrac{\tm}{t} + \frac{1}{\n}\sum_{i=1}^\p \frac{(1 + t\lambda_i)(\lambda_i + \lambda_i\parfrac{\tm}{t}) - \lambda_i(1 + t\lambda_i + \tm\lambda_i)}{(1 + t\lambda_i + \tm\lambda_i)^2} = \lambda\parfrac{\tm}{t} + \frac{1}{\n}\sum_{i=1}^\p\frac{ (1 + t\lambda_i)\lambda_i\parfrac{\tm}{t}- \tm\lambda_i^2}{(1 + t\lambda_i + \tm\lambda_i)^2}. \]
Solving for $\parfrac{\tm}{t}$ at $t = 0$, we have that
\begin{align*}
\parfrac{\tm}{t}(0)
&= \paren*{\lambda + \frac{1}{\n}\sum_{i=1}^\p\frac{\lambda_i}{(1 + \tm\lambda_i)^2}}^{-1}\frac{1}{\n}\sum_{i=1}^{\p}\frac{\tm\lambda_i^2}{(1 + \tm\lambda_i)^2} \\
&= \paren*{\frac{\lambda}{\kappa} + \frac{1}{\n}\sum_{i=1}^\p\frac{\kappa\lambda_i}{(\kappa + \lambda_i)^2}}^{-1}\frac{1}{\n}\sum_{i=1}^\p\frac{\lambda_i^2}{(\kappa + \lambda_i)^2} \\
&= \frac{1}{1 - \frac{1}{\n}\sum_{i=1}^\p \frac{\lambda_i^2}{(\kappa + \lambda_i)^2}}- 1 \\
&= \parfrac{\kappa}{\lambda} - 1
\end{align*}
where the last equality follows from \Cref{lemma:kappa}.
\end{proof}

\subsubsection{Proof of \Cref{proposition:hastie-app}}

\begin{proof}[Proof of \Cref{proposition:hastie-app}]
Recall that $f'(0) = \Risk(\hat\beta_\lambda)$. And by \Cref{lemma:mtilde},
\[ g'(0) = \paren*{1 + \parfrac{\widetilde m}{t}}\kappa^2\beta^\transp\paren[\big]{\Sigma + \kappa I}^{-1}\Sigma\paren[\big]{\Sigma + \kappa I}^{-1}\beta = \parfrac{\kappa}{\lambda}\kappa^2\beta^\transp\paren[\big]{\Sigma + \kappa I}^{-1}\Sigma\paren[\big]{\Sigma + \kappa I}^{-1}\beta = \OmniRisk^\lambda. \]

To bound $\abs{f'(0) - g'(0)}$, we apply \Cref{lemma:derivative} to $h$ and $U\coloneqq\{ t : \abs{t} < \frac 12\norm{\Sigma}_{\mathrm{op}}^{-1}\}$. Note that $h$ extends by \Cref{lemma:omni-continuation} to $\{w\in\C : \Re(w)\in U\}$. We have that 
\[ \abs{f(0) - g(0)}\lesssim\n^{-\frac 12 + o(1)}\cdot\frac{1}{\tilde\kappa}\sqrt{\parfrac{\tilde\kappa}{\lambda}} \]
by \Cref{lemma:omni-equal} and $\abs{g(w)}\lesssim 1/\lambda$ uniformly over $\{w\in\C : \Re(w)\in U\}$ by \Cref{lemma:omni-bounded}. Setting $M \coloneqq \n^D/\lambda$, we get from Markov's inequality and \Cref{lemma:omni-bounded} the high probability bound
\[ \Pr\brack*{\sup_{\Re(w)\in U}\abs{f(w)}\ge M}\le\n^{-D}. \]
Hence, by \Cref{lemma:derivative} applied to $h$,
\[ \abs*{\Risk(\hat\beta_\lambda)-\OmniRisk^\lambda} = \abs{f'(0) - g'(0)} = \abs{h'(0)}\lesssim\frac{\delta}{\norm{\Sigma}_{\mathrm{op}}^{-1}}\log\paren*{\frac{M}{\delta}}\lesssim \n^{-\frac 12 + o(1)}\cdot\frac{\norm{\Sigma}_{\mathrm{op}}}{\tilde\kappa}\sqrt{\parfrac{\tilde\kappa}{\lambda}} \le \n^{-\frac{1}{2} + o(1)}\cdot\frac{\norm{\Sigma}_{\mathrm{op}}}{\lambda}. \]
The last inequality above follows \Cref{lemma:kappa}.
\end{proof}

\section{Reducing Noise and Misspecification to the Noiseless Case}\label{sec:noise}

In this appendix, we elaborate on how noisy (or misspecified) linear regression in high dimensions can be embedded into the \emph{noiseless} model introduced in \Cref{sec:preliminaries}, making precise the discussion in \Cref{sec:challenges/noise}.
Specifically, we will show that ridge regression on any noisy (or misspecified) instance can be {uniformly} approximated for all $\lambda\ge 0$ by ridge regression on a noiseless \emph{approximating instance} when $\p > \n$. The intuition for this approximation is that, when $\p > \n$, a noisy (or misspecified) problem is indistinguishable from a problem where the ground truth $\beta$ is ``complex'' and has large norm.

Given this approximation, our subsequent analyses hold whenever the distribution of the approximating instance satisfies \Cref{hypothesis:local-mp}. In the case of noise, we will in fact show that \Cref{hypothesis:local-mp} holds for the approximating instance if it holds for the original covariate distribution. In particular, while the approximating instance may involve a poorly conditioned covariance matrix or a large $\norm{\beta}_2$, they need not pose challenges for our random matrix hypothesis (or our subsequent analysis). (On the other hand, as discussed in \Cref{sec:challenges}, the poor conditioning of the covariance matrix and the large norm of $\beta$ can challenge typical approaches to analyzing ridge regression.)

\subsection{Model}

Consider the more general model in which labels $y'\in\R$ are given by $y' = \beta^\transp x + \xi$, where the covariate vector $x$ and the linear approximation error $\xi$ are drawn jointly, as $(x, \xi)\sim\Disttilde$, from a distribution $\Disttilde$ over $\R^p\times\R$. We assume that $\beta$ provides the best approximation to $y'$ given $x$ among linear functions $\R^\p\to\R$ for $x$ drawn according to $\Disttilde$. This implies the approximation error $\xi$ satisfies
\[ \E_{(x,\xi)\sim\Disttilde}[\xi x] = \E_{(x,\xi)\sim\Disttilde}[(y' - \beta^\transp x) x] = 0. \]
Finally, let $\sigma^2\coloneqq \E_{(x,\xi)\sim\Disttilde}[\xi^2]$ be the squared error of the linear approximation.

We highlight two special cases of this model. If $\E[\xi\given x] = 0$, then $\xi$ can be thought of as observation noise on $\beta^\transp x$. On the other hand, if $\xi$ is constant conditioned on $x$, then we have a noiseless, but misspecified, linear model. This setup can also capture combinations of these two extremes, involving both observation noise and misspecification.

Slightly abusing notation, we also use $\xi$ to denote the vector $\begin{bsmallmatrix}\xi_1 & \xi_2 & \cdots & \xi_\n\end{bsmallmatrix}^\transp\in\R^\n$ of approximation errors for the dataset $X$. The ``type'' of $\xi$ will be clear from the context in which it is used.

\subsection{The Approximating Instance}

We embed this more general instance of linear regression into our noiseless setup by introducing an extra dimension that captures the contribution of the noise and/or misspecification. Let $t > 0$ be a small constant (which we will consider in the limit $t\to 0^+$). We reparameterize $y'$ as $y' = \betatilde^\transp\xtilde$, where
\[
x' = 
\begin{bmatrix}
x \\ t^{\frac 12}\xi
\end{bmatrix}
\qquad\text{and}\qquad
\betatilde =
\begin{bmatrix}
\beta \\ t^{-\frac 12}
\end{bmatrix}.
\]
Because $\E_{(x,\xi)\sim\Dist'}[\xi x] = 0$, note that $x'$ has second moment matrix
\[ \Sigmatilde \coloneqq \E\bigl[x'x'^\transp\bigr] = \begin{bmatrix}
  \Sigma & 0 \\
  0 & t\sigma^2
\end{bmatrix}.  \]
While $\norm{\beta'}_2$ does not converge as $t\to 0^+$, note that $\betatilde^\transp\Sigmatilde\betatilde = \beta^\transp\Sigma\beta + \sigma^2$ has no dependence on $t$. %

Let $\hat\beta_\lambda$ be the ridge regression estimator for the original problem, and let $\betatildehat_\lambda$ be the ridge regression for the modified problem with parameter $t$. We show the following:

\begin{proposition}\label{proposition:uniform}
  For each fixed $\lambda > 0$, the ridge regression estimator $\betatildehat_\lambda$ converges to \smash{$\begin{bsmallmatrix}\smash{\hat\beta_\lambda}\vphantom{\beta_\lambda} \\ 0 \end{bsmallmatrix}$} as $t\to 0^+$. If $\p > \n$ and $\Disttilde$ is non-degenerate\footnote{It suffices that \smash{$\Pr_{(x,\xi)\sim\Disttilde}[x\in U] = 0$} for any $\n$-dimensional subspace $U\subseteq\R^\p$. Some assumption is necessary here to rule out ``effectively'' low-dimensional distributions that lie in a $\p'$-dimensional subspace of $\R^\p$ for some $\p'\le\n$.}, then this convergence is uniform over all $\lambda\ge 0$ almost surely.
\end{proposition}

\begin{proof}
Let $\Sigmatildehat \coloneqq \frac{1}{\n} X'^\transp X'$.  Recall that the estimators \smash{$\hat\beta_\lambda$ and $\betatildehat_\lambda$} can be expressed in the closed forms,
\[ \hat\beta_\lambda = (\Sigmahat + \lambda I)^{-1}\frac{1}{\n} X^\transp y'\qquad\text{and}\qquad
\betatildehat_\lambda = (\Sigmatildehat + \lambda I)^{-1}\frac{1}{\n} X'^\transp y', \]
respectively. It suffices to show that
\begin{equation}\label{eq:noise-identity}
\betatildehat_\lambda - \begin{bmatrix}\hat\beta_\lambda \\ 0 \end{bmatrix} = \frac{t^{\frac 12}}{\n + t \xi^\transp(Q + \lambda I)^{-1}\xi}\begin{bmatrix}t^{\frac 12}\frac{1}{\n} X^\transp (Q + \lambda I)^{-1}\xi\xi^\transp (Q + \lambda I)^{-1} \\ \xi^\transp(Q + \lambda I)^{-1}\end{bmatrix} y',
\end{equation}
where $Q\coloneqq\frac{1}{\n} XX^\transp$ is the normalized kernel matrix: for any fixed $\lambda > 0$, it is clear that taking $t\to 0^+$ makes the difference converge to $0$. Moreover, when $\p > \n$, $Q$ is almost surely non-singular under the non-degeneracy assumption. Hence we may bound the right-hand side in terms of the smallest eigenvalue of $Q$, giving us uniform convergence over all $\lambda\ge 0$.

It remains to show \eqref{eq:noise-identity}. Note that
\[ \Sigmatildehat = \begin{bmatrix} \Sigmahat & t^{\frac 12}\frac{1}{\n} X^\transp \xi \\ t^{\frac 12}\frac{1}{\n} \xi^\transp X & t\frac{1}{\n}\xi^\transp\xi \end{bmatrix} \]
The Schur complement of the top-right block of $\Sigmatildehat + \lambda I$ is
\[ \lambda + \frac{t}{\n}\xi^\transp\xi - \frac{t^{\frac 12}}{\n}\xi^\transp X\cdot(\Sigmahat + \lambda I)^{-1}\cdot \frac{t^{\frac 12}}{\n} X^\transp\xi = \lambda + \frac{t}{\n}\xi^\transp\xi - \frac{t}{\n}\xi^\transp (Q + \lambda I)^{-1}Q\xi = \lambda\paren*{1 + \frac{t}{\n}\xi^\transp\paren*{Q + \lambda I}^{-1}\xi}. \]
Therefore, the block matrix inversion formula gives us 
\begin{align}\label{eq:inverse-diff}
\paren[\big]{\Sigmatildehat + \lambda I}^{-1}
- &\begin{bmatrix}
\paren[\big]{\Sigmahat + \lambda I}^{-1} & 
 \\
 & 0
\end{bmatrix} \\
&= \frac{1}{\lambda}\frac{1}{1 + \frac{t}{\n}\xi^\transp(Q + \lambda I)^{-1}\xi}\begin{bmatrix}
\frac{t}{\n^2}\paren[\big]{\Sigmahat + \lambda I}^{-1}X^\transp\xi\xi^\transp X\paren[\big]{\Sigmahat + \lambda I}^{-1} & -t^{\frac 12}\frac{1}{\n}\paren[\big]{\Sigmahat + \lambda I}^{-1} X^\transp\xi \\
-t^{\frac 12}\frac{1}{\n}\xi^\transp X\paren[\big]{\Sigmahat + \lambda I}^{-1} & 1
\end{bmatrix} \nonumber \\
&= \frac{1}{\lambda}\frac{1}{1 + \frac{t}{\n}\xi^\transp(Q + \lambda I)^{-1}\xi}\begin{bmatrix}
\frac{t}{\n^2}X^\transp\paren[\big]{Q + \lambda I}^{-1}\xi\xi^\transp \paren[\big]{Q + \lambda I}^{-1}X & -t^{\frac 12}\frac{1}{\n}X^\transp\paren[\big]{Q + \lambda I}^{-1} \xi \\
-t^{\frac 12}\frac{1}{\n}\xi^\transp \paren[\big]{Q + \lambda I}^{-1}X^\transp & 1
\end{bmatrix}. \nonumber
\end{align}
Multiplying by $\frac{1}{\n} X'^\transp y'$, we recover \eqref{eq:noise-identity}:
\begin{align*}
\betatildehat_\lambda
- \begin{bmatrix}
\hat\beta_\lambda \\ 0
\end{bmatrix} 
&= \frac{1}{\lambda}\frac{1}{1 + \frac{t}{\n}\xi^\transp(Q + \lambda I)^{-1}\xi}\begin{bmatrix}
\frac{t}{\n^2}X^\transp\paren[\big]{Q + \lambda I}^{-1}\xi\xi^\transp \paren[\big]{Q + \lambda I}^{-1} Q - \frac{t}{\n^2} X^\transp\paren[\big]{Q + \lambda I}^{-1} \xi\xi^\transp \\
-t^{\frac 12}\frac{1}{\n}\xi^\transp \paren[\big]{Q + \lambda I}^{-1} Q + t^{\frac 12}\frac{1}{\n}\xi^\transp
\end{bmatrix}y' \\
&=
\frac{t^{\frac 12}}{\n + t\xi^\transp(Q + \lambda I)^{-1}\xi}\begin{bmatrix}
t^{\frac 12}\frac{1}{\n}X^\transp\paren[\big]{Q + \lambda I}^{-1}\xi\xi^\transp \paren[\big]{Q + \lambda I}^{-1} \\
\xi^\transp\paren[\big]{Q + \lambda I}^{-1}
\end{bmatrix}y'. \qedhere
\end{align*}
\end{proof}

\subsection{The Random Matrix Hypothesis for Noisy Labels}

For the theory of \Cref{sec:theory} to apply, the random matrix hypothesis (\Cref{hypothesis:local-mp}) should hold for the approximating instance of noiseless regression derived from the reduction. Thus, we study when the reduction preserves \Cref{hypothesis:local-mp}, given that it holds for the marginal distribution $\Dist$ of $x$.
For fully general $\xi$, which may be arbitrarily correlated with $x$, we note that the error introduced by the reduction can be bounded in $\sigma$ (but this bound does not improve with $\n$).
We can say more when $\xi$ is \emph{noise} such that $\E\brack{\xi\given x} = 0$ for all $x$. In this case, we show that the reduction preserves the local Marchenko-Pastur law, in the sense that the approximation error increases additively by $\lesssim\n^{-\frac 12}$ (\Cref{proposition:noise}).

For our analysis, we bound the additional error introduced by the reduction to the two approximate equalities posited by \Cref{hypothesis:local-mp}. Specifically, we compare, as $t\to 0^+$, the errors of these approximations for the original and the approximating instances. It is not hard to see that the ``averaged'' law \eqref{eq:local-mp-A}, given by
\[ \frac{1}{\n}\sum_{i=1}^{\n}\frac{1}{\hat\lambda_i + \lambda} \approx \frac{1}{\kappa}, \]
is preserved exactly as $t\to 0^+$: this approximate equality relates a continuous function of $\Sigmatildehat$ to a continuous function of $\Sigmatilde$, and we have the convergences
\[ \lim_{t\to 0^+} \Sigmatildehat = \begin{bmatrix}\Sigmahat &  \\ & 0\end{bmatrix}\quad\text{and}\quad\lim_{t\to 0^+}\Sigmatilde = \begin{bmatrix} \Sigma & \\ & 0 \end{bmatrix}. \]
Thus, we focus on the ``local'' law \eqref{eq:local-mp-B}, given by
\[ v^\transp \lambda \paren[\big]{\lambda I + \Sigmahat}^{-1} v \approx v^\transp \kappa\paren[\big]{\kappa I + \Sigma}^{-1} v. \]
The next proposition bounds the approximation error of \eqref{eq:local-mp-B} when moving from the original instance to the approximating instance in the case where $\xi$ is noise. We give our bound assuming the formal version \Cref{hypothesis:local-mp-formal} of \Cref{hypothesis:local-mp} for the marginal distribution $\Dist$ of $x$.

\begin{proposition}\label{proposition:noise}
Suppose $\E[\xi\given X] = 0$ and $\frac{1}{\sigma}\paren[\big]{\E\brack{\abs{\xi_i}^p\given x_i}}^{\frac 1p}\le C_p < \infty$ almost surely for all $p \in \N$. If $\beta^\transp\Sigma\beta + \sigma^2\le 1$ and $\lambda > \n^{-\frac 32 + o(1)}$ is such that \Cref{hypothesis:local-mp-formal} holds for the marginal distribution $\Dist$ of $x$ over $S = (\frac 12\lambda, \frac 32\lambda)$, then
\[ \lim_{t\to 0^+} {\abs*{\betatilde^\transp\lambda\paren[\big]{\Sigmatildehat + \lambda I}^{-1}\betatilde - \beta'^\transp\kappa\paren[\big]{\Sigmatilde + \kappa I}^{-1}\beta'}}\lesssim \n^{-\frac 12 + o(1)}\cdot\frac{1}{\kappa}{\sqrt{\parfrac{\kappa}{\lambda}}}. \]
\end{proposition}

\begin{proof}
For the approximating instance, we have that
\[ \lim_{t\to 0^+}\beta'^\transp\kappa\paren[\big]{\Sigmatilde + \kappa I}^{-1}\beta'- \frac{1}{t} = \beta^\transp\kappa\paren[\big]{\Sigma + \kappa I}^{-1}\beta, \]
and by \eqref{eq:inverse-diff}, that
\[ \lim_{t\to 0^+}\betatilde^\transp\lambda\paren[\big]{\Sigmatildehat + \lambda I}^{-1}\betatilde - \frac{1}{t} = \beta^\transp\lambda\paren[\big]{\Sigmahat + \lambda I}^{-1}\beta - \frac{2}{\n}\xi^\transp X\paren[\big]{\Sigmahat + \lambda I}^{-1}\beta. \]

The triangle inequality therefore implies that the approximation error increases by at most
\begin{multline}\label{eq:approx-err}
\lim_{t\to 0^+}\abs*{\betatilde^\transp\lambda\paren[\big]{\Sigmatildehat + \lambda I}^{-1}\betatilde - \beta'^\transp\kappa\paren[\big]{\Sigmatilde + \kappa I}^{-1}\beta} -  \abs*{\beta^\transp\lambda\paren[\big]{\Sigmahat + \lambda I}^{-1}\beta - \beta^\transp\kappa\paren[\big]{\Sigma + \kappa I}^{-1}\beta} \\
\lesssim \frac{1}{\n}\abs*{\mkern0.5mu\xi^\transp X\paren[\big]{\Sigmahat + \lambda I}^{-1}\beta}.
\end{multline}
It thus suffices to bound $\frac{1}{\n}\abs*{\xi^\transp u}$, where \smash{$u \coloneqq X\paren[\big]{\Sigmahat + \lambda I}^{-1}\beta$}. This follows from a standard moment bounding argument after conditioning on $X$. Let $\norm{\cdot}_p$ denote the $L_p$-norm of a random variable. Conditioning on a fixed $X$, note that the entries of $\xi$ are independent, mean $0$ random variables by assumption. Thus, for any deterministic vector $v\in\R^\n$ and any $p\in\N$, it follows from the Marcinkiewicz-Zygmund inequality and the triangle inequality that
\begin{align*}
\norm*{\xi^\transp v}_{p}
&= \norm*{\sum_{i=1}^\n \xi_iv_i}_p
\lesssim \sqrt{p\cdot\norm*{\sum_{i=1}^\n \xi_i^2v_i^2}_{\!\frac p2}}
\le\sqrt{p\sum_{i=1}^\n \norm{\xi_i^2}_{\!\frac p2}v_i^2}
=\sqrt{p}\,C_p\cdot\sigma\norm{v}_2,
\end{align*}
where all $L_p$ norms are taken conditional on $X$. Thus, by Markov's inequality, conditional on $X$,
\begin{equation}\label{eq:markov}
\Pr\brack*{\frac{1}{\n}\abs*{\xi^\transp v}\ge \frac{t}{\sqrt\n}}\le\paren*{{\frac{\norm*{\xi^\transp v}_p}{t\sqrt\n}}}^p\le\paren*{\frac{\norm{v}_2}{\sqrt\n}\cdot\frac{\sqrt p\,C_p\cdot\sigma}{t}}^p.
\end{equation}
We now set $v = u$ and bound $\smash{\frac1{\sqrt\n}\norm{u}_2}$. By \eqref{eq:gcv-proof} in the argument of \Cref{proposition:gcv-app} and \Cref{lemma:kappa},
\begin{align*}
\frac{1}{\n}\norm{u}_2^2
&= \beta^\transp\paren[\big]{\Sigmahat + \lambda I}^{-1}\Sigmahat\paren[\big]{\Sigmahat + \lambda I}^{-1}\beta \\
&\lesssim\parfrac{\kappa}{\lambda}\beta^\transp\paren[\big]{\Sigma + \kappa I}^{-1}\Sigma\paren[\big]{\Sigma + \kappa I}^{-1}\beta + \beta^\transp\Sigma\beta\cdot\n^{-\frac 12 + o(1)}\frac{1}{\lambda\kappa}\sqrt{\parfrac{\kappa}{\lambda}} \\
&\le\beta^\transp\Sigma\beta\cdot\frac{1}{\kappa^2}\parfrac{\kappa}{\lambda}\paren[\bigg]{1 + \n^{-\frac 12 + o(1)}{\frac{\Tr(\Sigma)}{\n\lambda}}} \\
&\lesssim \beta^\transp\Sigma\beta\cdot\frac{1}{\kappa^2}\parfrac{\kappa}{\lambda}.
\end{align*}

For any constant $\epsilon > 0$, we may set $p\coloneqq \ceil{D/\epsilon}$ and
\[ t\coloneqq\n^\epsilon\cdot \sqrt p\,C_p \cdot\sigma\sqrt{\beta^\transp\Sigma\beta}\cdot\frac{1}{\kappa}\sqrt{\parfrac{\kappa}{\lambda}}\le \n^{\epsilon}\cdot\sqrt p C_p\cdot\frac{1}{\kappa}\sqrt{\parfrac{\kappa}{\lambda}}. \]
By \eqref{eq:markov}, this implies
that $\Pr\brack[\big]{\frac{1}{\n}\abs*{\xi^\transp u}\ge \frac{t}{\sqrt\n}}\lesssim \n^{-D}$ over the randomness of $X$.
Taking $\epsilon\to 0^+$ slowly in $\n$, we therefore obtain the high probability bound
\[ \frac{1}{\n}\abs*{\xi^\transp u}\lesssim \n^{-\frac 12 + o(1)}\cdot\frac{1}{\kappa}\sqrt{\parfrac{\kappa}{\lambda}}. \]
Combining with \Cref{hypothesis:local-mp-formal} now yields the desired result.
\end{proof}

Finally, we note that, with the weaker assumption that $\xi_1\lesssim\sigma$ and $y_1\lesssim 1$, equation \eqref{eq:approx-err} can also be bounded as
\[ \frac1\n\abs*{\mkern0.5mu\xi^\transp X\paren[\big]{\Sigmahat + \lambda I}^{-1}\beta}\le \frac{1}{\n}\norm{\xi}_2\norm*{X\paren[\big]{\Sigmahat + \lambda I}^{-1}\beta}_2\le \frac{1}{\lambda}\cdot\frac{\norm\xi_2}{\sqrt\n}\cdot\frac{\norm{y}_2}{\sqrt\n}\lesssim\frac{\sigma}{\lambda}. \]
While this bound limits the error in terms of $\sigma$ for very general misspecification, and thus is useful when $\sigma$ is small, it does not improve as $\n$ increases.

\subsection{\Cref{theorem:gcv-formal} and \Cref{proposition:hastie-app} for Noisy Labels}

An immediate consequence of \Cref{proposition:uniform,,proposition:noise} is that our analysis of GCV applies to ridge regression with noisy labels, since any instance with noisy labels can be seen as a limit of noiseless approximating instances that preserve the local Marchenko-Pastur law.

As another application of our reduction, we recover without further work the formula for the generalization risk of ridge regression with noisy labels, in greater generality than previously known \cite{canatar21spectral, hastie21surprises}.

\begin{corollary}\label{corollary:hastie-noise}
Suppose $\E[\xi\given X] = 0$ and $\frac{1}{\sigma}\paren[\big]{\E\brack{\abs{\xi_i}^p\given x_i}}^{\frac 1p}\le C_p < \infty$ almost surely for all $p \in\N$.
If $\lambda > \n^{-\frac 32 + o(1)}$ is such that \Cref{hypothesis:local-mp-extended} holds for the marginal distribution $\Dist$ of $x$ over $S = (\frac 12\lambda, \frac 32\lambda)$, then
\[ \abs*{\OmniRisk^{\lambda,\sigma} - \Risk(\hat\beta_\lambda)} \lesssim \n^{-\frac 12 + o(1)}\cdot\paren*{\beta^\transp\Sigma\beta + \sigma^2}\frac{\norm{\Sigma}_{\mathrm{op}}}{\lambda}, \]
where $\OmniRisk^{\lambda,\sigma}$ is defined to be
\begin{align*} \OmniRisk^{\lambda,\sigma}
&\coloneqq \frac{\partial\kappa}{\partial\lambda}\cdot\kappa^2\sum_{i=1}^\p \biggl(\frac{\lambda_i}{(\kappa + \lambda_i)^2} \bigl(\beta^\transp v_i\bigr)^2 \biggr) + \parfrac{\kappa}{\lambda}\cdot\sigma^2
= \OmniRisk^\lambda + \parfrac{\kappa}{\lambda}\cdot\sigma^2.
\end{align*}
\end{corollary}

\begin{proof}
Combining \Cref{proposition:uniform,,proposition:noise} with \Cref{proposition:hastie-app}, it suffices to compute the limit as $t\to 0^+$ of $\OmniRisk^\lambda(t)$ for the approximating instance with parameter $t$. Indeed, we have that
\begin{align*}
\lim_{t\to 0^+}\OmniRisk^\lambda(t)
&= \parfrac{\kappa}{\lambda}\kappa^2\beta^\transp\paren[\big]{\Sigma + \kappa I}^{-1}\Sigma\paren[\big]{\Sigma + \kappa I}^{-1}\beta + \lim_{t\to 0^+}\parfrac{\kappa}{\lambda}\cdot\kappa^2\frac{t\sigma^2}{(\kappa + t\sigma^2)^2}\paren*{t^{-\frac 12}}^2 = \OmniRisk^\lambda + \parfrac{\kappa}{\lambda}\cdot\sigma^2. \qedhere
\end{align*}
\end{proof}
\section{Proofs for \Cref{sec:applications}}\label{sec:proofs-2}

In this section, we prove \Cref{proposition:A,,proposition:B}. We also formalize the notation: we write $A\asymp B$ if there exists a constant $C > 0$ (fixed throughout) such that $C^{-1}A\le B\le CA$.

\begin{proof}[Proof of \Cref{proposition:A}]
Let $\kappa = \kappa(0, \n)$. Applying the Marchenko-Pastur law \eqref{eq:local-mp-A} at $\lambda = 0$, we have that
\[ \Tr\paren*{(XX^\transp)^{-1}} = \frac{1}{\n}\sum_{i=1}^\n \frac{1}{\hat\lambda_i}\approx\frac{1}{\kappa}. \]
Moreover, $\kappa$ satisfies $\n = \sum_{i=1}^\p \frac{\lambda_i}{\kappa + \lambda_i}$ by \eqref{eq:kappa}. Let $i^*$ be the smallest index $i$ such that $\kappa > \lambda_i$. Then, \smash{$\kappa\asymp (i^*)^{-1-\gamma}$} by the eigenvalue decay assumption. Therefore,
\[ N = \sum_{i=1}^\p \frac{\lambda_i}{\kappa + \lambda_i}\asymp i^* + \frac{1}{\kappa}\sum_{i=i^*}^\p\lambda_i\asymp i^* + \frac{1}{\kappa}\int_{i^*}^{\p} x^{-1-\gamma}\,dx \asymp i^* + \frac{1}{\kappa} (i^*)^{-\gamma}\asymp i^*. \]
It follows that $\kappa\asymp\n^{-1-\gamma}$ and
$\n^{-1}\Tr\paren*{(XX^\transp)^{-1}}\asymp\frac{1}{\n\kappa}\asymp\n^\gamma$.
\end{proof}

\begin{proof}[Proof of \Cref{proposition:B}]
Let $\kappa = \kappa(\lambda, \n)$. By the fact that $y = X\beta$ and the local Marchenko-Pastur law \eqref{eq:local-mp-B}, we have that
\[ y^\transp\paren*{XX^\transp + \n\lambda I}^{-1}y = \beta^\transp\Sigmahat(\Sigmahat + \lambda I)^{-1}\beta \approx \beta^\transp\Sigma(\Sigma + \kappa I)^{-1} \beta = \sum_{i=1}^\p \frac{\lambda_i}{\lambda_i + \kappa}(\beta^\transp v_i)^2. \]
Let $i^*$ be the smallest index $i$ such that $\kappa > \lambda_i$. Then, $\kappa\asymp (i^*)^{-1-\gamma}$ by the eigenvalue decay assumption.
Therefore, we may approximate the right-hand side as
\[ \sum_{i=1}^\p \frac{\lambda_i}{\lambda_i + \kappa}(\beta^\transp v_i)^2\asymp\sum_{i=1}^{i^*} (\beta^\transp v_i)^2 + \frac{1}{\kappa}\sum_{i=i^*}^\p \lambda_i(\beta^\transp v_i)^2 \asymp \int_1^{i^*} x^{-\delta}\,dx + \frac{1}{\kappa}\int_{i^*}^{\p} x^{-1-\gamma-\delta}\,dx. \]
Using the fact that $\delta < 1$, we further approximate
\[ \int_1^{i^*} x^{-\delta}\,dx + \frac{1}{\kappa}\int_{i^*}^{\p} x^{-1-\gamma-\delta}\,dx\asymp (i^*)^{1-\delta} + \frac{1}{\kappa} (i^*)^{-\gamma-\delta}\asymp (i^*)^{1-\delta}\asymp\kappa^{-\frac{1-\delta}{1 + \gamma}}. \]
Composing the above approximations proves the proposition.
\end{proof}

\section{Characterizing Classical vs.\ Non-classical Ridge Regression via the Train-Test Gap}\label{sec:classical}

Building on our theoretical analysis of \Cref{sec:theory,,sec:proofs}, we identify a precise and intuitive separation between the ``classical'' and ``non-classical'' regimes of ridge regression: we argue that the separation is characterized by the ratio between the generalization and empirical risks of the estimator $\hat\beta_\lambda$. We then discuss how our empirical setting belongs to the non-classical regime, whereas many previous non-asymptotic analyses of GCV (and ridge regression) \cite{golub79generalized, hsu14ridge, jacot20kernel} only apply in the classical regime.

A salient feature of overparameterized machine learning environments is the possibility of a large gap between the empirical and generalization risks. Thus, this gap serves as a natural candidate for characterizing ``non-classical'' learning problems. For ridge regression, our developments in \Cref{sec:theory,,sec:proofs} let us precisely discuss this gap. \Cref{theorem:gcv} implies that the ratio between the generalization and empirical risks of $\hat\beta_\lambda$ can be approximated as
\[ \frac{\Risk(\hat\beta_\lambda)}{\TrainRisk(\hat\beta_\lambda)}\approx\frac{\GCV_\lambda}{\TrainRisk(\hat\beta_\lambda)} = \paren*{\sum_{i=1}^\n \frac{\lambda}{\lambda + \hat\lambda_i}}^{-2}\approx \paren*{\frac{\kappa}{\lambda}}^2, \]
where the last approximation follows from \eqref{eq:local-mp-A} of \Cref{hypothesis:local-mp}. In particular, the ratio $\kappa/\lambda$ between the effective and the explicit regularizations determines the (multiplicative) train-test gap.

We say that a ridge regression instance is \emph{non-classical} if $\kappa / \lambda\gg 1$, for $\kappa = \kappa(\lambda, \n)$, and \emph{classical} otherwise. (Note that \Cref{lemma:kappa} implies $\kappa/\lambda\ge 1$ always.) Thus, non-classical instances are characterized by having a large train-test gap. The quantity $\kappa/\lambda$ shows up in several places besides the train-test gap: it arises in the definition \eqref{eq:kappa-app} of $\kappa$, and also in our bound for \Cref{proposition:gcv-app} relating $\GCV_\lambda$ and $\OmniRisk^\lambda$\footnote{Note that the multiplier on $\n^{-\smash{\frac 12} + o(1)}$ in the error bound can also be bounded by $(\kappa/\lambda)^{3/2}$.}. Generally, it appears that problems with a larger $\kappa/\lambda$ are more challenging to understand: this ratio determines the ``constant'' factor as $\n$ grows in our bounds; for other analyses, we will observe that that they in fact \emph{do not apply} once $\kappa/\lambda$ exceeds a constant and thus are limited to the classical regime.

\begin{remark}
Note that while $\p\ge\n$ is {necessary} for a problem to lie in the non-classical regime, it is not sufficient. Even in high dimensions, if we take $\lambda$ to be sufficiently large, we will find ourselves back in the classical regime. However, this can be far from optimal in terms of generalization (see, e.g., \Cref{fig:train-test}).
\end{remark}

The quantity $\kappa/\lambda$ connects to our empirical setting via the train-test gap. As can be seen from the empirical and generalization risk curves for \NTK{} regression on pretrained ResNet-34 representations of CIFAR-100 in \Cref{fig:train-test}, the optimal regularization is such that ratio between generalization and empirical risk is much larger than $1$ (meaning that $\kappa/\lambda$ is large as well), with this trend holding consistently across models and datasets. Thus, for a theoretical analysis to be applicable to our empirical setting, it should work when $\kappa/\lambda$ is large.

We next discuss how this feature of large $\kappa/\lambda$ can be challenging for more ``classical'' analyses of GCV and ridge regression:
\begin{description}
\itemsep=0pt
\item[Fixed design.] The first analyses of GCV (and ridge regression) \cite{craven78smoothing, golub79generalized} were for the setting of \emph{fixed design}, where the estimator $\hat\beta_\lambda$ is both trained and evaluated on the same dataset $x_1,\ldots,x_\n\in\R^\p$, but with noisy labels $y_i = \beta^\transp x_i + \epsilon_i$ resampled between train and evaluation time. Without noise, the generalization risk would simply be the empirical risk. Thus, when specialized to the noiseless case, such arguments for the consistency of the GCV estimator would imply the empirical risk approximates the generalization risk, which we know to be false.

To concretely see which assumption fails in such an analysis, we note that \citet{golub79generalized} require in their proof of the consistency of GCV that \smash{$\frac{1}{\n}\Tr\paren[\big]{\Sigmahat(\Sigmahat + \lambda I)^{-1}}\to 0$}. However, we also have that
\[ \frac{1}{\n}\Tr\paren[\big]{\Sigmahat(\Sigmahat + \lambda I)^{-1}} = \frac{1}{\n}\sum_{i=1}^\n \frac{\hat\lambda_i}{\hat\lambda_i + \lambda} = 1 - \lambda\cdot\frac{1}{\n}\sum_{i=1}^\n \frac{1}{\hat\lambda_i + \lambda}\approx 1 - \frac{\lambda}{\kappa}, \]
where the last approximation follows from \eqref{eq:local-mp-A} of \Cref{hypothesis:local-mp}. Thus, their assumption also implies $\kappa/\lambda\to 1$.

\item[Convergence of $\Sigmahat\to\Sigma$.] One approach to bounding generalization in the setting of \emph{random design} (i.e., as described in \Cref{sec:model}) is to show $\Sigmahat\approx\Sigma$ in an appropriate sense \cite{hsu14ridge, steinhardt21notes, bach21learning}. Being able to do so, however, often implies that $\TrainRisk(\hat\beta_\lambda)\approx\Risk(\hat\beta_\lambda)$, since the formulas for empirical and generalization risk can be obtained from each other by swapping a $\smash{\Sigmahat}$ for a $\Sigma$.

Concretely, the analyses of \citet{hsu14ridge} and \citet{steinhardt21notes} assume $\n\ge 2\sum_{i=1}^\p\frac{\lambda_i}{\lambda + \lambda_i}$. Now, since $\kappa\ge\lambda$, we have by \eqref{eq:kappa-app} that these analyses apply only when
\[ \frac{\kappa}{\lambda} = \paren*{1 - \frac{1}{\n}\sum_{i=1}^\p \frac{\lambda_i}{\kappa + \lambda_i}}^{-1}\le\paren*{1 - \frac{1}{\n}\sum_{i=1}^\p \frac{\lambda_i}{\lambda + \lambda_i}}^{-1}\le 2. \]
Similarly, \citet{bach21learning} assumes $\n\lambda\ge 2\Tr(\Sigma)$, in which case $\frac\kappa\lambda\le 1 + \frac{\Tr(\Sigma)}{\n\lambda} < 2$ by \Cref{lemma:kappa}.

\item[Classical random matrix theory.] Finally, we note that more classical random matrix theory techniques, e.g., those used by \citet{jacot20kernel}, which were originally developed for asymptotic analyses in the fixed dimensional ratio limit \cite{marchenkopastur}, can also struggle in the $\kappa/\lambda\gg 1$ regime. For instance, the bounds of \citet{jacot20kernel} are only non-vacuous when \smash{$\frac{\Tr(\Sigma)}{\n\lambda}\le 1$}, in which case $\frac{\kappa}{\lambda}\le 2$ by \Cref{lemma:kappa}. (In contrast, \Cref{hypothesis:local-mp} is motivated by recent developments in random matrix theory \cite{erdos17dynamical, knowles17anisotropic} that provide fine-grained control over the resolvent via fluctuation averaging arguments.)
\end{description}

\section{Details of the Experimental Setup}\label{sec:details}

\subsection{Evaluating GCV}

Recall from \Cref{sec:setup} that, for each model-dataset pair, we compute a kernel $K\in\R^{\n_0\times\n_0}$, where $\n_0$ is the dataset size, from the model's \NTK{} representations of the dataset, and that we approximate the full \NTK{} by $I\tensor K\in\R^{(\n_0\times\c)\times(\n_0\times\c)}$. To solve our classification tasks, we perform kernel regression on the one-hot labels $y_i\in\R^\c$ corresponding to each data point $x_i$, after normalizing each label to have mean $0$. Using our approximation, we have the decomposition of this task into $\c$ independent kernel regression problems, one for each class.

To aggregate risk, we simply sum the mean squared error over the $\c$ output dimensions. Observe that the normalization is such that predicting $0$ trivially obtains risk $\le 1$. To implement GCV for $\c$-dimensional output, we do the same, summing independent estimates of generalization risk for each of the $\c$ output dimensions.

For consistent comparisons across dataset sizes, we evaluate for each dataset size $\n$ the $\lambda$ values $\{\lambda_0/\n : \lambda_0\in\Lambda_0\}$ for each $\n$, where $\Lambda_0\subseteq\R_{\ge 0}$ is a set of base values chosen in proportion to \smash{$\norm{\Sigmahat}_{\mathrm{op}}$}. The range of $\Lambda_0$ is chosen to be the smallest one so that the generalization risk approximately converges at both extremes across all dataset sizes.

To solve the kernel regression problems for many regularization levels $\lambda$, we first diagonalize the kernel matrix. Doing so also allows for efficient computation of $\GCV_\lambda$ over multiple values of $\lambda$. The largest kernel matrices that we work with are obtained from the Food-101 dataset and have size $75750\times 75750$. We note that, while these matrices are substantial in size, they are much smaller than the \NTK{} representations before applying the kernel trick: a ResNet-101 has 44 million parameters, and thus, the matrix of \NTK{} representations would be of size approximately $75750\times 44\cdot 10^6$.

\subsection{Comparing GCV to Alternate Approaches}

To estimate $\alpha$ and $\sigma$ for $\SpecRisk$, we first note that the risk estimate is linear in $\alpha^2$ and $\sigma^2$. Thus, we fit $\alpha^2$ and $\sigma^2$ to minimize the mean squared error of the estimates over the set of $(\n, \lambda)$ pairs considered. We use these estimated $\alpha$ and $\sigma$ for all downstream evaluations.

For the correlation benchmark, we simply compute the Pearson correlation coefficient between each set of predictions over all pairs $(\n, \lambda)$ and corresponding values observed for generalization risk. Observe that correlation is (up to sign) invariant under affine transformations of the predictions.

For the scaling law benchmark, we first find for each $\n$ the $\lambda^*_\n$ that minimizes the generalization risk of ridge regression. Given a predictor, let \smash{$\widehat{\mathcal R}^*_\n$} be the risk prediction corresponding to $\n$ and $\lambda^*$. To estimate the rate $\hat\alpha$ of optimal scaling from each predictor, we fit the slope of the pairs $(\n, \smash{\widehat{\mathcal R}^*})$ on a log-log plot. To estimate the true scaling rate, we apply the same procedure to the observed generalization risks $\smash{\Risk(\hat\beta_{\lambda^*_\n})}$.
\section{Deriving the Spectrum-only Estimate}\label{sec:isotropic}

The result of \citet{dobriban18high} can be recovered from \Cref{corollary:hastie-noise} by assuming an isotropic prior $\mathcal{N}(0, {\alpha^2} I)$ over $\beta$. Indeed, we have that 
\[ \E_{\beta\sim\mathcal{N}(0, \alpha^2 I)}\brack*{\OmniRisk^{\lambda,\sigma}} = \alpha^2\cdot\parfrac{\kappa}{\lambda}\kappa^2\sum_{i=1}^\p \frac{\lambda_i}{(\kappa + \lambda_i)^2} + \sigma^2\cdot\parfrac{\kappa}{\lambda}. \]
To obtain an estimate for the first term, by \Cref{theorem:gcv-formal}, we can use the GCV estimate for the noiseless case:
\begin{align*}
\alpha^2\cdot\parfrac{\kappa}{\lambda}\kappa^2\sum_{i=1}^\p \frac{\lambda_i}{(\kappa + \lambda_i)^2}
&= \alpha^2\cdot\kappa^2\parfrac{}{\lambda}\paren*{-\Tr\paren[\Big]{\Sigma\paren[\big]{\Sigma + \kappa I}^{-1}}} \\
&\approx\alpha^2\cdot\hat\kappa^2\parfrac{}{\lambda}\paren*{-\Tr\paren[\Big]{\Sigmahat\paren[\big]{\Sigmahat + \lambda I}^{-1}}} \\
&= \alpha^2\cdot\hat\kappa^2\sum_{i=1}^\n\frac{\hat\lambda_i}{(\lambda + \hat\lambda_i)^2}.
\end{align*}
And for the second term, using the fact that $\kappa \approx \hat\kappa$, we have
\[ \sigma^2\cdot\parfrac{\kappa}{\lambda}\approx\sigma^2\cdot\parfrac{\hat\kappa}{\lambda} = \frac{\sigma^2}{\n}\cdot\hat\kappa^2\sum_{i=1}^\n\frac{1}{(\lambda + \hat\lambda_i)^2}. \]
This recovers the expressions used in \Cref{sec:comparison}.
\section{Empirical Evidence for the Local Marchenko-Pastur Law}\label{sec:verifying-mp}

In this section, we present evidence for the validity of \Cref{hypothesis:local-mp} in our empirical setting. Our findings here give further support to random matrix effects being a central driver of the phenomena surrounding overparameterized generalization.

While it is impossible to directly verify \Cref{hypothesis:local-mp} due to the high dimensionality of our empirical setting, we can still check whether direct consequences of this hypothesis hold. In particular, consider
\[ f(\lambda, \n)\coloneqq y^\transp\paren*{XX^\transp + \n\lambda I}^{-1}y = \beta^\transp\Sigmahat(\Sigmahat + \lambda I)^{-1}\beta \approx \beta^\transp\Sigma(\Sigma + \kappa I)^{-1} \beta, \]
where the approximate equality holds by \eqref{eq:local-mp-B}. Thus, if \Cref{hypothesis:local-mp} holds, then $f(\lambda, \n)$ should be determined by $\kappa(\lambda, \n)$. By \eqref{eq:local-mp-A} of \Cref{hypothesis:local-mp}, we may also estimate $\kappa(\lambda, \n)$ as
\[ \kappa(\lambda, \n)\approx\paren*{\sum_{i=1}^\n \frac{1}{\lambda + \hat\lambda_i}}^{-1}\eqqcolon\hat\kappa(\lambda, \n). \]
To check the consistency of \Cref{hypothesis:local-mp}, we can therefore examine whether the curves traced out by $(\hat\kappa(\lambda, \n), f(\lambda, \n))$ for varying $\lambda$ coincide across values of $\n$. We plot a version of this in \Cref{fig:coincide}, where we multiply $f$ by $\hat\kappa$ for normalization.

\begin{figure}
    \centering
    \includegraphics[width=\columnwidth]{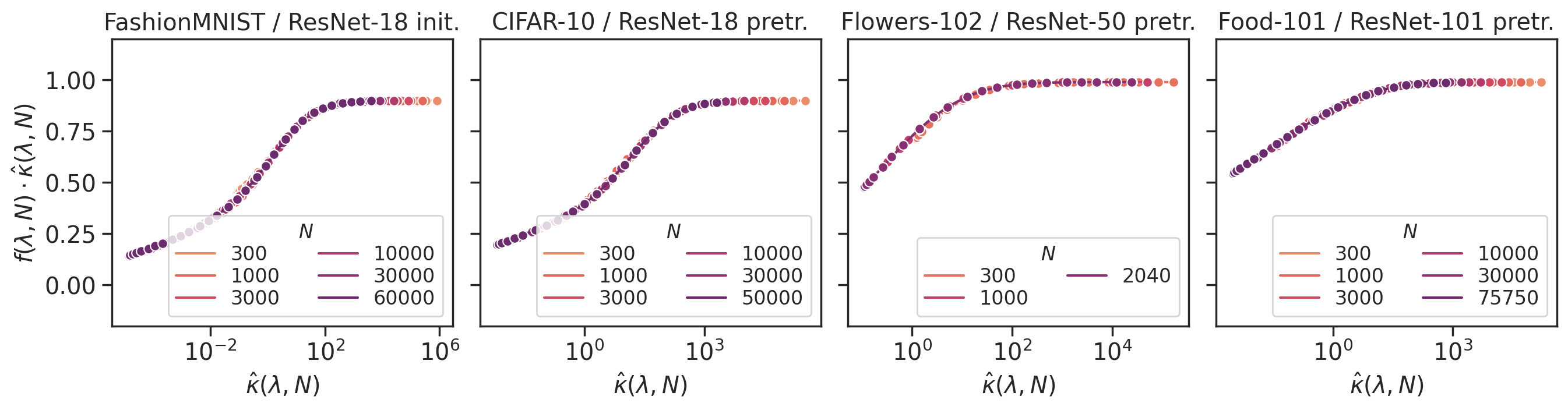}
    \caption{Plotting $(\hat\kappa(\lambda, \n), f(\lambda, \n)\cdot\hat\kappa(\lambda, \n))$ for varying values of $\lambda$ and $\n$}
    \label{fig:coincide}
\end{figure}

Examining \Cref{fig:coincide}, we find that the curves traced out for different values of $\n$ almost coincide, as predicted by \Cref{hypothesis:local-mp}, with this holding across a range of models and datasets. Thus, we find support for the local Marchenko-Pastur law being valid in our empirical setting.

\clearpage

\section{Additional Experiments and Figures}\label{sec:additional}

\subsection{Growth of $\norm{\hat\beta_0}_2/\sqrt\n$ in $\n$}

In this section, we provide additional examples of when the norm-based estimate $\norm{\hat\beta_0}_2/\sqrt\n$ increases as $\n$ increases and the generalization risk decreases in \Cref{fig:ng}, showing that this observation is consistent across models and datasets.

\begin{figure}[H]
    \centering
    \includegraphics[width=0.45\textwidth]{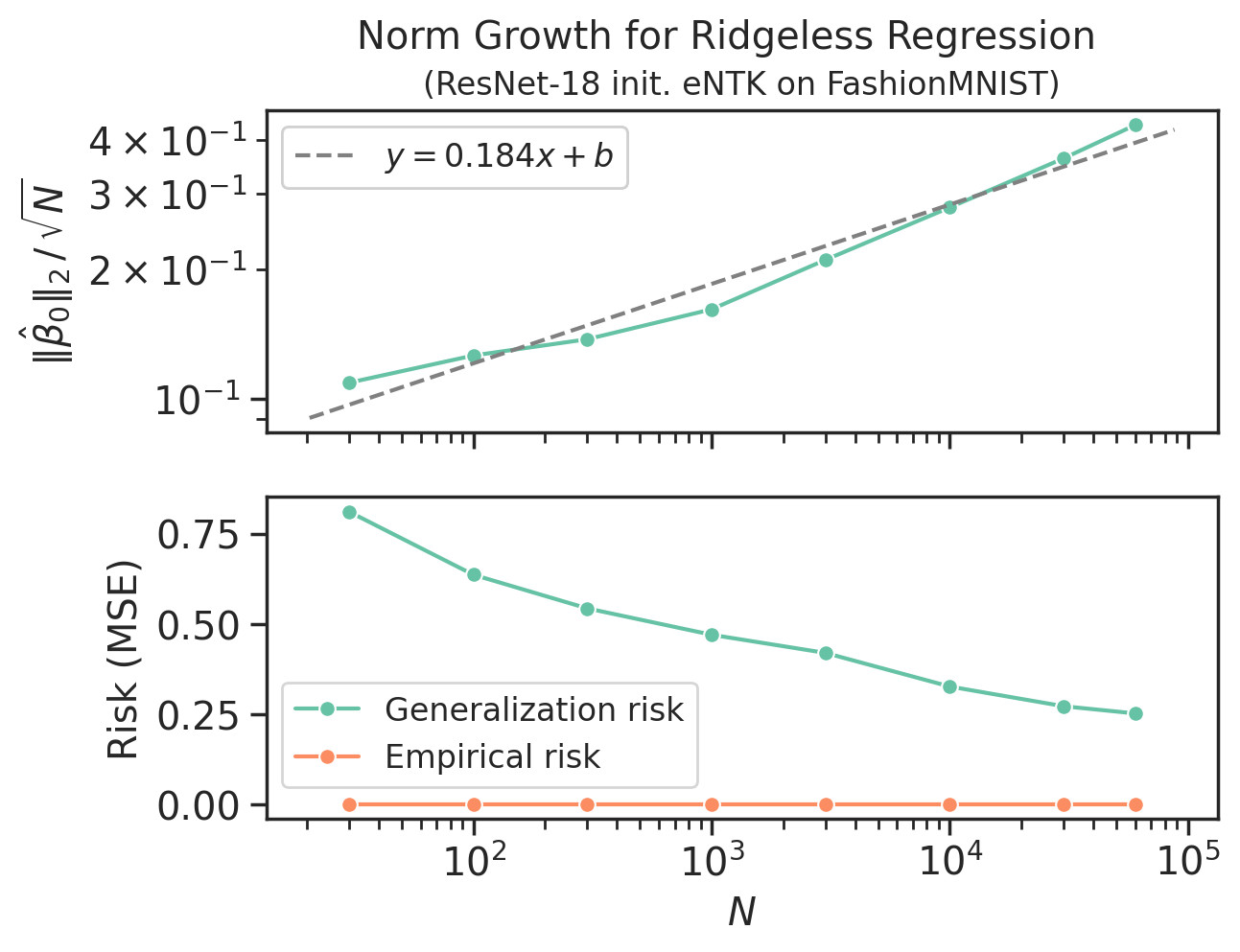}
    \includegraphics[width=0.475\textwidth]{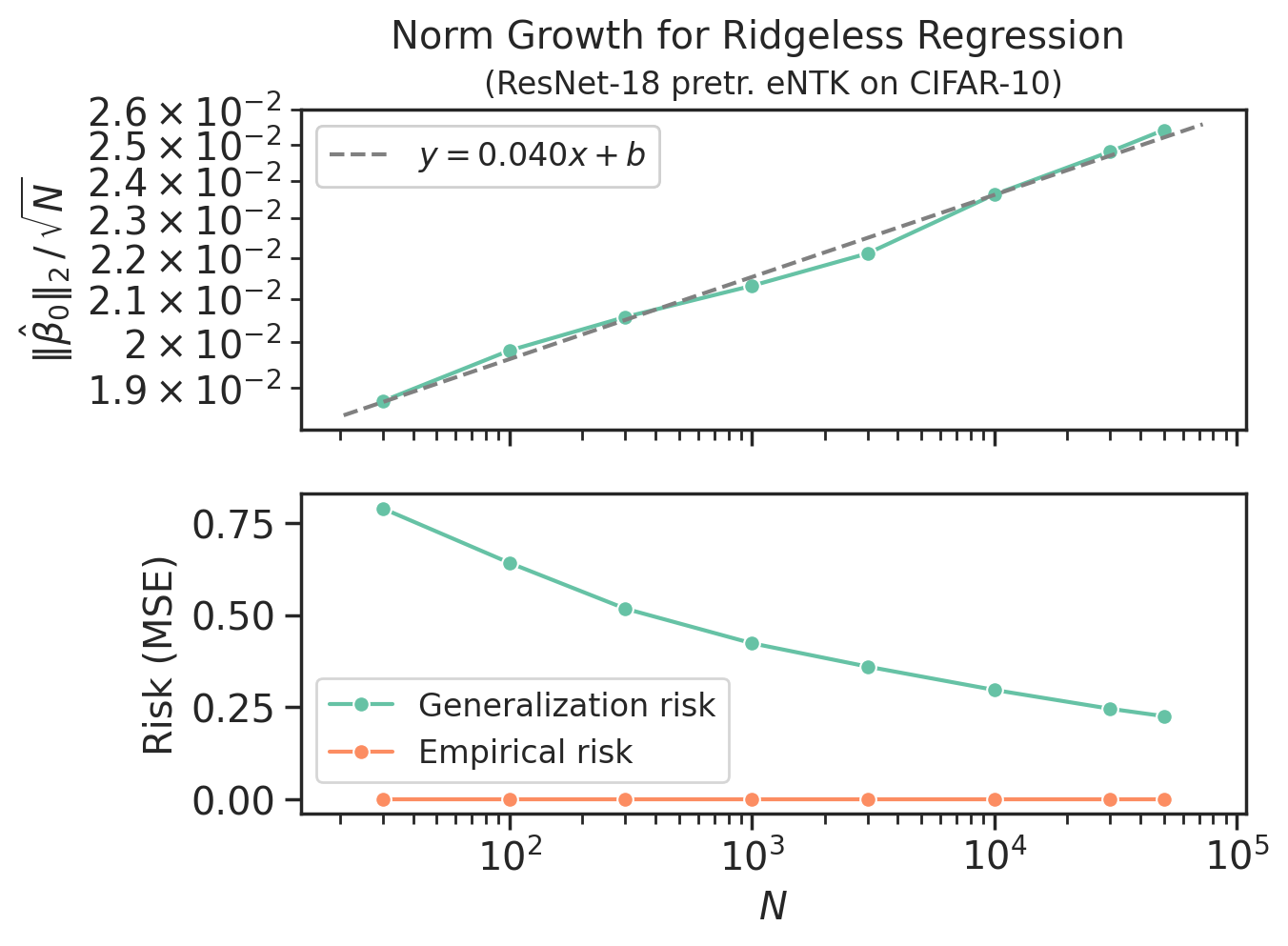}
    \includegraphics[width=0.45\textwidth]{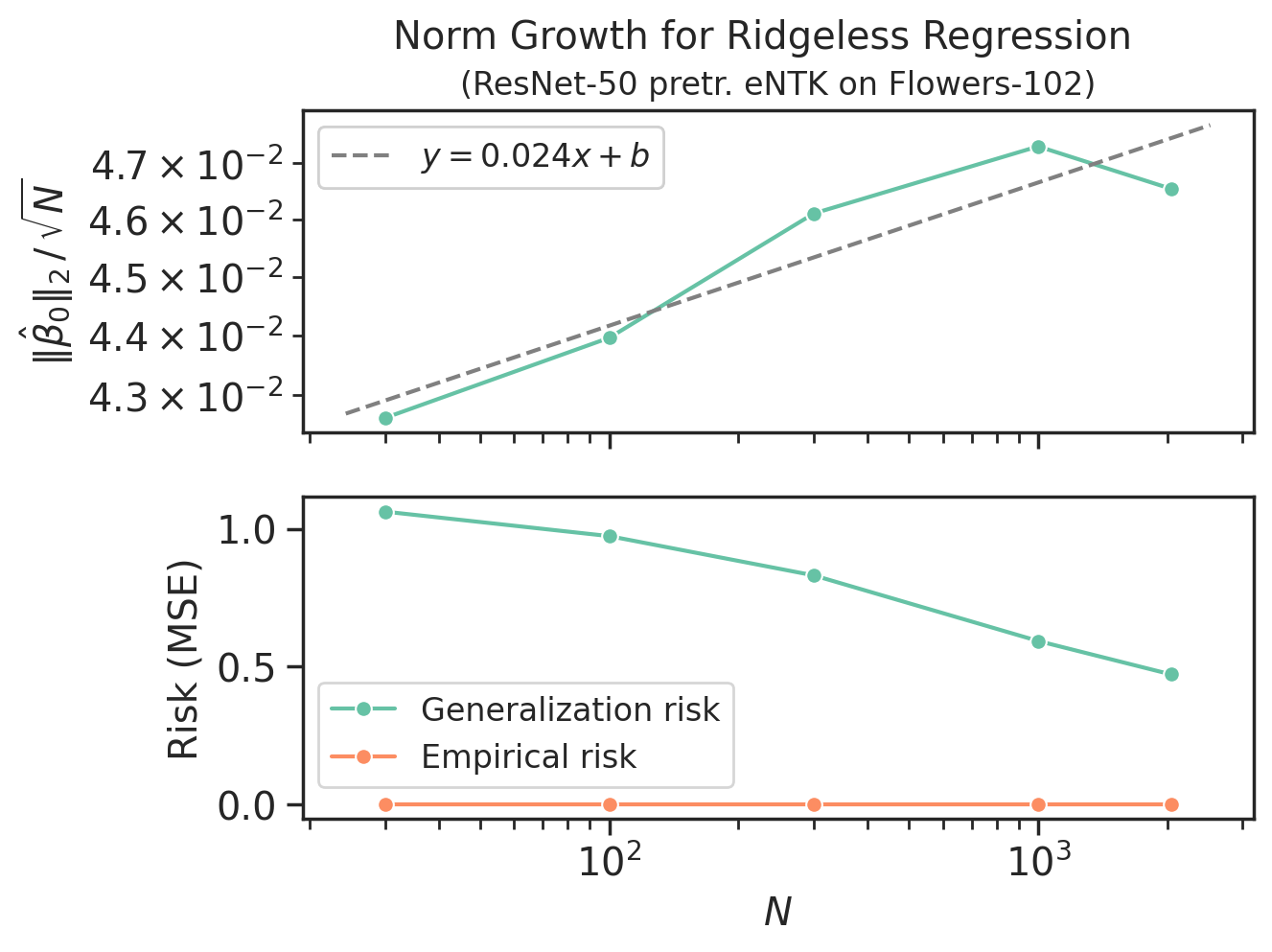}
    \includegraphics[width=0.45\textwidth]{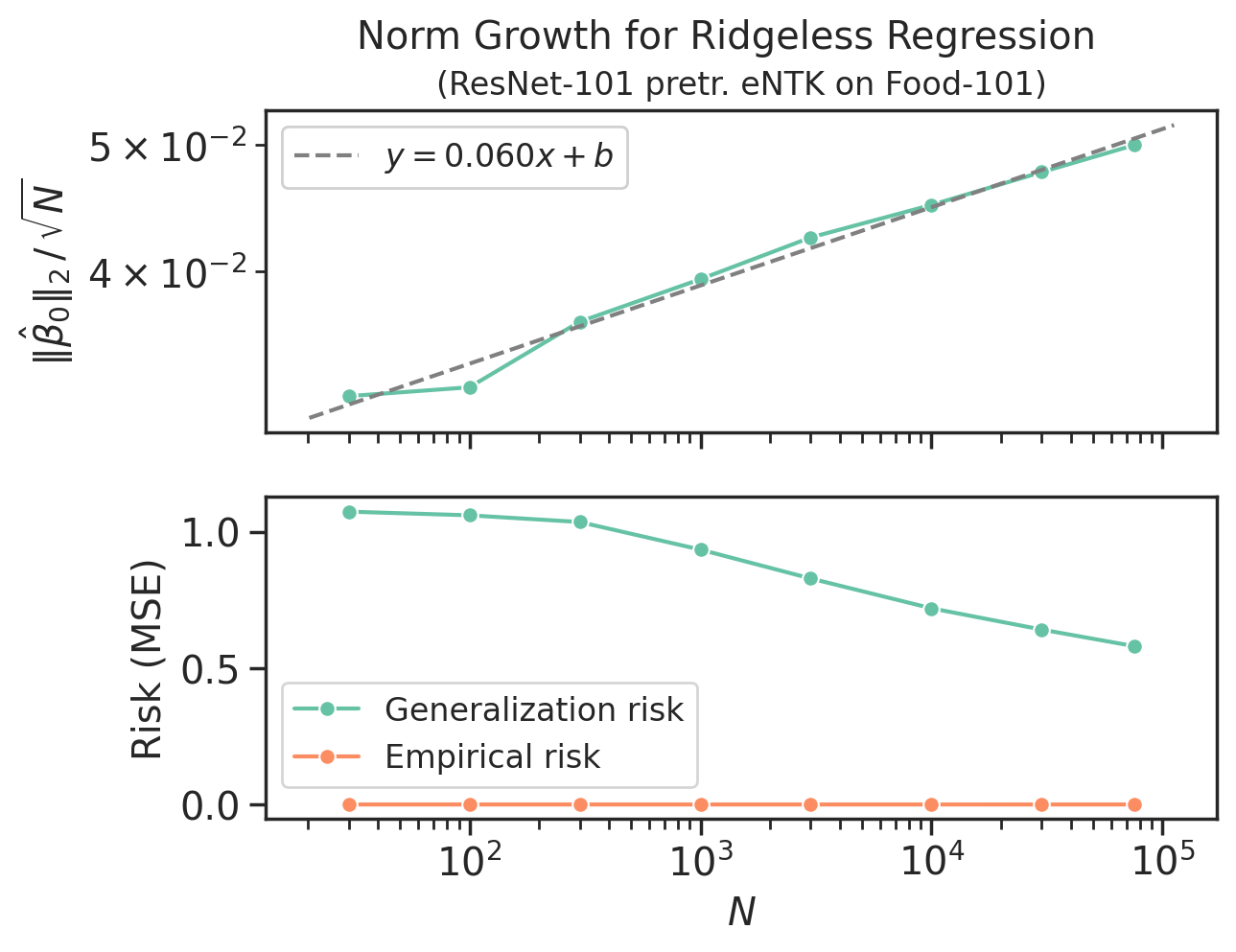}
    \caption{Additional plots showing the growth of the norm $\norm{\hat\beta_0}_2/\sqrt\n$ for ridge regression on the \NTK{}s additional models and datasets.}
    \label{fig:ng}
\end{figure}

\clearpage

\subsection{Spectrum Comparisons}

In this section, we provide additional examples of the slow convergence of the spectrum (\Cref{fig:spectr}) and of pretrained models having higher effective dimension (\Cref{fig:pretr}), showing that these trends also hold over a variety of datasets and models.

\begin{figure}[H]
    \centering
    \includegraphics[width=0.7\textwidth]{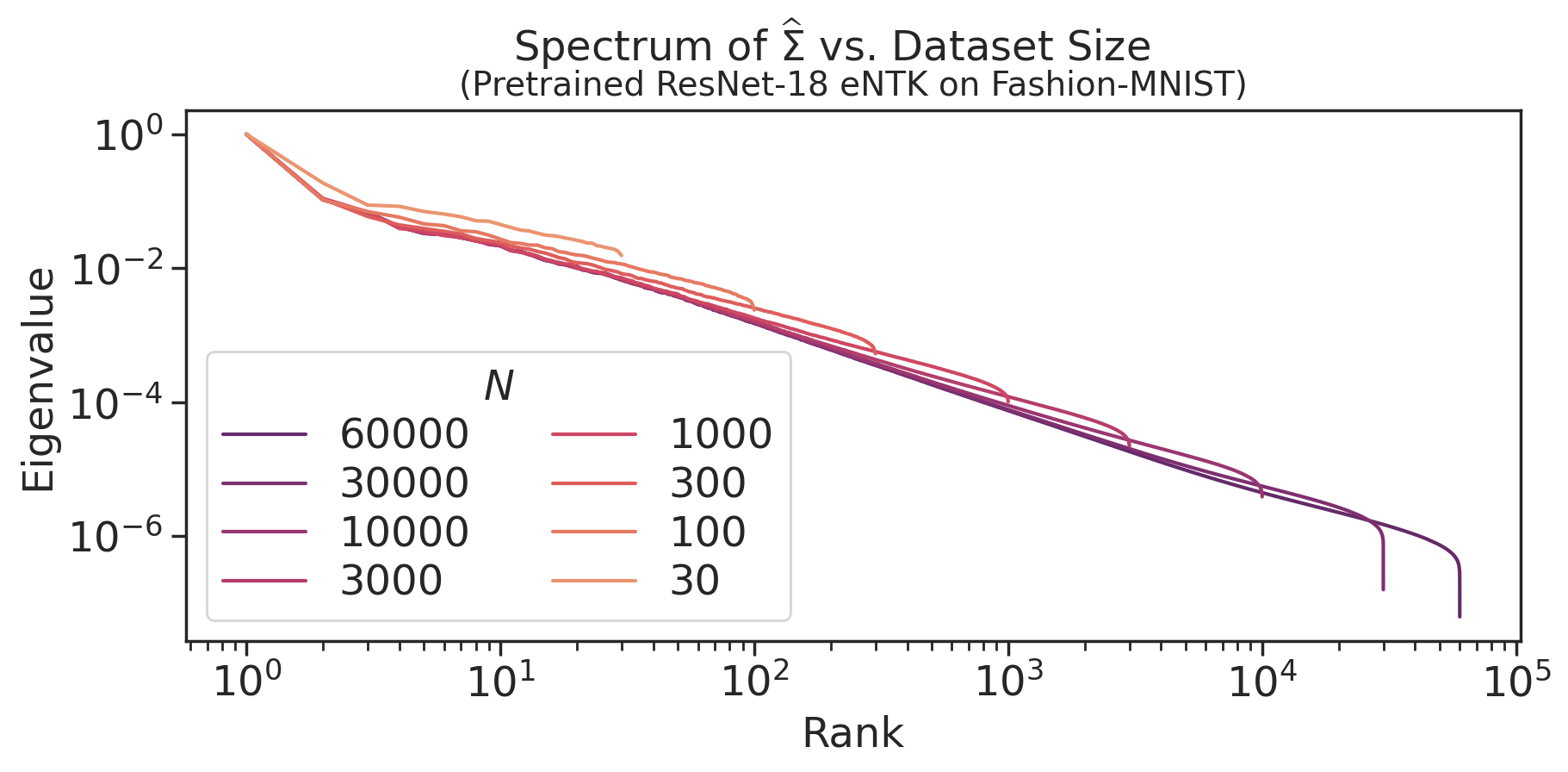}
    \includegraphics[width=0.7\textwidth]{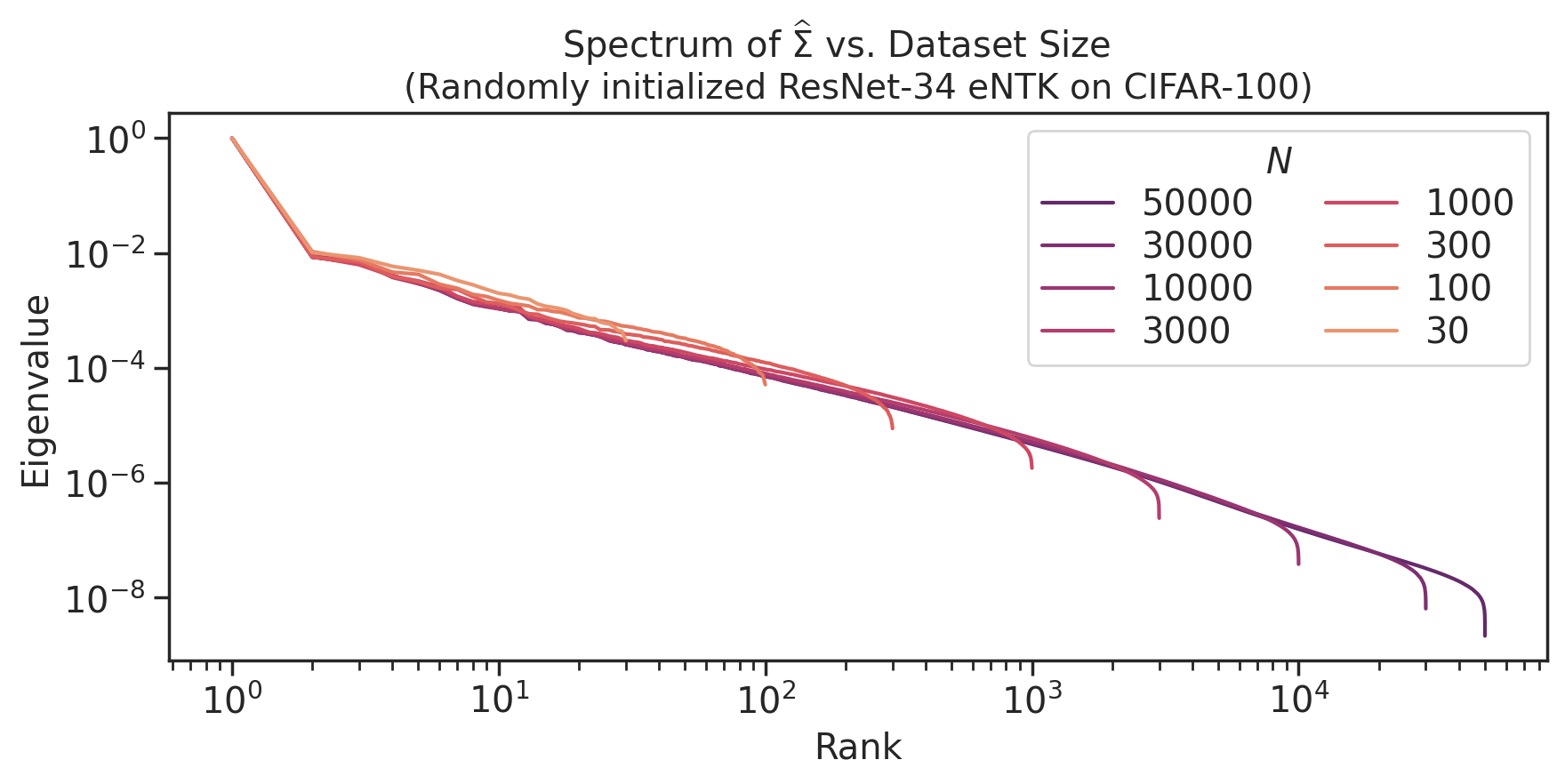}
    \includegraphics[width=0.7\textwidth]{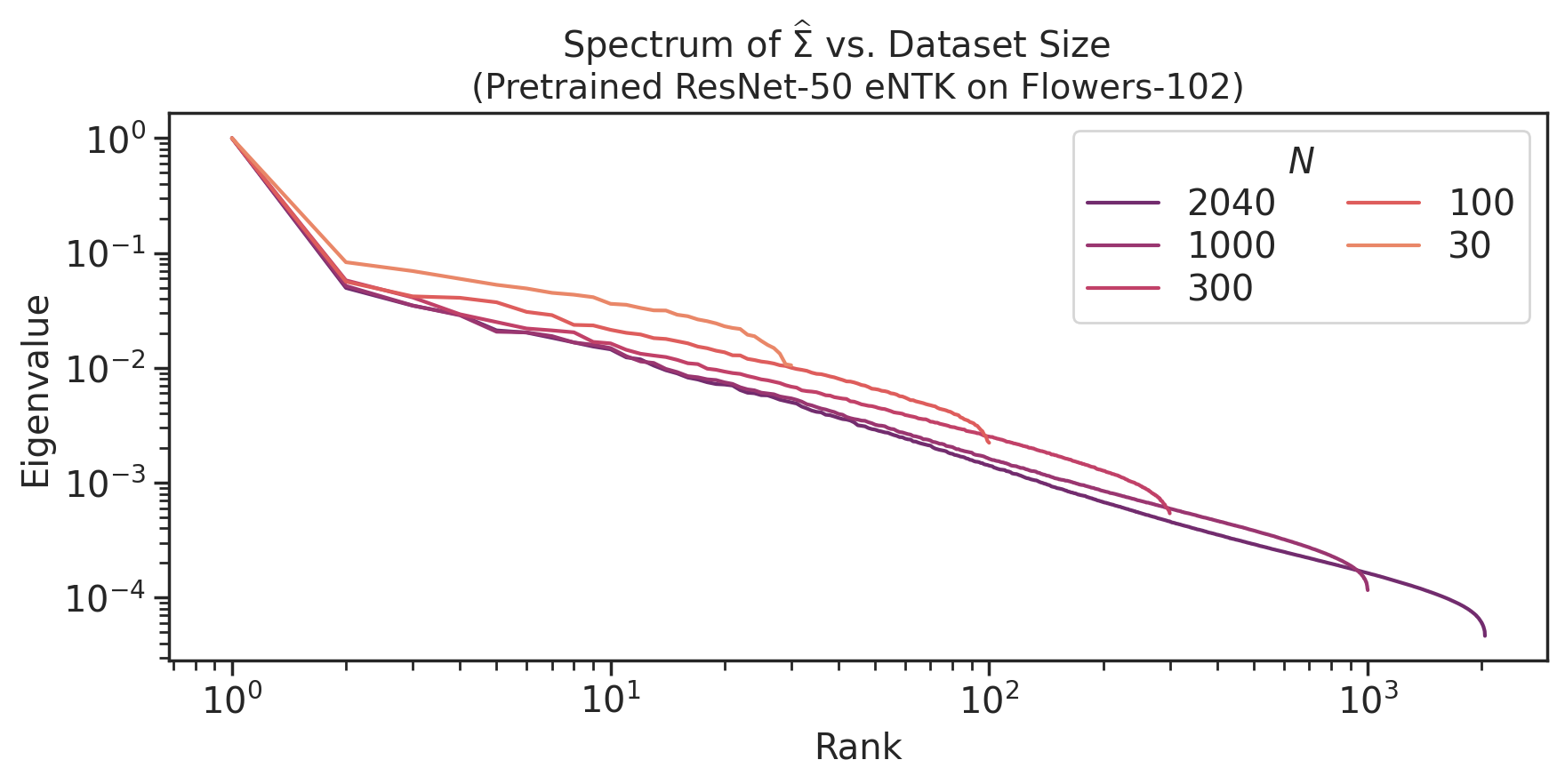}
    \caption{Additional plots showing the slow convergence of the empirical eigenvalue spectrum to the population eigenvalue spectrum.}
    \label{fig:spectr}
\end{figure}

\begin{figure}[H]
    \centering
    \includegraphics[scale=0.8]{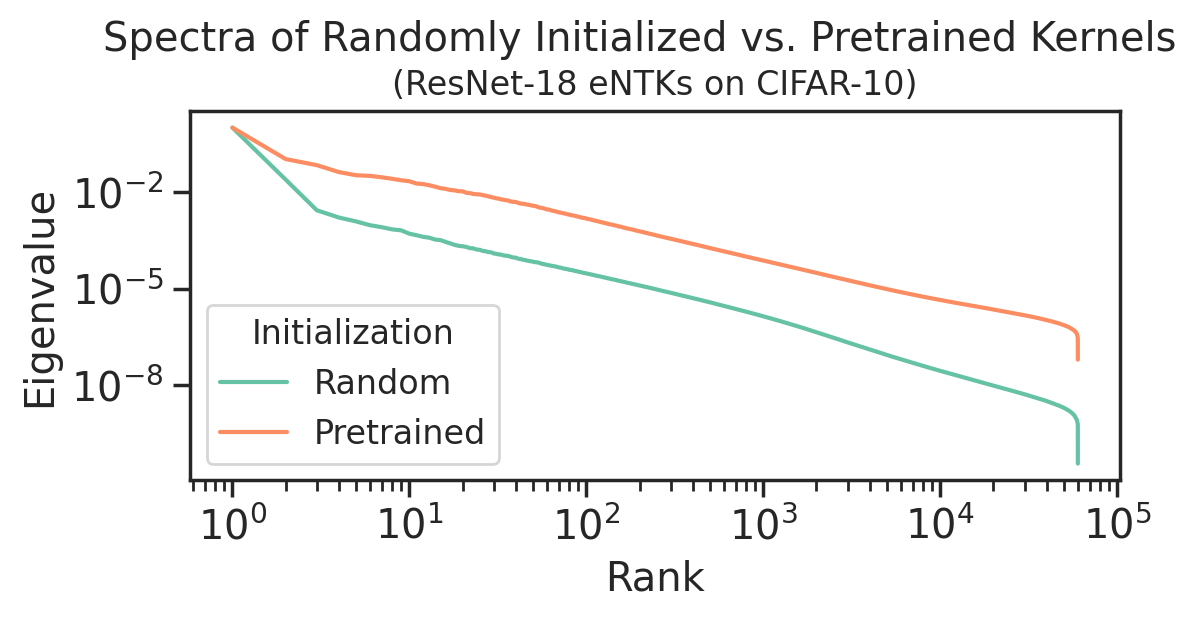}
    \includegraphics[scale=0.8]{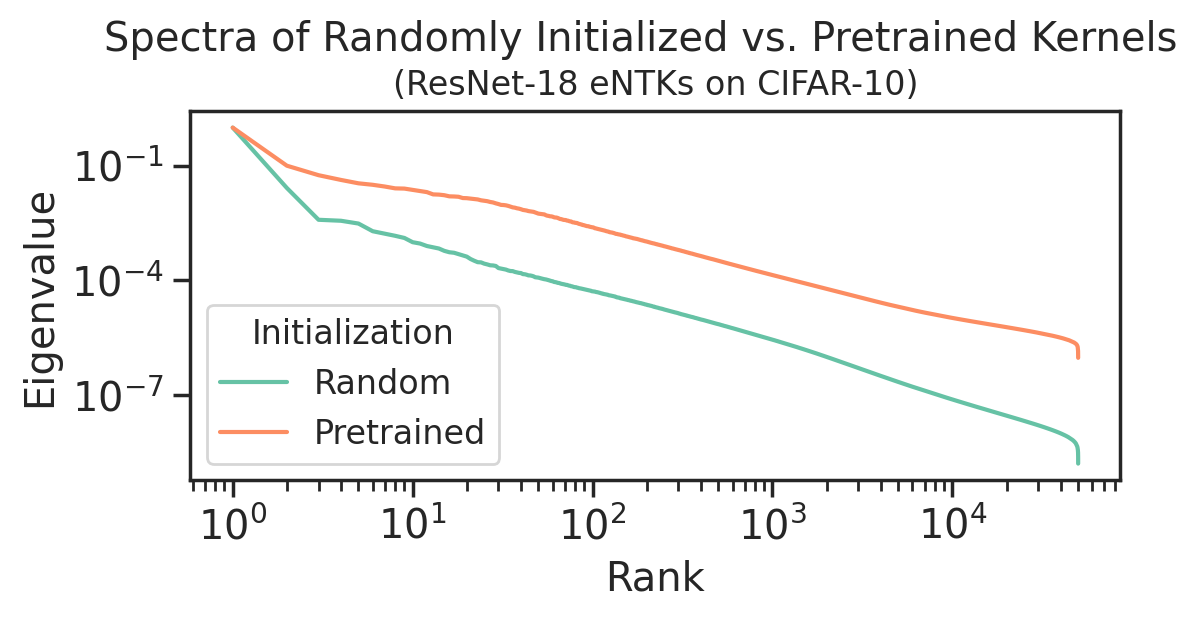}
    \includegraphics[scale=0.8]{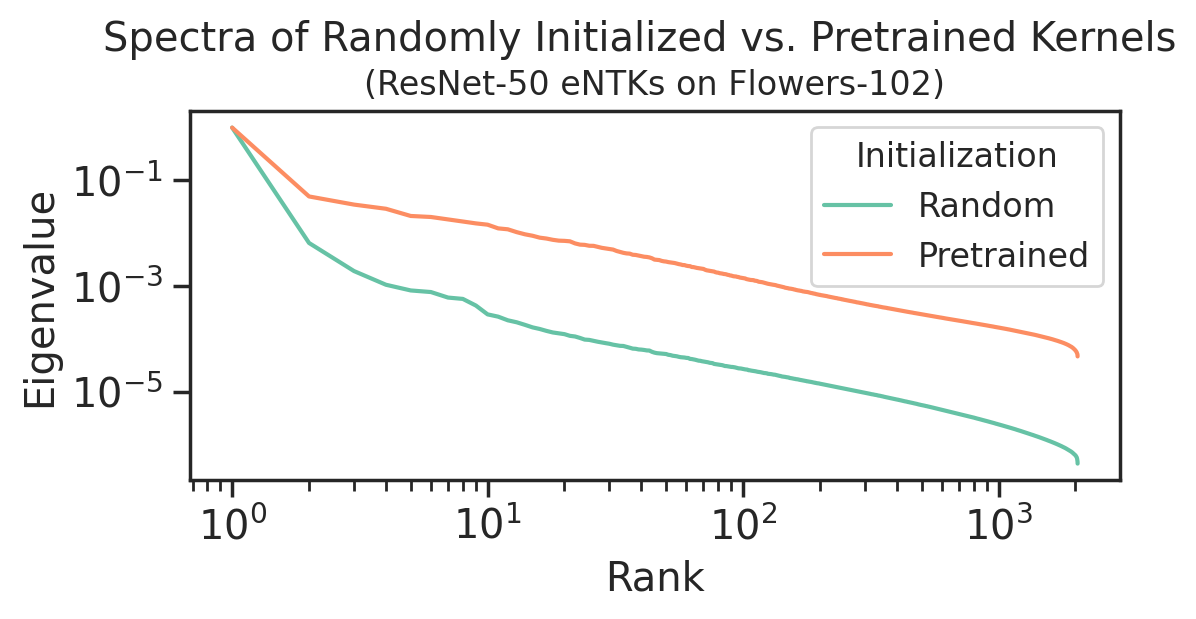}
    \caption{Additional plots showing that pretrained representations have slower eigendecay and thus higher effective dimension.}
    \label{fig:pretr}
\end{figure}

\clearpage

\subsection{Regression on Last Layer Activations}

In this section, we consider predicting the generalization risk of ridge regression on the last layer activations of pretrained models. \Cref{fig:all-last} plots the results of these experiments. These plots show that, in this lower-dimensional setting that spans the under- and overparameterized regimes, the GCV estimator continues to perform well.

\begin{figure}[H]
    \centering
    \includegraphics[width=\textwidth]{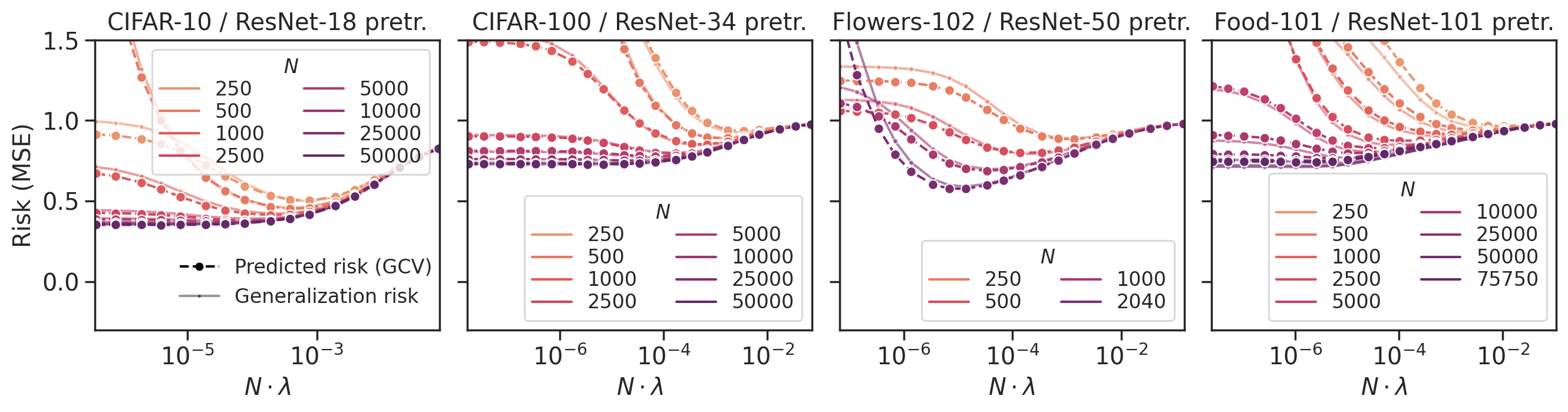}
    \caption{Generalization risk vs.\ the GCV prediction for regression on the \emph{last-layer} activations, for various datasets and networks, across sample sizes $\n$ and regularization levels $\lambda$}
    \label{fig:all-last}
\end{figure}

\clearpage

\subsection{Plots for the Norm- and Spectrum-Based Predictors}

To provide further intuition about the predictors $\NormRisk$ and $\SpecRisk$, we provide plots of the predictions that they make for our empirical setting in \Cref{fig:all-norm,,fig:all-spec}.

\begin{figure}[H]
    \centering
    \includegraphics[width=\textwidth]{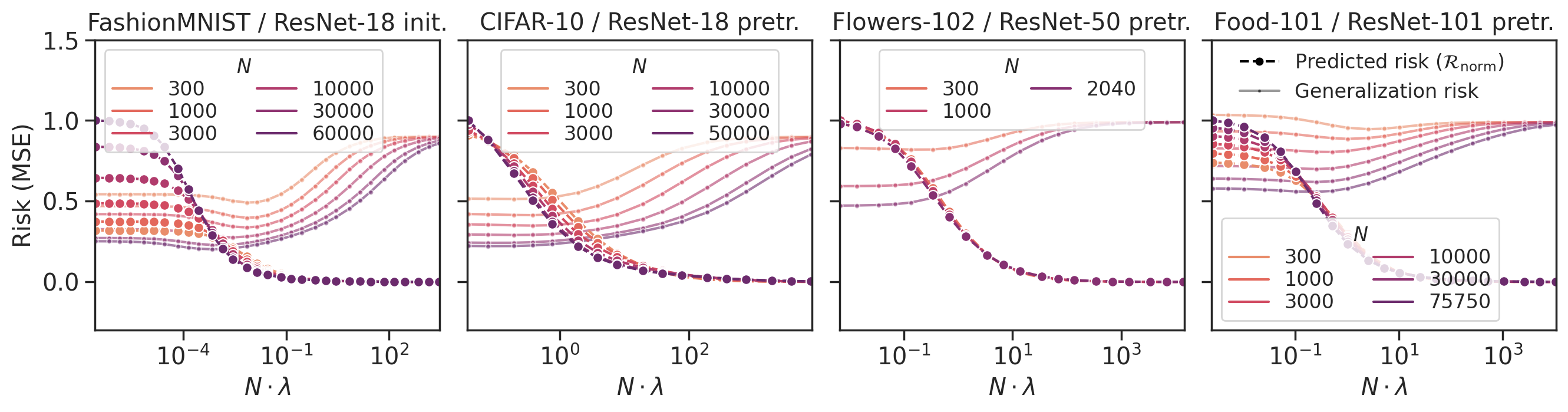}
    \caption{Plots of the norm-based predictor $\norm{\hat\beta_\lambda}_2/\sqrt\n$ against the generalization risk for various datasets and architectures. We normalize the predictions so that the maximum prediction in any graph is $1$. Note that the prediction tends to be negatively correlated with the actual test risk.}
    \label{fig:all-norm}
\end{figure}

\begin{figure}[H]
    \centering
    \includegraphics[width=\textwidth]{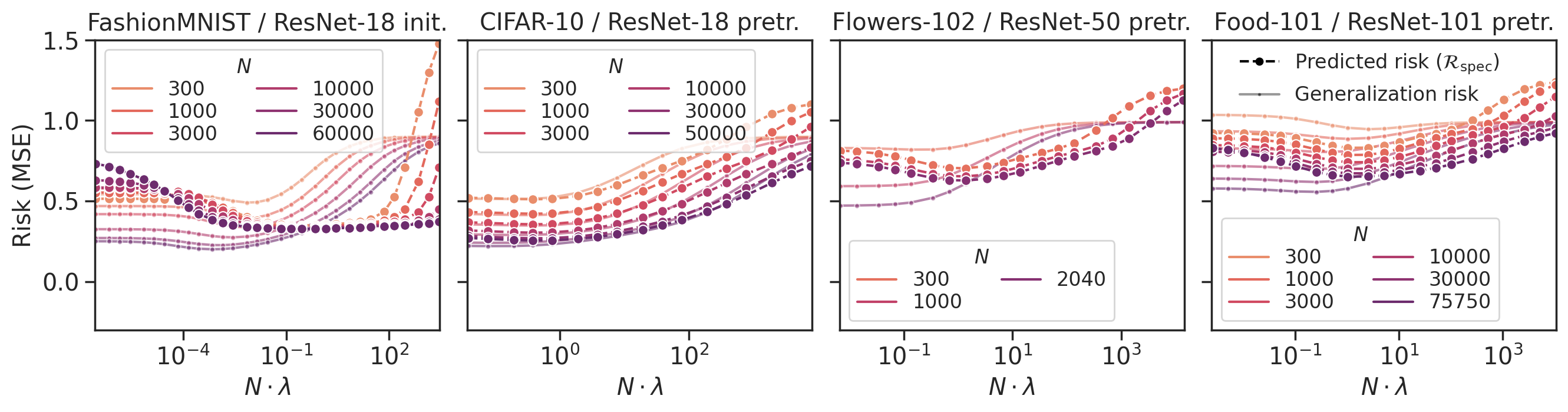}
    \caption{Plots of the $\SpecRisk^{\alpha,\sigma}$ for $\alpha$, $\sigma$ fitted as per \Cref{sec:details} against the generalization risk for various datasets and architectures. Note that this approach has trouble in particular fitting the randomly initialized setting.}
    \label{fig:all-spec}
\end{figure}

\end{document}